\documentclass{article}

\PassOptionsToPackage{numbers, compress}{natbib}

\usepackage[final]{neurips_2023}

\usepackage[utf8]{inputenc} 
\usepackage[T1]{fontenc}    
\usepackage[hidelinks]{hyperref}       
\usepackage{url}            
\usepackage{booktabs}       
\usepackage{amsfonts}       
\usepackage{nicefrac}       
\usepackage{microtype}      
\usepackage{xcolor}         

\usepackage{graphicx}
\usepackage{tikz}
\usepackage{amsmath,amssymb}
\usepackage{mathtools}
\usepackage[normalem]{ulem}  %
\usepackage{tabu}
\usepackage{multirow}
\usepackage{pgfplots,pgfplotstable}
\usepgfplotslibrary{fillbetween}
\usepackage{calc}
\usepackage{pbox}
\usepackage{algorithm}
\usepackage[noend]{algpseudocode}
\usepackage{caption}
\usepackage{afterpage}
\usepackage{mathtools}
\usepackage{algpseudocode}
\usepackage{siunitx}
\usepackage{adjustbox}
\usepackage{amsthm}
\usepackage{cleveref}
\usepackage{graphicx}
\usepackage{subcaption}
\usepackage{multirow}

\newcommand{\xx}{\boldsymbol{x}}

\newcommand{\bz}{\boldsymbol{z}}
\newcommand{\bx}{\boldsymbol{x}}

\newcommand{\bxi}{\boldsymbol{\xi}}

\newcommand{\by}{\boldsymbol{y}}

\newcommand{\beps}{\boldsymbol{\epsilon}}
\newcommand{\btheta}{\boldsymbol{\theta}}

\newcommand{\boldzero}{\mathbf{0}}

\newcommand{\diff}{\mathrm{d}}

\newcommand{\noise}{\boldsymbol{\epsilon}}

\newcommand{\bm}{\boldsymbol{w}_t}
\newcommand{\rbm}{\bar{\boldsymbol{w}}_t}

\newcommand{\nablaxx}{\nabla_{\hspace{-0.5mm}\xx}}

\newcommand{\Real}{\mathbb{R}}
\newcommand{\Normal}{\mathcal{N}}
\newcommand{\Identity}{\boldsymbol{I}}
\newcommand{\Expectation}{\mathbb{E}}

\newcommand{\bw}{\text{\boldmath{$w$}}}

\DeclareMathOperator*{\argmin}{arg\,min}

\newtheorem{theorem}{Theorem}[section]
\newtheorem{proposition}[theorem]{Proposition}
\newtheorem{lemma}[theorem]{Lemma}
\newtheorem{corollary}[theorem]{Corollary}
\newtheorem{definition}[theorem]{Definition}
\newtheorem{assumption}[theorem]{Assumption}
\newtheorem{remark}{Remark}
\title{\textit{SA-Solver}: Stochastic Adams Solver for Fast Sampling of Diffusion Models}
\author{%
  Shuchen Xue\textnormal{$^{1,4}$}\thanks{Work done during an internship at Huawei Noah's Ark Lab. Email: \nolinkurl{xueshuchen17@mails.ucas.ac.cn}, \hspace*{1.35em}\nolinkurl{luoweijian@stu.pku.edu.cn}}, Mingyang Yi\textnormal{$^{2}$}\thanks{Corresponding authors: Mingyang Yi (\nolinkurl{yimingyang2@huawei.com})}, Weijian Luo\textnormal{$^{3}$}\footnotemark[1], Shifeng Zhang\textnormal{$^{2}$},Jiacheng Sun\textnormal{$^{2}$},\and \textbf{Zhenguo Li}\textnormal{$^{2}$}, \textbf{Zhi-Ming Ma}\textnormal{$^{1,4}$} \\
  \textnormal{$^{1}$}University of Chinese Academy of Sciences \textnormal{$^{2}$} Huawei Noah's Ark Lab
  \textnormal{$^{3}$} Peking University\\\textnormal{$^{4}$}Academy of Mathematics and Systems Science, Chinese Academy of Sciences\\
}
\begin{document}
\maketitle
\begin{abstract}
Diffusion Probabilistic Models (DPMs) have achieved considerable success in generation tasks. As sampling from DPMs is equivalent to solving diffusion SDE or ODE which is time-consuming, numerous fast sampling methods built upon improved differential equation solvers are proposed. The majority of such techniques consider solving the diffusion ODE due to its superior efficiency. However, stochastic sampling could offer additional advantages in generating diverse and high-quality data. In this work, we engage in a comprehensive analysis of stochastic sampling from two aspects: variance-controlled diffusion SDE and linear multi-step SDE solver. Based on our analysis, we propose \textit{SA-Solver}, which is an improved efficient stochastic Adams method for solving diffusion SDE to generate data with high quality. Our experiments show that \textit{SA-Solver} achieves: 1) improved or comparable performance compared with the existing state-of-the-art (SOTA) sampling methods for few-step sampling; 2) SOTA FID on substantial benchmark datasets under a suitable number of function evaluations (NFEs). Code is available at \url{https://github.com/scxue/SA-Solver}.
\end{abstract}
\section{Introduction}\label{sec:intro}
Diffusion Probabilistic Models (DPMs)~\citep{sohl2015deep, ho2020denoising, song2020score} have demonstrated substantial success across a broad spectrum of generative tasks such as image synthesis~\citep{dhariwal2021diffusion, Rombach_2022_CVPR, ho2022cascaded}, video generation~\citep{ho2022vdm,blattmann2023videoldm}, text-to-image generation~\citep{2021glide,dalle2,imagen}, speech synthesis~\citep{lam2022bddm, BinauralGrad}, \emph{etc}. The primary mechanism of DPMs involves a forward diffusion process that incrementally introduces noise into data. Simultaneously, a reverse diffusion process is learned to generate data from this noise. Despite DPMs demonstrating enhanced generation performance in comparison to alternative methods such as Generative Adversarial Networks (GAN)~\citep{goodfellow2014generative} or Variational Autoencoders (VAE)~\citep{kingma2013auto}, the sampling process of DPMs demand hundreds of evaluations of network function evaluations (NFE)~\citep{ho2020denoising}. The substantial computation requirement poses a significant limitation to their wider application in practice.
\par
The existing literature on improving the sampling efficacy of DPMs can be categorized into two ways, depending on whether conducting extra training on the DPMs. The first category necessitates supplementary training~\citep{luhman2021knowledge,salimans2022progressive,meng2022on,watson2021learning,xiao2021tackling, wang2023diffusiongan}, which often emerges as a bottleneck, thereby limiting their practical application. Due to this, we focus on exploring the second category, which consists training-free methods to improve the sampling efficiency of DPMs in this paper. Current training-free samplers employ efficient numerical schemes to solve the diffusion SDE/ODE\citep{song2020denoising,lu2022dpmsa,zhao2023unipc,jolicoeur2021gotta,bao2022analytic}. Compared with solving diffusion SDE (stochastic sampler) \citep{jolicoeur2021gotta,bao2022analytic,Karras2022edm}, solving diffusion ODE (deterministic sampler) \citep{song2020denoising,lu2022dpmsa,zhao2023unipc} empirically exhibits better sampling efficiency. Existing stochastic samplers typically exhibit slower convergence speed. However, empirical observations in \citep{song2020score,Karras2022edm} indicate that the stochastic sampler has the potential to generate higher-quality data when increasing the sampling steps. This empirical observation motivates us to further explore the efficient stochastic sampler.  
\par
Owing to the observed superior performance of stochastic sampler \citep{song2020score,Karras2022edm}, we speculate that adding properly scaled noise in the diffusion SDE may facilitate the quality of generated data. Thus, instead of solving the vanilla diffusion SDE in \citep{song2020denoising}, we propose to consider a family of diffusion SDEs which shares the same marginal distribution \citep{zhang2023fast,Karras2022edm} with different noise scales. Meanwhile, efficient stochastic solvers are not carefully studied, which could be the reason that diffusion ODE exhibits better sampling efficiency. To overcome this problem, we study the linear multi-step SDE solvers~\citep{2006multistepsde} and incorporate them in the sampling. 
\par
Based on these studies, we propose \textit{SA-Solver} with theoretical convergence order to solve the proposed diffusion SDEs. Our \textit{SA-Solver} is based on the stochastic Adams method in \citep{2006multistepsde}, by adapting it to the exponentially weighted integral and analytical variance. With the proposed diffusion SDEs and \textit{SA-Solver}, we can efficiently generate data with controllable noise scales. We empirically evaluate our \textit{SA-Solver} on plenty of benchmark datasets of image generation. The evaluation criterion is the Fréchet Inception Distance (FID) score~\citep{heusel2017gans} under different number of function evaluations (NFEs). The experimental results can be summarized as three folds: 1) Under small NFEs, our \textit{SA-Solver} has improved or comparable FID scores, compared with baseline methods; 2) Under suitable NFEs our \textit{SA-Solver} achieves the State-of-the-Art FID scores over all benchmark datasets; 3) \textit{SA-Solver} achieves superior performance over deterministic samplers when the model is not fully trained.
\section{Related Works}
The DPMs originate from the milestone work \citep{sohl2015deep}, and are further developed by \citep{ho2020denoising} and \citep{song2020score} to successfully generate high-quality data, under the framework of discrete and continuous diffusion SDEs respectively. In this paper, we mainly focus on the latter framework. As mentioned in Section \ref{sec:intro}, plenty of papers are working on accelerating the sampling of DPMs due to their low efficiency, distinguished by whether conducting a supplementary training stage. Training-based methods, e.g., knowledge distillation~\citep{luhman2021knowledge,salimans2022progressive,meng2022on}, learning-to-sample \citep{watson2021learning}, and integration with GANs \citep{xiao2021tackling, wang2023diffusiongan}, have the potential to sampling for one or very few steps to enhance the efficiency, but their applicability is limited by the lack of a plug-and-play nature, thereby constraining their broad applicability across diverse tasks. Thus we mainly focus on the training-free methods in this paper.
\paragraph{Solving Diffusion ODE.} Since the sampling process is equivalent to solving diffusion SDE (ODE), the training-free methods are mainly built on solving the differential equations via high-efficiency numerical methods. As ODEs are easier to solve compared with SDEs, the ODE sampler has attracted great attention. For example, \citet{song2020denoising} provides an empirically efficient solver DDIM. \citet{zhang2023fast} and \citet{lu2022dpmsa} point out the semi-linear structure of diffusion ODEs, and develop higher-order ODE samplers based on it. \citet{zhao2023unipc} further improve these samplers in terms of NFEs by integrating the mechanism of predictor-corrector method.   
\paragraph{Solving Diffusion SDE.} Though less explored than the ODE sampler, the SDE sampler exhibits the potential of generating higher-quality data~\citep{Karras2022edm}. Thus developing an efficient SDE sampler as we did in this paper is a meaningful topic. In the existing literature, researchers~\citep{ho2020denoising,bao2022analytic,song2020score} solve the diffusion SDE by first-order discretization numerical method. The higher-order stochastic sampler of diffusion SDE has also been discussed in~\citep{jolicoeur2021gotta}. \citet{Karras2022edm} proposes another stochastic sampler (which is not a general SDE numerical solver) tailored for diffusion problems. However, in contrast to our proposed \textit{SA-Solver}, the existing SDE samplers are limited due to their low efficiency~\citep{ho2020denoising,bao2022analytic,song2020score} or sensitivity to hyperparameters~\citep{Karras2022edm}.  
\par
We found a concurrent paper proposing an SDE sampler SDE-DPM-Solver++ \citep{lu2023dpmsolver++} which is similar to our \textit{SA-Solver}. Though both methods develop multi-step diffusion SDE samplers, our \textit{SA-Solver} is different from SDE-DPM-Solver++ as follows: 1) \textit{SA-Solver} incorporates the predictor-corrector method, which helps improve the quality of generated data~\citep{song2020score,li2023erasolver,zhao2023unipc}; 2) In contrast to SDE-DPM-Solver++, \textit{SA-Solver} has theoretical guarantees with proved convergence order; 3) SDE-DPM-Solver++ is a special case of \textit{SA-Solver} when the predictor step equals 2 with no corrector in our predictor-corrector method, while our solver supports arbitrary orders with analytical forms.
\section{Preliminary}\label{sec:pre}
In the regime of the continuous stochastic differential equation (SDE), Diffusion Probabilistic Models (DPMs)~\citep{sohl2015deep, ho2020denoising, song2020score, kingma2021variational} construct noisy data through the following linear SDE:
\begin{equation}
\label{eq: forward process}
\diff \xx_t = f(t) \xx_t \diff t + g(t)  \diff \bm,
\end{equation}
where $\bm \in \Real^d$ represents the standard Wiener process, $f(t) \xx_t$ and $g(t)$ respectively denote the drift and diffusion coefficients. For each time $t \in [0,T]$, $\xx_t | \xx_0 \sim \Normal(\alpha_t \xx_0, \sigma^2_t\Identity)$.
\par
Let $p_t(\xx)$ denotes the marginal distribution of $\xx_t$, the coefficients $f(t)$ and $g(t)$ are meticulously selected to guarantee that the marginal distribution $p_T(\xx_T)$ closely approximates a Gaussian distribution, i.e., $\Normal(\boldzero, \Identity)$, and the \textit{signal-to-noise-ratio} (SNR) $\alpha_t^2/\sigma_t^2$ is strictly decreasing w.r.t. $t$. In the sequel, we follow the established notations in~\citep{kingma2021variational}:
\begin{equation}
\label{eq: relation_f_alpha}
f(t) = \frac{\diff \log \alpha_t}{\diff t},\hspace{4mm} g^2(t) = \frac{\diff \sigma_t^2}{\diff t} - 2\frac{\diff \log \alpha_t}{\diff t} \sigma_t^2.
\end{equation}
\citet{ANDERSON1982313} demonstrates a pivotal theorem that the forward process~\eqref{eq: forward process} has an equivalent reverse-time diffusion process (from $T$ to $0$) as the following equation, so that generating process can be equivalent to numerically solve the diffusion SDE \citep{ho2020denoising,song2020score}.   
\begin{equation}
\label{eq: reverse SDE}
\diff \xx_t = \left[f(t) \xx_t - g^2(t)\nablaxx\log p_t(\xx_t)\right] \diff t + g(t)  \diff \rbm, \qquad \xx_T \sim p_T(\xx_T)
\end{equation}
where $\rbm$ represents the Wiener process in reverse time, and $\nablaxx\log p_t(\xx)$ is the score function. 
\par
Moreover, \citet{song2020score} also prove that there exists a corresponding deterministic process whose trajectories share the same marginal probability densities $p_t(\xx)$ as~\eqref{eq: reverse SDE}, so that the ODE solver can be adopted for efficient sampling~\citep{lu2022dpmsa,zhao2023unipc}:
\begin{equation}
\label{eq: reverse ODE}
\diff \xx_t = \left[f(t) \xx_t - \frac{1}{2}g^2(t)\nablaxx\log p_t(\xx_t)\right] \diff t , \qquad \xx_T \sim p_T(\xx_T)
\end{equation}
To get the \textit{score function} $\nablaxx\log p_t(\xx_t)$ in \eqref{eq: reverse SDE}, we usually take neural network $\boldsymbol{s}_{\boldsymbol{\theta}}(\xx, t)$ parameterized by $\boldsymbol{\theta}$ to approximate it by optimizing the denoising score matching loss~\citep{song2020score}: 
\begin{equation}
\label{eq: score matching loss}
\boldsymbol{\theta}^* = \argmin_{\boldsymbol{\theta}} \Expectation_t \Bigl\{ \lambda(t) \Expectation_{\xx_0} \Expectation_{\xx_t|\xx_0} \bigl[\bigl\| \boldsymbol{s}_{\boldsymbol{\theta}}(\xx, t) - \nabla_{\xx_t}\log p_{0t}(\xx_t|\xx_0) \bigl\|_2^2 \bigl] \Bigl\}.
\end{equation}
In practice, two methods are used to reparameterize the score-based model~\citep{ramesh2022hierarchical}. The first approach utilizes a \textit{noise prediction model} such that  $\noise_{\boldsymbol{\theta}} (\xx_t, t) = -\sigma_t \boldsymbol{s}_{\boldsymbol{\theta}}(\xx_t, t)$, while the second employs a \textit{data prediction model}, represented by $\xx_{\boldsymbol{\theta}} (\xx_t, t) = (\xx_t - \sigma_t \noise_{\boldsymbol{\theta}} (\xx_t, t))/\alpha_t$. The reparameterized models are plugged into the sampling process \eqref{eq: reverse SDE} or \eqref{eq: reverse ODE} according to their relationship with $\boldsymbol{s}_{\boldsymbol{\theta}}(\xx_t, t)$. 

\section{Variance Controlled Diffusion SDEs}
\label{sec: Controllable Variance Reverse SDEs}
As mentioned in Section \ref{sec:intro}, most of the existing training-free efficient samplers are based on solving diffusion ODE \eqref{eq: reverse ODE}, e.g., \citep{lu2022dpmsa,song2020denoising,zhao2023unipc}, because of their improved efficiency compared with the solvers of diffusion SDE \eqref{eq: reverse SDE}. However, the empirical observations in \citep{Karras2022edm,song2020denoising} exhibit that the quality of data generated by solving diffusion SDE outperforms diffusion ODE given sufficient computational budgets. For example, in \citep{song2020denoising}, the diffusion ODE sampler DDIM \citep{song2020denoising} significantly improve the FID score of diffusion SDE sampler DDPM \citep{ho2020denoising} (from 133.37 to 6.84) on \texttt{CIFAR10} dataset~\citep{Krizhevsky09learningmultiple} under 20 NFEs. However, under 1000 NFEs, the DDPM beats the DDIM in terms of FID score (3.17 v.s. 4.04). There may be a trade-off between stochasticity and efficiency. Thus, we conjecture that adding proper scale noise during the generating process may improve the quality of generated data with few NFEs. 
\par
In this section, we explore a family of variance-controlled diffusion SDEs, so that we can use proper noise scales during the sampling stage. 
Inspired by Proposition 1 in \citep{zhang2023fast} and Eq. (6) in \citep{Karras2022edm}, we propose the following proposition to construct the aforementioned diffusion SDEs.
\begin{proposition}
\label{prop: controllable variance reverse SDEs}
For any bounded measurable function $\tau(t):[0,T] \rightarrow \Real$, the following Reverse SDEs 
\begin{equation}
\label{eq: controllable variance reverse SDEs}
\diff \xx_t = \left[f(t) \xx_t - \left(\frac{1 + \tau^2(t)}{2}\right) g^2(t)\nablaxx\log p_t(\xx_t)\right] \diff t + \tau(t)g(t)  \diff \rbm,\hspace{4mm}\xx_T \sim p_T(\xx_T)
\end{equation}
share the same marginal probability distributions with (\ref{eq: reverse ODE}) and (\ref{eq: reverse SDE}) .
\end{proposition}
The proof can be found in Appendix~\ref{appendix: proof of prop4.1}. The proposition indicates that by solving any of the diffusion SDEs in \eqref{eq: controllable variance reverse SDEs}, we can sample from the target distribution. It is worth noting that the magnitude of noise varies with $\tau(t)$, and $\tau(t) = 0$ or $\tau(t) = 1$ respectively correspond to the diffusion ODE and SDE in \citep{song2020score}. Thus we can control the magnitude of added noise during the sampling process by varying it.   
\par
In practice, we numerically solve the diffusion SDEs \eqref{eq: controllable variance reverse SDEs} by substituting score function $\nabla_{\bx}\log{p_{t}(\bx_{t})}$ in it with the ``data prediction reparameterization model'' $\bx_{\btheta}(\bx_{t}, t)$ according to $\nabla_{\bx}\log{p_{t}(\bx_{t})}\approx -(\bx_{t} - \alpha_{t}\bx_{\btheta}(\bx_{t}, t)) / \sigma_{t}^{2}$ as pointed out in Section \ref{sec:pre}. Then diffusion SDEs to be solved become
\begin{equation}
\label{eq: reparadata1 controllable variance reverse SDEs}
    d\bx_{t} = \left[f(t)\bx_{t} + \left(\frac{1 + \tau^{2}(t)}{2\sigma_{t}}\right)g^{2}(t)\left(\frac{\bx_{t} - \alpha_{t}\bx_{\btheta}(\bx_{t}, t)}{\sigma_{t}}\right)\right]dt + \tau(t)g(t)d\bar{\bw}_{t}. 
\end{equation}
\begin{remark}
    We reparameterize the score function in diffusion SDEs \eqref{eq: controllable variance reverse SDEs} with data prediction model $\bx_{\btheta}(\bx_{t}, t)$ to get Eq. \eqref{eq: reparadata2 controllable variance reverse SDEs}. The equation can be also reparameterized by the ``noise prediction model'' $\beps_{\btheta}(\bx_{t}, t)$ as discussed in Section \ref{sec:pre}. Though the obtained diffusion ODEs e.g., Eq. \eqref{eq: reparadata2 controllable variance reverse SDEs} are equivalent, the numerical solver applied to them will result in different solutions. For our proposed \textit{SA-Solver}, we find the diffusion SDEs reparameterized data prediction model significantly improves the quality of generated data. More details and theoretical explanations are in Sec.~\ref{sec:exp} and Appendix~\ref{appendix: data or noise reparameterization}. For the remaining part of the paper, we focus on data reparameterization.
\end{remark}
We then solve the diffusion SDEs \eqref{eq: reparadata2 controllable variance reverse SDEs} with change-of-variable applying to it, i.e., changing time variable $t$ to log-SNR $\lambda_{t} = \log{(\alpha_{t} / \sigma_{t})}$. Noting the following relationship in Eq.~\eqref{eq: relation_f_alpha}
\begin{equation}
\label{eq: change of variable f and alpha relationship}
f(t) = \frac{\diff \log \alpha_t}{\diff t},\hspace{4mm} g^2(t) = \frac{\diff \sigma_t^2}{\diff t} - 2\frac{\diff \log \alpha_t}{\diff t} \sigma_t^2 = -2 \sigma_t^2 \frac{\diff \lambda_t}{\diff t},
\end{equation}
and plugging them into \eqref{eq: reparadata1 controllable variance reverse SDEs}, it becomes 
\begin{equation}
\label{eq: reparadata2 controllable variance reverse SDEs}
\diff \xx_t = \left[\frac{\diff \log \alpha_t}{\diff t} \xx_t - (1 + \tau^2(t))(\bx_{t} - \alpha_{t}\bx_{\btheta}(\bx_{t}, t)) \frac{\diff \lambda_t}{\diff t} \right]\diff t + \tau(t) \sigma_t \sqrt{-2\frac{\diff \lambda_t}{\diff t}}  \diff \rbm.
\end{equation}
The above equation has an explicit solution owing to its semi-linear structure \citep{atkinson2011numerical}. 
\begin{proposition}
\label{prop: exact solution data prediction model}
Given $\xx_s$ for any time $s > 0$, the solution $\xx_t$ at time $t \in [0,s]$ of \eqref{eq: reparadata2 controllable variance reverse SDEs} is
\begin{equation}
\label{eq: exact solution data prediction model}
\begin{aligned}
&\xx_t = \frac{\sigma_t}{\sigma_s} e^{-\int_{\lambda_s}^{\lambda_t} \tau^2(\Tilde{\lambda}) \diff \Tilde{\lambda}} \xx_s +  \sigma_t \boldsymbol{F}_{\btheta}(s,t) + \sigma_t \boldsymbol{G}(s,t), \\
&\boldsymbol{F}_{\btheta}(s,t) = \int_{\lambda_s}^{\lambda_t} e^{-\int_{\lambda}^{\lambda_t} \tau^2(\Tilde{\lambda}) \diff \Tilde{\lambda} }\left(1+\tau^2\left(\lambda\right)\right) e^{\lambda} \xx_{\btheta}\left(\xx_{\lambda}, \lambda\right) \diff \lambda \\
&\boldsymbol{G}(s,t) = \int_s^t e^{-\int_{\lambda_u}^{\lambda_t} \tau^2(\Tilde{\lambda}) \diff \Tilde{\lambda} } \tau(u) \sqrt{-2\frac{\diff \lambda_u}{\diff u}} \diff \bar{\boldsymbol{w}}_u,
\end{aligned}
\end{equation}
where $\boldsymbol{G}(s,t)$ is an Itô integral \citep{oksendal2013stochastic} with the special property 
\begin{equation}
 \label{eq: Analytical variance of reverse SDEs data}
\sigma_t  \boldsymbol{G}(s,t) \sim \Normal\Bigl(\boldzero, \sigma_t^2 \bigl( 1 - e^{-2\int_{\lambda_{s}}^{\lambda_t} \tau^2(\Tilde{\lambda}) \diff \Tilde{\lambda}} \bigl) \Bigl).
\end{equation}
\end{proposition}
The proof can be seen in Appendix~\ref{appendix: proof of prop 4.2}. With this proposition, we can sample from the diffusion model via numerically solving Eq.~\eqref{eq: exact solution data prediction model} starting from $\bx_{T}$ approximated by a Gaussian distribution. 
\section{\textit{SA-Solver}: Stochastic Adams Method to Solve Diffusion SDEs}
Stochastic training-free samplers for solving diffusion SDEs have not been studied as systematically as their deterministic ODE counterparts. This stems from the inherent challenges associated with designing numerical schemes for SDEs compared to ODEs \citep{kloeden1992Numerical}. Existing stochastic sampling methods either use only variant of one-step discretization of diffusion SDEs~\citep{ho2020denoising, bao2022analytic, song2020score}, or are specifically designed sampling procedures for diffusion processes~\citep{Karras2022edm} which are not general purpose SDE solvers. Jolicoeur-Martineau \textit{et al.}~\citep{jolicoeur2021gotta} uses stochastic Improved Euler's method~\citep{roberts2012siescheme} with adaptive step sizes. However, it still necessitates hundreds of steps to yield a high-quality sample. As observed by~\citep{jolicoeur2021gotta}, off-the-shelf SDE solvers are generally ill-suited for diffusion models, often exhibiting inferior qualities or even failing to converge. We postulate that the current dearth of fast stochastic samplers is principally due to factor that existing methodologies predominantly tend to rely on one-step discretization or its variants, or alternatively, on heuristic designs of stochastic samplers. To address this factor, we leverage advanced contemporary tools in numerical solutions for SDEs, specifically, \textit{stochastic Adams methods}~\citep{2006multistepsde}. It necessitates fewer evaluations compared to Stochastic Runge-Kutta schemes, making it a more suitable choice for problems which are computationally expensive - a characteristic that diffusion sampling certainly exemplifies.
\par
Next, we formally present our Stochastic Adams Solver (\textit{SA-Solver}). To solve Eq.~\eqref{eq: reparadata2 controllable variance reverse SDEs}, we first take $M + 1$ time steps $\left\{t_i\right\}_{i=0}^M$ which is strictly decreased from $t_0 = T$ to $t_M = 0$.\footnote{The diffusion SDEs \eqref{eq: reparadata1 controllable variance reverse SDEs} are reverse-time SDEs, so that the $t_{i}$ here is increased.} Then we can iteratively obtain the $\bx_{t_{i}}$ (so that $\bx_{0}$ approximates the required data) by the following relationship. 
\begin{equation}
\label{eq: integration of reverse SDE}
\xx_{t_{i + 1}} = \frac{\sigma_{t_{i + 1}}}{\sigma_{t_{i}}} e^{-\int_{\lambda_{t_i}}^{\lambda_{t_{i+1}}} \tau^2(\lambda_u) \diff \lambda_u} \xx_{t_i}+  \sigma_{t_{i + 1}} \boldsymbol{F}_{\btheta}(t_i,t_{i+1})+ \sigma_{t_{i + 1}}  \boldsymbol{G}(t_i,t_{i+1})
\end{equation}
As pointed out in Proposition \ref{prop: exact solution data prediction model}, the Itô integral term $\boldsymbol{G}(t_{i},t_{i+1})$ in above equation follows a Gaussian that can be directly sampled so we need to solve the deterministic integral term $F_{\btheta}(t_{i},t_{i+1})$. 

We further combine Eq.~\eqref{eq: integration of reverse SDE} with the predictor-corrector method, which is a widely used numerical method. It works in two main steps. First, a predictor step is taken to make an initial approximation of the solution. Second, a corrector step will refine the predictor's approximation by taking the predicted value into account. It has been proven successful in the wide application of numerical analysis~\citep{atkinson2011numerical}. Especially, there are some attempts to use the predictor-corrector method to help sample diffusion models~\citep{song2020score, li2023erasolver, zhao2023unipc}. In the subsequent Section~\ref{subsec: SA-Predictor} and Section~\ref{subsec: SA-Corrector}, we will separately derive our \textit{SA-Predictor} and \textit{SA-Corrector} using Eq.~\eqref{eq: integration of reverse SDE}. Our algorithm is outlined in Algorithm \ref{alg: SA-Solver}.
\par
\begin{algorithm}[t]
  \caption{\textit{SA-Solver}}\label{alg: SA-Solver}
  \begin{algorithmic}[1]
    \Require data prediction model $\xx_{\boldsymbol{\theta}}$, timesteps $\left\{t_i\right\}_{i=0}^M$, initial value $\xx_{t_0}$, predictor step $s_p$, corrector step $s_c$, buffer $B$ to store former evaluation of $\xx_{\boldsymbol{\theta}}$, $\tau(t)$ to control variance.
    \State $B \xleftarrow{\text{buffer}} \xx_{\boldsymbol{\theta}}(\xx_{t_0}, t_0)$
    \For{$i=1$ to $max$$(s_p, s_c)$}\Comment{Warm-up}
        \State sample $\bxi \sim \Normal(\boldzero, \Identity)$
        \State calculate steps for warm-up  $s_p^m = min(i, s_p)$, $s_c^m = min(i, s_c)$
        \State $\xx_{t_i}^p \gets s_p^m$-step \textit{SA-Predictor}$(\xx_{t_{i-1}}, B, \bxi)$ (Eq.~\eqref{s-step SA-Predictor}) \Comment{Prediction Step}
        \State $B \xleftarrow{\text{buffer}} \xx_{\boldsymbol{\theta}}(x_{t_i}^p, t_i)$\Comment{Evaluation Step}
        \State $\xx_{t_i} \gets s_c^m$-step \textit{SA-Corrector}$(\xx_{t_i}^p, \xx_{t_{i-1}}, B, \bxi)$ (Eq.~\eqref{s-step SA-Corrector})\Comment{Correction Step}
    \EndFor
    
    \For{$i=max$$(s_p, s_c) + 1$ to $M$}
        \State sample $\bxi \sim \Normal(\boldzero, \Identity)$
        \State $\xx_{t_i}^p \gets s_{p}$-step \textit{SA-Predictor}$(\xx_{t_{i-1}}, B, \bxi)$ (Eq.~\eqref{s-step SA-Predictor})\Comment{Prediction Step}
        \State $B \xleftarrow{\text{buffer}} \xx_{\boldsymbol{\theta}}(x_{t_i}^p, t_i)$\Comment{Evaluation Step}
        \State $\xx_{t_i} \gets s_{c}$-step \textit{SA-Corrector}$(\xx_{t_i}^p, \xx_{t_{i-1}}, B, \bxi)$ (Eq.~\eqref{s-step SA-Corrector})\Comment{Correction Step}
    \EndFor
    \Return $\xx_{t_M}$
  \end{algorithmic}
\end{algorithm}
\subsection{SA-Predictor}
\label{subsec: SA-Predictor}
The fundamental idea behind stochastic Adams methods is to leverage preceding model evaluations like $\xx_{\btheta}(\xx_{t_i}, t_{i}), \xx_{\btheta}(\xx_{t_{i-1}}, t_{i-1}), \cdots, \xx_{\btheta}(\xx_{t_{i-(s-1)}}, t_{i-(s-1)})$. These evaluations can be retained with negligible cost implications. Given these preceding model evaluations, a natural strategy for estimating $F_{\btheta}(t_{i},t_{i+1})$ involves the application of Lagrange interpolations~\citep{atkinson2011numerical} of these evaluations. Lagrange interpolation of $s$ points $\xx_{\btheta}(\xx_{t_i}, t_{i}), \xx_{\btheta}(\xx_{t_{i-1}}, t_{i-1}), \cdots, \xx_{\btheta}(\xx_{t_{i-(s-1)}}, t_{i-(s-1)})$ is a polynomial $L(t)$ with degrees $s-1$:
\begin{equation}
\label{eq: Lagrange interpolation}
L(t) = \sum_{j=0}^{s-1} l_{i-j}(t) \xx_{\btheta}(\xx_{t_{i-j}}, t_{i-j}),
\end{equation}
where $l_{i-j}(t): \Real \rightarrow \Real$ is the Lagrange basis. Lagrange interpolation is an excellent approximation of $\xx_{\btheta}(\xx_{t}, t)$ with the special property: $L(t_{i-j}) = \xx_{\btheta}(\xx_{t_{i-j}}, t_{i-j}), \hspace{4mm} \forall\ 0\leq j \leq s-1$. Thus a natural way to estimate $F_{\btheta}(t_{i},t_{i+1})$ is to replace $\xx_{\boldsymbol{\theta}} (\xx_{\lambda_u}, \lambda_u)$ with $L(\lambda)$, which is just a change-of-variable of $L(t)$. The formula for $s$-step \textit{SA-Predictor} is then derived.
\paragraph{$s$-step \textit{SA-Predictor}}Given the initial value $\xx_{t_i}$ at time $t_i$, a total of $s$ former model evaluations $\xx_{\btheta}(\xx_{t_i}, t_{i}), \xx_{\btheta}(\xx_{t_{i-1}}, t_{i-1}), \cdots, \xx_{\btheta}(\xx_{t_{i-(s-1)}}, t_{i-(s-1)})$, our $s$-step \textit{SA-Predictor} is defined as:
\begin{equation}
\label{s-step SA-Predictor}
\xx_{t_{i + 1}} = \frac{\sigma_{t_{i + 1}}}{\sigma_{t_{i}}} e^{-\int_{\lambda_{t_i}}^{\lambda_{t_{i+1}}} \tau^2(\Tilde{\lambda}) \diff \Tilde{\lambda}} \xx_{t_i} + \sum_{j=0}^{s-1} b_{i-j} \xx_{\btheta}(\xx_{t_{i-j}}, t_{i-j}) + \Tilde{\sigma}_i \bxi, \hspace{4mm} \bxi \sim \Normal(\boldzero, \Identity),
\end{equation}
where $\Tilde{\sigma}_i = \sigma_{t_{i + 1}}\sqrt{1 - e^{-2\int_{\lambda_{t_i}}^{\lambda_{t_{i + 1}}} \tau^2(\Tilde{\lambda}) \diff \Tilde{\lambda}}  }$ according to Proposition~\ref{prop: exact solution data prediction model} and $b_{i-j}$ is given by:
\begin{equation}
\label{eq: SA-Predictor coefficient}
b_{i-j} = \sigma_{t_{i + 1}}   \int_{\lambda_{t_i}}^{\lambda_{t_{i+1}}} e^{-\int_{\lambda}^{\lambda_{t_{i + 1}}} \tau^2(\Tilde{\lambda}) \diff \Tilde{\lambda} }\left(1+\tau^2\left(\lambda\right)\right) e^{\lambda}    l_{i-j}(\lambda) \diff \lambda,\hspace{4mm} \forall\ 0\leq j \leq s-1
\end{equation}
We show the convergence result in the following theorem. The proof can be found in Appendix~\ref{appendix: Derivations and Proofs for SA-Solver}.
\begin{theorem}[Strong Convergence of $s$-step \textit{SA-Predictor}]
\label{thm: coefficients and convergence of predictor}
Under mild regularity conditions, our $s$-step \textit{SA-Predictor} (Eq.~\eqref{s-step SA-Predictor}) has a global error in strong convergence sense of $\mathcal{O}(\underset{0 \leq t \leq T}{\sup}\tau(t) h + h^s)$, where $h = \underset{1 \leq i \leq M}{\max} \left( t_i - t_{i-1} \right)$. 
\end{theorem}
\subsection{SA-Corrector}
\label{subsec: SA-Corrector}
Eq.~\eqref{s-step SA-Predictor} offers a ``prediction'' $\xx_{t_{i+1}}^{p}$ that relies on information preceding or coinciding with the time step $t_i$ since we only use $\xx_{\btheta}(\xx_{t_i}, t_i)$ along with other model evaluations antecedent to it, while the integration is over time $\left[t_i, t_{i+1}\right]$. Then predictor-corrector method can be incorporated to better estimate $F_{\btheta}(t_{i},t_{i+1})$ in Eq.~\eqref{eq: integration of reverse SDE}. We perform a model evaluation $\xx_{\btheta}(\xx_{t_{i+1}}^{p}, t_{i+1})$ and construct the Lagrange interpolations of $\xx_{\btheta}(\xx_{t_{i+1}}^{p}, t_{i+1}), \xx_{\btheta}(\xx_{t_i}, t_{i}), \cdots, \xx_{\btheta}(\xx_{t_{i-(s'-1)}}, t_{i-(\hat{s}-1)})$:
\begin{equation}
\label{eq: second Lagrange interpolation}
\hat{L}(t) = \hat{l}_{i+1}(t)\xx_{\btheta}(\xx_{t_{i+1}}^{p}, t_{i+1}) + \sum_{j=0}^{\hat{s}-1} \hat{l}_{i-j}(t) \xx_{\btheta}(\xx_{t_{i-j}}, t_{i-j}),
\end{equation}
where $\hat{l}_{i-j}(t): \Real \rightarrow \Real$ is the Lagrange basis and $\hat{s}$ can be different with $s$ in Eq.~(\ref{eq: Lagrange interpolation}). The $\hat{s}$-step \textit{SA-Corrector} is derived by replacing $\xx_{\boldsymbol{\theta}} (\xx_{\lambda_u}, \lambda_u)$ with $\hat{L}(\lambda)$ which is a change-of-variable of $\hat{L}(t)$. 

\paragraph{$\hat{s}$-step \textit{SA-Corrector}} Given the initial value $\xx_{t_i}$ at time $t_i$, a total of $\hat{s}$ former model evaluations $\xx_{\btheta}(\xx_{t_i}, t_{i}), \xx_{\btheta}(\xx_{t_{i-1}}, t_{i-1}), \cdots, \xx_{\btheta}(\xx_{t_{i-(\hat{s}-1)}}, t_{i-(\hat{s}-1)})$, model evaluation of ``prediction'' $\xx_{\btheta}(\xx_{t_{i+1}}^{p}, t_{i+1})$, our $\hat{s}$-step \textit{SA-Corrector} is defined as:
\begin{equation}
\label{s-step SA-Corrector}
\xx_{t_{i + 1}} = \frac{\sigma_{t_{i + 1}}}{\sigma_{t_{i}}} e^{-\int_{\lambda_{t_i}}^{\lambda_{t_{i+1}}} \tau^2(\Tilde{\lambda}) \diff \Tilde{\lambda}} \xx_{t_i} +\hat{b}_{i+1} \xx_{\btheta}(\xx_{t_{i+1}}^{p}, t_{i+1}) + \sum_{j=0}^{\hat{s}-1} \hat{b}_{i-j} \xx_{\btheta}(\xx_{t_{i-j}}, t_{i-j}) + \Tilde{\sigma}_i \bxi,  
\end{equation}
where $\bxi \sim \Normal(\boldzero, \Identity)$, $\Tilde{\sigma}_i = \sigma_{t_{i + 1}}\sqrt{1 - e^{-2\int_{\lambda_{t_i}}^{\lambda_{t_{i + 1}}} \tau^2(\Tilde{\lambda}) \diff \Tilde{\lambda}}  }  $ according to Proposition~\ref{prop: exact solution data prediction model} and the coefficients $\hat{b}_{i+1}$, $\hat{b}_{i-j}$ is given by:
\begin{equation}
\label{eq: SA-Corrector coefficient}
\begin{aligned}
&\hat{b}_{i-j} = \sigma_{t_{i + 1}}   \int_{\lambda_{t_i}}^{\lambda_{t_{i+1}}} e^{-\int_{\lambda}^{\lambda_{t_{i + 1}}} \tau^2(\Tilde{\lambda}) \diff \Tilde{\lambda} }\left(1+\tau^2\left(\lambda\right)\right) e^{\lambda}    \hat{l}_{i-j}(\lambda) \diff \lambda,\hspace{4mm} \forall\ 0\leq j \leq s-1 \\
& \hat{b}_{i+1} = \sigma_{t_{i + 1}}   \int_{\lambda_{t_i}}^{\lambda_{t_{i+1}}} e^{-\int_{\lambda}^{\lambda_{t_{i + 1}}} \tau^2(\Tilde{\lambda}) \diff \Tilde{\lambda} }\left(1+\tau^2\left(\lambda\right)\right) e^{\lambda}    \hat{l}_{i+1}(\lambda) \diff \lambda
\end{aligned}
\end{equation}
We show the convergence result in the following theorem. The proof can be found in Appendix~\ref{appendix: Derivations and Proofs for SA-Solver}.
\begin{theorem}[Strong Convergence of $\hat{s}$-step \textit{SA-Corrector}]
\label{thm: coefficients and convergence of corrector}
Under mild regularity conditions, our $\hat{s}$-step \textit{SA-Corrector} (Eq.~\eqref{s-step SA-Corrector}) has a global error in strong convergence sense of $\mathcal{O}(\underset{0 \leq t \leq T}{\sup}\tau(t) h + h^{\hat{s}+1})$, where $h = \underset{1 \leq i \leq M}{\max} \left( t_i - t_{i-1} \right)$.
\end{theorem}
\subsection{Connection with other samplers}
We briefly discuss the relationship between our \textit{SA-Solver} and other existing solvers for sampling diffusion ODEs or diffusion SDEs.
\paragraph{Relationship with DDIM~\citep{song2020denoising}} DDIM generate samples through the following process:
\begin{equation}
\label{eq: DDIM sample}
\xx_{t_{i+1}} = \alpha_{t_{i+1}} \left( \frac{\xx_{t_i} - \sigma_{t_i} \boldsymbol{\epsilon}_{\btheta}(\xx_{t_i}, t_i) }{\alpha_{t_i}} \right) + \sqrt{1 - \alpha_{t_{i+1}}^2 - \hat{\sigma}_{t_i}^2} \boldsymbol{\epsilon}_{\btheta}(\xx_{t_i}, t_i) + \hat{\sigma}_{t_i} \boldsymbol{\xi}, 
\end{equation}
where $\boldsymbol{\xi} \sim \Normal(\boldzero, \Identity)$, $\hat{\sigma}_{t_i}$ is a variable parameter. In practice, DDIM introduces a parameter $\eta$ such that when $\eta = 0$, the sampling process becomes deterministic and when $\eta = 1$, the sampling process coincides with original DDPM~\citep{ho2020denoising}. Specifically, $\hat{\sigma}_{t_i} = \eta \sqrt{\frac{1 - \alpha_{t_{i+1}}^2}{1 - \alpha_{t_{i}}^2} \left( 1 - \frac{\alpha_{t_{i}}^2}{\alpha_{t_{i+1}}^2} \right)}$.
\begin{corollary}[Relationship with DDIM]
\label{corol: relationship with ddim}
For any $\eta$ in DDIM, there exists a $\tau_{\eta}(t):\Real \rightarrow \Real$ which is a piecewise constant function such that DDIM-$\eta$ coincides with our $1$-step \textit{SA-Predictor} when $\tau(t) = \tau_{\eta}(t)$ with data parameterization of our variance-controlled diffusion SDE.
\end{corollary}
The proof can be found in Appendix~\ref{appendix: proof of corollary 5.3}.
\paragraph{Relationship with DPM-Solver++(2M)~\citep{lu2023dpmsolver++}} DPM-Solver++ is a high-order solver which solves diffusion ODEs for guided sampling. DPM-Solver++(2M) is a special case of our $2$-step \textit{SA-Predictor} when $\tau(t) \equiv 0$.
\paragraph{Relationship with UniPC~\citep{zhao2023unipc}} UniPC is a unified predictor-corrector framework for solving diffusion ODEs. UniPC-p is a special case of our \textit{SA-Solver} when $\tau(t) \equiv 0$ with predictor step $p$, corrector step $p$ in Algorithm~\ref{alg: SA-Solver}.

\section{Experiments}
\label{sec:exp}

In this section, we demonstrate the effectiveness of \textit{SA-Solver} over the existing sampling methods on both a small number of function evaluations (NFEs) settings and a considerable number of NFEs settings, with extensive experiments. We use Frechet Inception Distance (FID)~\citep{heusel2017gans} as the evaluation metric to show the effectiveness of \textit{SA-Solver}. Unless otherwise specified, 50K images are sampled for evaluation. The experiments are conducted on various datasets, with image sizes ranging from 32x32 to 256x256. We also evaluate the performance of various models, including ADM~\citep{dhariwal2021diffusion}, EDM~\citep{Karras2022edm}, Latent Diffusion~\citep{Rombach_2022_CVPR}, and DiT~\citep{peebles2022scalable}.

For ease of computation, we take $\tau(t) \equiv \tau$ as a constant function or a piecewise constant function. We leave the detailed settings for $\tau(t)$, predictor step, and corrector step in Appendix~\ref{appendix: experiment details}. For the following experiments, we first discuss the effectiveness of the data-prediction model. Then we evaluate the performance of \textit{SA-Solver} under different random noise scales $\tau$ to demonstrate the principles for selecting $\tau$ under few-steps and a considerable number of steps. Finally, we compare \textit{SA-Solver} with the existing solver to demonstrate its effectiveness.

\subsection{Comparison between Data-Prediction Model and Noise-Prediction Model}

We first discuss the necessity of using a data-prediction model for \textit{SA-Solver}. We test on ImageNet 256x256 (latent diffusion model) with $\tau(t) \equiv 1$. Results of the data-prediction and noise-prediction model are shown in Table \ref{tab:data_prediction}. It can be seen that the data-prediction model can achieve better sampling quality values under different NFEs, thus we use the data-prediction model in the rest of the experiments. More detailed discussions and theoretical analysis can be seen in Appendix~\ref{appendix: data or noise reparameterization}. 
\begin{table}[t]
  \caption{Compared results by FID $\downarrow$ under data-prediction and noise-prediction models, measured by different NFEs. The latent diffusion model in ImageNet 256x256 is used for evaluation.}
  \label{tab:data_prediction}
  \centering
  {\setlength{\extrarowheight}{1.5pt}
	\begin{tabular}{lcc}
        \hline
	    NFEs & Noise-prediction & Data-prediction \\
        \hline
        20 & 310.5 & \bf 3.88 \\
        40 & 5.85 & \bf 3.47 \\
        60 & 3.54 & \bf 3.41 \\
        80 & 3.41 & \bf 3.38 \\
        \hline
        \\[-2ex]
	\end{tabular}
 }
\end{table}
\begin{table}[t]
  \caption{Compared results by FID $\downarrow$ under different predictor steps and corrector steps, measured by different NFEs. The VE-baseline model~\citep{Karras2022edm} in CIFAR10 32x32 is used for evaluation.}
  \label{tab:pcablation}
  \centering
  {\setlength{\extrarowheight}{1.5pt}
\begin{tabular}{lcccc}
        \hline
	    method $\backslash$ setting (NFE, $\tau$)& 15,0.4 & 23,0.8 & 31,1.0 & 47,1.4 \\
        \hline
Predictor 1-steps only & 13.76& 12.44  & 11.72  & 14.67 \\
Predictor 1-steps, Corrector 1-step & 8.49 & 6.87& 6.13 & 6.75\\
Predictor 3-steps only  & 5.30& 3.93& 3.52& 2.98 \\
Predictor 3-steps, Corrector 3-steps & \bf 4.91 &  \bf3.77 &  \bf3.40 & \bf 2.92 \\ 
        \hline
        \\[-2ex]
\end{tabular}
}
\end{table}
\begin{table}[t]
  \caption{Sampling quality measured by FID of different sampling methods on DiT, Min-SNR ImageNet~\citep{peebles2022scalable,hang2023efficient} models. DiT-XL/2-G and ViT-XL-patch2-32 with $s=1.5$ are used.}
  \label{tab:dit}
  \centering
  {\setlength{\extrarowheight}{3.0pt}
	\begin{tabular}{lcc}
        \hline
	    Model & \multicolumn{2}{c}{FID ($\downarrow$)} \\
        \hline
        \multirow{2}{*}{DiT ImageNet 256x256}&DDPM (NFE=250)&\textbf{SA-Solver (Ours)} (NFE=60) \\
        & 2.27 & \bf 2.02 \\
        \hline
        \multirow{2}{*}{Min-SNR ImageNet 256x256}&Heun (NFE=50)&\textbf{SA-Solver (Ours)} (NFE=20) \\
        & 2.06 & \bf 1.93 \\
        \hline
        
        \multirow{2}{*}{DiT ImageNet 512x512}&DDPM (NFE=250)&\textbf{SA-Solver (Ours)} (NFE=60) \\
        & 3.04 & \bf 2.80 \\
        \hline
        \\[-2ex]
	\end{tabular}
 }
\end{table}

\subsection{Ablation Study on Predictor/Corrector Steps and Predictor-Corrector Method}
\label{sec: ablationpc}
To verify the effectiveness of our proposed Stochastic Linear Multi-step Methods and Predictor-Corrector Method, we conduct an ablation study on the CIFAR10 dataset as follows. We use EDM~\citep{Karras2022edm} baseline-VE pretrained checkpoint. Concretely, we vary the number of predictor steps and meanwhile conduct them with/without corrector to separately explore the effect of the two components. As can be seen in Table~\ref{tab:pcablation}, both Stochastic Linear Multi-step Methods (Predictor 1-steps only v.s. Predictor 3-steps only) and Predictor-Corrector Method (Predictor 1-steps only v.s. Predictor 1-steps, Corrector 1-step, and Predictor 3-steps only v.s. Predictor 3-steps, Corrector 3-steps) improve the performance of our sampler.

\subsection{Effect on Magnitude of Stochasticity}

The proposed \textit{SA-Solver} is evaluated on various types of datasets and models, including ImageNet 256x256~\citep{deng2009imagenet} (latent diffusion model~\citep{Rombach_2022_CVPR}), LSUN Bedroom 256x256~\citep{yu15lsun} (pixel diffusion model~\citep{dhariwal2021diffusion}), ImageNet 64x64 (pixel diffusion model~\citep{dhariwal2021diffusion}), and CIFAR10 32x32 (pixel diffusion model~\citep{Karras2022edm}). The models corresponding to these datasets cover pixel-space and latent-space diffusion models, with unconditional, conditional, and classifier-free guidance settings ($s=1.5$ in ImageNet 256x256).

We used different constant $\tau$ values for \textit{SA-Solver}, namely $\{0.0, 0.2, 0.4, ..., 1.6\}$, where larger value of $\tau$ correspond to larger magnitude of stochasticity. The FID results under different NFE and $\tau$ values are shown in Fig. \ref{fig:ablations}. Note that for LSUN Bedroom, 10K images are sampled for evaluation. The experiments indicate that (1) under relatively small NFEs, smaller nonzero $\tau$ values can achieve better FID results; (2) under a considerable number of steps (20-100), large $\tau$ can achieve better FID. This phenomenon is consistent with the theoretical analysis we conducted in Appendix~\ref{appendix: Derivations and Proofs for SA-Solver} and Appendix~\ref{appendix: selection of stochasticity}, in which the sampling error with stochasticity is dominated under small NFE, while larger $\tau$ can significantly improve the quality of generated samples as the number of steps increases. In subsequent experiments, unless otherwise specified, we will report the results of a proper $\tau(t)$ value. Details can be found in Appendix~\ref{appendix: experiment details}.

\begin{figure}[t]
\centering
\includegraphics[width=\textwidth]{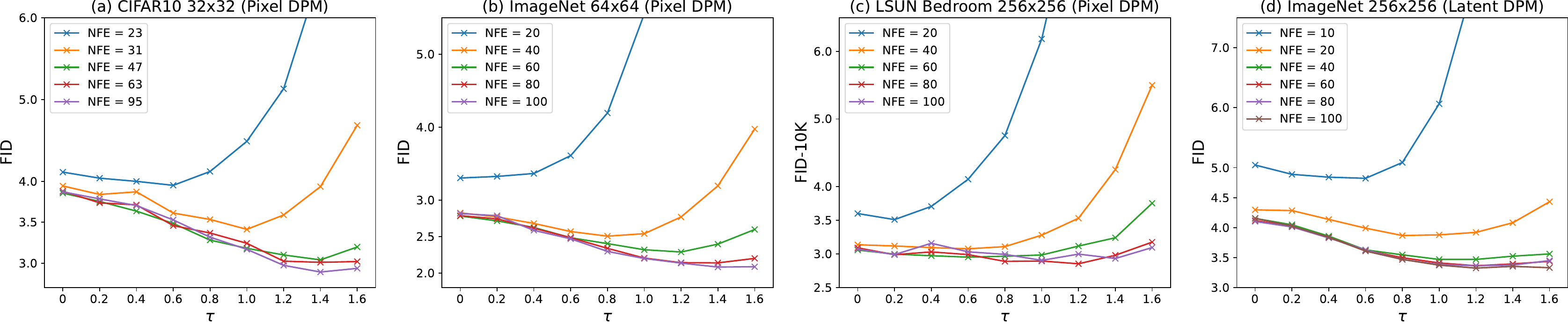}
\caption{Sampling quality measured by FID $\downarrow$ of \textit{SA-Solver} under a different number of function evaluations (NFE), varying the stochastic noise scale $\tau$. For LSUN Bedroom, 10K samples are used to evaluate FID.}
\label{fig:ablations}
\end{figure}

\begin{figure}[t]
\centering
\includegraphics[width=\textwidth]{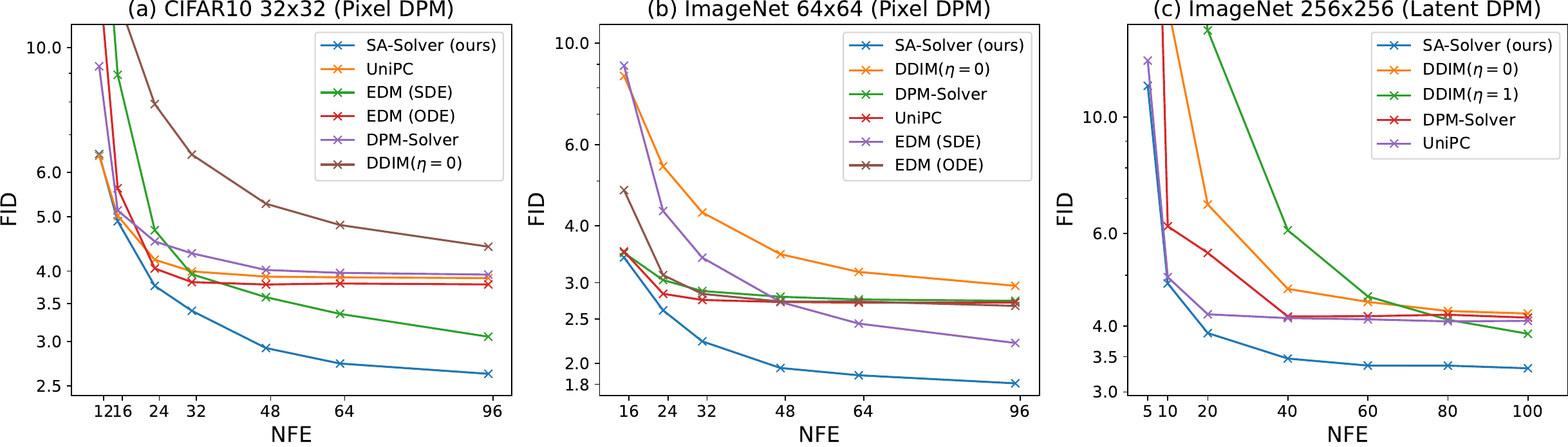}
\caption{Sampling quality measured by FID $\downarrow$ of different sampling methods of DPMs under different NFEs.}
\label{fig:sotas}
\end{figure}

\begin{figure}[t]
\centering
\includegraphics[width=\textwidth]{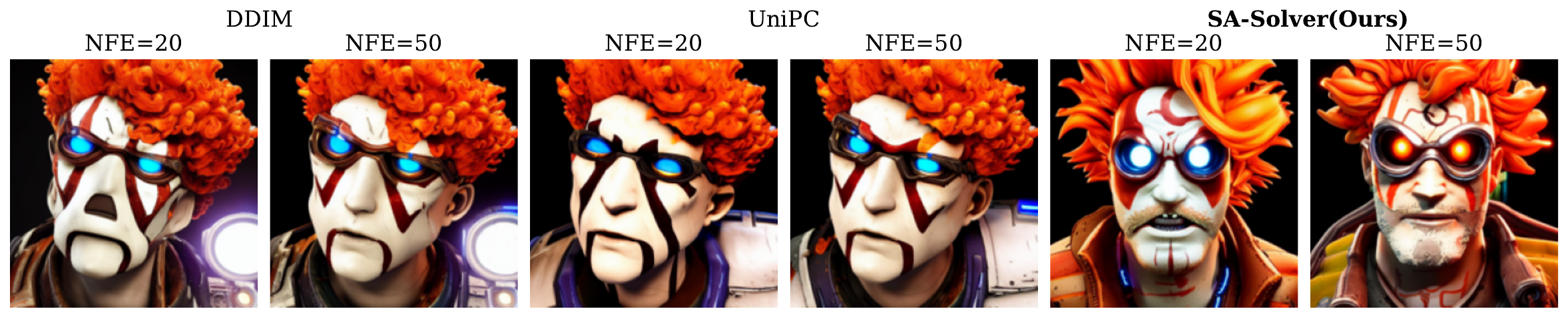}
\caption{Qualitative comparisons between our \textit{SA-Solver} and previous state-of-the-art methods. All images are generated by Stable Diffusion v1.5 with the same random seed. The main part of the prompt is ``portrait of curly orange haired mad scientist man''. We set the guidance scale as 7.5. The proposed \textit{SA-Solver} is able to generate images with more details.}
\label{fig:t2i}
\end{figure}

\begin{figure}[h!]
\centering
\includegraphics[width=0.8\textwidth]{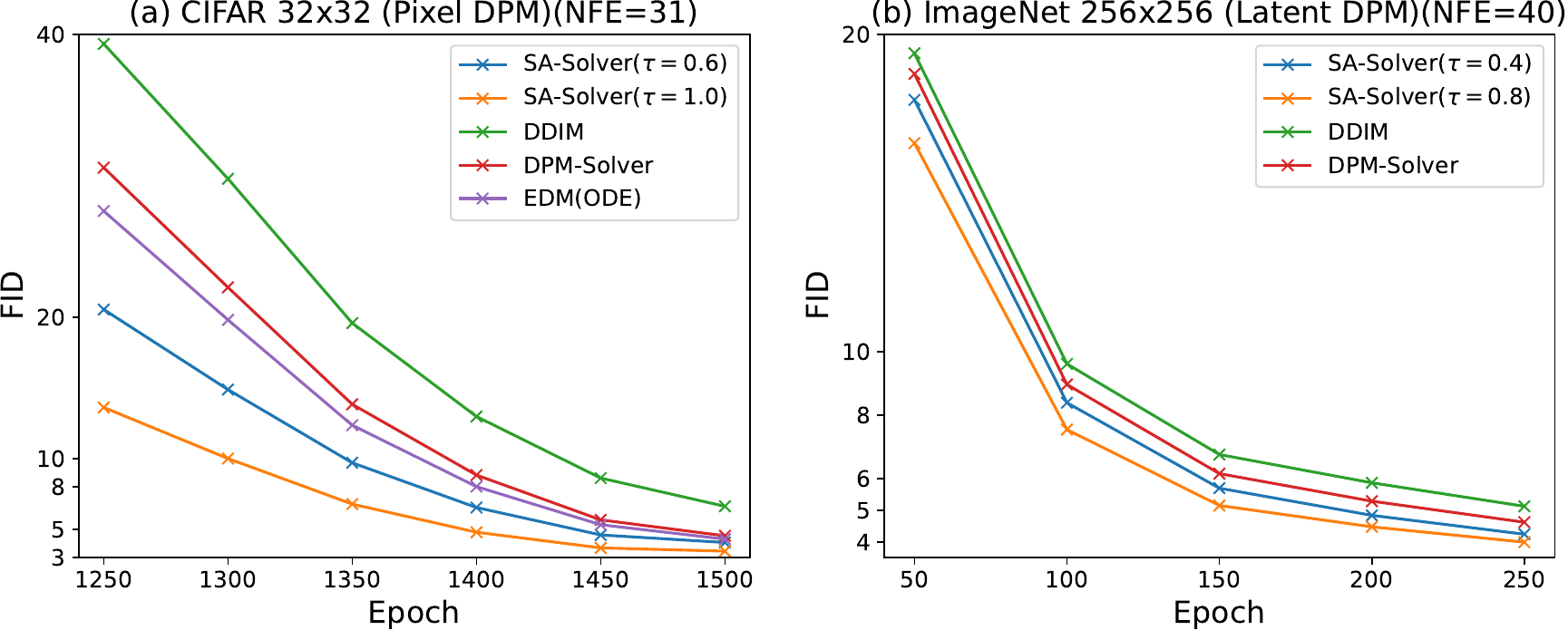}
\caption{Sampling quality measured by FID $\downarrow$ of different sampling methods of DPMs under different training epochs.}
\label{fig:ins}
\end{figure}

\subsection{Comparison with State-of-the-Art}
\label{sec: sota result}

We compare \textit{SA-Solver} with existing state-of-the-art sampling methods, including DDIM~\citep{song2020denoising}, DPM-Solver~\citep{lu2022dpmsa}, UniPC~\citep{zhao2023unipc}, Heun sampler and stochastic sampler in EDM~\citep{Karras2022edm}. Unless otherwise specified, the methods are tested using the default hyper-parameters in the original papers or code.

\textbf{Results on CIFAR10 32x32 and ImageNet 64x64.}
We use the EDM~\citep{Karras2022edm} baseline-VE model for the CIFAR10 32x32 experiments and the ADM~\citep{dhariwal2021diffusion} model for the ImageNet 64x64 experiments. We use EDM's timesteps selection for all samplers for fair comparisons. EDM introduces a certain type of SDE and a corresponding stochastic sampler, which is used for comparison. The experimental results are shown in Fig. \ref{fig:sotas}(a-b). It can be seen that the proposed \textit{SA-Solver} consistently outperforms other samplers and achieves state-of-the-art FID results. It should be noticed for EDM samplers, we report its optimal result which is searched over four hyper-parameters. In fact, at 95 NFEs, \textit{SA-Solver} can achieve the best FID value of 2.63 in CIFAR10 and 1.81 in ImageNet 64x64 which outperforms all other samplers.

\textbf{Results on ImageNet 256x256 and 512x512.}
We evaluate with two classifier-free guidance models, one is the UNet-based latent diffusion model~\citep{Rombach_2022_CVPR} in which the VQ-4 encoder-decoder model is adopted, and the other is the DiT~\citep{peebles2022scalable} model using Vision Transformer based model with KL-8 encoder-decoder. The corresponding classifier-free guidance scale, namely $s=1.5$, is adopted for evaluation. For ImageNet 256x256 dataset with UNet based latent diffusion model, the results of different samplers are shown in Fig. \ref{fig:sotas}(c). Under a considerable number of steps, \textit{SA-Solver} achieves the best sample quality, in which the FID value is 3.87 with only 20 NFEs and below 3.5 with 40 NFEs or more. While for ODE solvers, the FID values cannot reach below 4, which shows the superiority of the proposed SDE solver.

Table \ref{tab:dit} consists results of current SOTA models in ImageNet 256x256 and 512x512. Note that the (Min-SNR) DiT-XL/2-G models are adopted~\citep{peebles2022scalable,hang2023efficient}. It can be seen clearly that better FID results are achieved compared with baseline solvers used by corresponding methods. We achieve 1.93 FID value in Min-SNR DiT model at ImageNet 256x256, and 2.80 in DiT model at ImageNet 512x512, both of which are state-of-the-art results under existing DPMs.

\textbf{Results of text-to-image generation} Fig. \ref{fig:t2i} shows the qualitative results on text-to-image generation. It can be seen that both UniPC and \textit{SA-Solver} can generate images with more details. Our \textit{SA-Solver} is able to generate more reasonable images with better details.

\subsection{Effect of Stochasticity for Inaccurate Score Estimation}
\label{sec: inaccurate score}

When the training data is not enough or the computational budget is limited, the estimated score is inaccurate. We empirically observed that the stochasticity significantly improve the sample quality under the circumstance. To further investigate this effect, we reproduce the early training stage of EDM~\citep{Karras2022edm} baseline-VE model for the CIFAR10 32x32 dataset and DiT-XL/2~\citep{peebles2022scalable} model for the ImageNet 256x256 dataset. We compare \textit{SA-Solver} with different stochastic level $\tau$ and existing state-of-the-art deterministic sampling methods. We use the same hyper-parameters as the corresponding experiment in section~\ref{sec: sota result}.

Figure~\ref{fig:ins} shows that \textit{SA-Solver} outperforms deterministic sampling methods, especially in the early stage of the training process. Moreover, larger $\tau$ value results in better performance. We also conduct a theoretical analysis that stochasticity can mitigate the error of estimation (see Appendix~\ref{appendix: selection of stochasticity}).
\section{Conclusions}
In this paper, we propose an efficient solver named \textit{SA-Solver} for solving Diffusion SDEs, achieving high sampling performance in both minimal steps and a suitable number of steps. To better control the scale of injected noise, we propose Variance Controlled Diffusion SDEs based on noise scale function $\tau(t)$ and propose the analytic form of the SDEs. Based on Variance Controlled Diffusion SDE, we propose \textit{SA-Solver}, which is derived from the stochastic Adams method and uses exponentially weighted integral and analytical variance to achieve efficient SDE sampling. Meanwhile, \textit{SA-Solver} has the optimal theoretical convergence bound. Experiments show that \textit{SA-Solver} achieves state-of-the-art sampling performance in various pre-trained DPMs models. Moreover, \textit{SA-Solver} achieves superior performance when the score estimation is inaccurate.

Although \textit{SA-Solver} achieves optimal sampling performance, the noise scale $\tau(t)$ selection under different NFEs needs further research. The paper proposes empirical criteria for selecting $\tau(t)$, more in-depth theoretical analysis is still needed.





\bibliography{ref}
\bibliographystyle{IEEEtranN}


\newpage
\appendix

\section{Derivations of Variance Controlled Diffusion SDEs}
\subsection{Proof of Proposition~\ref{prop: controllable variance reverse SDEs}}
\label{appendix: proof of prop4.1}
\paragraph{Proposition 4.1}
For any bounded measurable function $\tau(t):[0,T] \rightarrow \Real$, the following Reverse SDEs 
\begin{equation}
\label{eq: appendix variance controlled sde}
\diff \xx_t = \left[f(t) \xx_t - \left(\frac{1 + \tau^2(t)}{2}\right) g^2(t)\nablaxx\log p_t(\xx_t)\right] \diff t + \tau(t)g(t)  \diff \rbm,\hspace{4mm}\xx_T \sim p_T(\xx_T)
\end{equation}
has the same marginal probability distributions with (\ref{eq: reverse ODE}) and (\ref{eq: reverse SDE}) .
\begin{proof}

Denote $p(\xx, t): \Real^d \times [0,T] \rightarrow \Real^{+}$ as the probability density function of $\xx_t$ at time $t$, thus $p(\xx, T) = p_T(\xx)$. Fokker-Planck equation~\citep{oksendal2013stochastic} determines a Partial Differential Equation (PDE) that $p_t(\xx)$ satisfies:
\begin{equation}
\label{eq: fp equation1}
\begin{aligned}
-\frac{\partial p_t(\xx)}{\partial t} &= - \sum_{i=1}^d \frac{\partial}{\partial x_i} \left[-\left[ f(t)p_t(\xx) x_i-  \left(\frac{1 + \tau^2(t)}{2}\right) g^2(t) p_t(\xx)\frac{\partial \log p_t(\xx)}{\partial x_i}\right]\right]\\
 & \qquad + \frac{1}{2} \sum_{i=1}^d \sum_{j=1}^d \frac{\partial^2}{\partial x_i \partial x_j} \left[ \tau^2(t) g^2(t) \delta_{ij} p_t(\xx) \right]\\
 &= \sum_{i=1}^d \frac{\partial}{\partial x_i} \left[ 
 f(t)p_t(\xx) x_i-  \left(\frac{1 + \tau^2(t)}{2}\right) g^2(t) p_t(\xx)\frac{\partial \log p_t(\xx)}{\partial x_i}\right] \\
 & \qquad + \frac{1}{2} \sum_{i=1}^d  \frac{\partial}{\partial x_i} \left[ \frac{\partial}{\partial x_i} \left[ \tau^2(t) g^2(t) p_t(\xx) \right]\right].
\end{aligned}
\end{equation}

Eq.~\eqref{eq: appendix variance controlled sde} is a reverse-time SDE running from $T$ to $0$, thus there are two additional minus signs in Eq.~\eqref{eq: fp equation1} before term $\frac{\partial p_t(\xx)}{\partial t}$ and term $\left[ 
 f(t)p_t(\xx) x_i-  \left(\frac{1 + \tau^2(t)}{2}\right) g^2(t) p_t(\xx)\frac{\partial \log p_t(\xx)}{\partial x_i}\right]$ compared with vanilla Fokker-Planck equation in general cases. Here $\delta_{ij}$ is the Dirac symbol satisfies $\delta_{ij} = 1$ when $i=j$, otherwise, $\delta_{ij} = 0$. Notice that
\begin{equation}
\label{eq: fp equation2}
 \frac{\partial}{\partial x_i} \left[ \tau^2(t) g^2(t) p_t(\xx) \right] = \tau^2(t) g^2(t)   \frac{\partial}{\partial x_i} p_t(\xx)= \tau^2(t) g^2(t) p_t(\xx)  \frac{\partial \log p_t(\xx) }{\partial x_i} .
\end{equation}
Substituting Eq.~\eqref{eq: fp equation2} into Eq.~\eqref{eq: fp equation1}, we obtain that
\begin{equation}
\label{eq: fp equation3}
\begin{aligned}
-\frac{\partial p_t(\xx)}{\partial t} &=  \sum_{i=1}^d \frac{\partial}{\partial x_i} \left[ 
 f(t)p_t(\xx) x_i-  \left(\frac{1 + \tau^2(t)}{2}\right) g^2(t) p_t(\xx)\frac{\partial \log p_t(\xx)}{\partial x_i}\right] \\
 & \qquad + \frac{1}{2} \sum_{i=1}^d  \frac{\partial}{\partial x_i} \left[\tau^2(t) g^2(t) p_t(\xx)  \frac{\partial \log p_t(\xx) }{\partial x_i}  \right] \\
 &= \sum_{i=1}^d \frac{\partial}{\partial x_i} \left[  f(t)p_t(\xx) x_i-  \frac{1}{2} g^2(t) p_t(\xx)\frac{\partial \log p_t(\xx)}{\partial x_i} \right],
\end{aligned}
\end{equation}
which is independent of $\tau(t)$. With the same initial condition $p(\xx, T) = p_T(\xx)$, the family of Reverse SDEs in Eq.~\eqref{eq: appendix variance controlled sde} have exactly the same evolutions of probability density function because they share the same Fokker-Planck equation. Especially, when $\tau(t) = 0$, Eq.~\eqref{eq: appendix variance controlled sde} degenerates to diffusion ODEs and when $\tau(t) = 1$, Eq.~\eqref{eq: appendix variance controlled sde} degenerates to diffusion SDEs.
\end{proof}

\subsection{Two Reparameterizations and Exact Solution under Exponential Integrator}

In this subsection, we will show the exact solution of SDE in both \textit{data prediction} reparameterization and \textit{noise prediction} reparameterization. The noise term in \textit{data prediction} has smaller variance than \textit{noise prediction} ones, implying the necessity of adopting \textit{data prediction} reparameterization for the SDE sampler. 

\subsubsection{Data Prediction Reparameterization}
After approximating $\nablaxx\log p_t(\xx_t)$ with $\boldsymbol{s}_{\btheta}(\xx_t, t)$ and reparameterizing $\boldsymbol{s}_{\btheta}(\xx_t, t)$ with $-(\xx_{t} - \alpha_{t}\bx_{\btheta}(\bx_{t}, t)) / \sigma_{t}^{2}$, Eq.~\eqref{eq: appendix variance controlled sde} becomes
\begin{equation}
\label{eq: appendix variance controlled sde data para}
\diff \xx_t = \left[f(t) \xx_t + \left(\frac{1 + \tau^2(t)}{2\sigma_{t}}\right) g^2(t)\left(\frac{\xx_{t} - \alpha_{t}\bx_{\btheta}(\bx_{t}, t)}{\sigma_{t}}  \right)\right] \diff t + \tau(t)g(t)  \diff \rbm.
\end{equation}
Applying change-of-variable with log-SNR $\lambda_t = \log(\alpha_t/\sigma_t)$ and substituting the following relationship
\begin{equation}
\label{eq: appendix f g lambda}
f(t) = \frac{\diff \log \alpha_t}{\diff t}, \hspace{4mm} g^2(t) = \frac{\diff \sigma^2_t}{\diff t} -  2 \frac{\diff \log \alpha_t}{\diff t} \sigma^2_t = -2 \sigma^2_t \frac{\diff \lambda_t}{\diff t},
\end{equation}
Eq.~\eqref{eq: appendix variance controlled sde data para} becomes
\begin{equation}
\label{eq: appendix variance controlled sde data para2}
\begin{aligned}
\diff \xx_t &= \left[\frac{\diff \log \alpha_t}{\diff t} \xx_t - \left(1 + \tau^2(t)\right) \left(\xx_{t} - \alpha_{t}\bx_{\btheta}(\bx_{t}, t)  \right) \frac{\diff \lambda_t}{\diff t} \right] \diff t + \tau(t)\sigma_t \sqrt{-2 \frac{\diff \lambda_t}{\diff t}}  \diff \rbm \\
&= \left[\left(\frac{\diff \log \alpha_t}{\diff t} - \left(1 + \tau^2(t)\right)  \frac{\diff \lambda_t}{\diff t}  \right) \xx_t + \left(1 + \tau^2(t)\right)  \alpha_{t}\bx_{\btheta}(\bx_{t}, t)  \frac{\diff \lambda_t}{\diff t} \right] \diff t \\
& \qquad + \tau(t)\sigma_t \sqrt{-2 \frac{\diff \lambda_t}{\diff t}}  \diff \rbm.
\end{aligned}
\end{equation}

\subsubsection{Proof of Proposition~\ref{prop: exact solution data prediction model}}
\label{appendix: proof of prop 4.2}

\paragraph{Proposition 4.2}
Given $\xx_s$ for any time $s > 0$, the solution $\xx_t$ at time $t \in [0,s]$ of Eq. \eqref{eq: reparadata2 controllable variance reverse SDEs} is
\begin{equation}
\label{eq: appendix exact solution data prediction model}
\begin{aligned}
&\xx_t = \frac{\sigma_t}{\sigma_s} e^{-\int_{\lambda_s}^{\lambda_t} \tau^2(\Tilde{\lambda}) \diff \Tilde{\lambda}} \xx_s +  \sigma_t \boldsymbol{F}_{\btheta}(s,t) + \sigma_t \boldsymbol{G}(s,t), \\
&\boldsymbol{F}_{\btheta}(s,t) = \int_{\lambda_s}^{\lambda_t} e^{-\int_{\lambda}^{\lambda_t} \tau^2(\Tilde{\lambda}) \diff \Tilde{\lambda} }\left(1+\tau^2\left(\lambda\right)\right) e^{\lambda} \xx_{\btheta}\left(\xx_{\lambda}, \lambda\right) \diff \lambda \\
&\boldsymbol{G}(s,t) = \int_s^t e^{-\int_{\lambda_u}^{\lambda_t} \tau^2(\Tilde{\lambda}) \diff \Tilde{\lambda} } \tau(u) \sqrt{-2\frac{\diff \lambda_u}{\diff u}} \diff \bar{\boldsymbol{w}}_u,
\end{aligned}
\end{equation}
where $\boldsymbol{G}(s,t)$ is an \textit{Itô} integral \citep{oksendal2013stochastic} with the special property 
\begin{equation}
 \label{eq: appendix Analytical variance of reverse SDEs data}
\sigma_t  \boldsymbol{G}(s,t) \sim \Normal\Bigl(\boldzero, \sigma_t^2 \bigl( 1 - e^{-2\int_{\lambda_{s}}^{\lambda_t} \tau^2(\Tilde{\lambda}) \diff \Tilde{\lambda}} \bigl) \Bigl).
\end{equation}
\begin{proof}
Define $\boldsymbol{y}_t = e^{- \int_{t_0}^t \left(\frac{\diff \log \alpha_v}{\diff v} - \left(1 + \tau^2(v)\right)  \frac{\diff \lambda_v}{\diff v}  \right) \diff v}  \xx_t $, where $t_0 \in [0,T]$ is a constant. Differentiate $\boldsymbol{y}_t$ with respect to $t$, we get
\begin{equation}
\label{eq: appendix diff y_t}
\begin{aligned}
\diff \boldsymbol{y}_t & = -\left(\frac{\diff \log \alpha_t}{\diff t} - \left(1 + \tau^2(t)\right)  \frac{\diff \lambda_t}{\diff t}  \right)   e^{- \int_{t_0}^t \left(\frac{\diff \log \alpha_v}{\diff v} - \left(1 + \tau^2(v)\right)  \frac{\diff \lambda_v}{\diff v}  \right) \diff v}  \xx_t \diff t\\
& \qquad +  e^{- \int_{t_0}^t \left(\frac{\diff \log \alpha_v}{\diff v} - \left(1 + \tau^2(v)\right)  \frac{\diff \lambda_v}{\diff v}  \right) \diff v} \diff \xx_t \\
& = -\left(\frac{\diff \log \alpha_t}{\diff t} - \left(1 + \tau^2(t)\right)  \frac{\diff \lambda_t}{\diff t}  \right)   e^{- \int_{t_0}^t \left(\frac{\diff \log \alpha_v}{\diff v} - \left(1 + \tau^2(v)\right)  \frac{\diff \lambda_v}{\diff v}  \right) \diff v}  \xx_t \diff t\\
& \qquad +   e^{- \int_{t_0}^t \left(\frac{\diff \log \alpha_v}{\diff v} - \left(1 + \tau^2(v)\right)  \frac{\diff \lambda_v}{\diff v}  \right) \diff v} \left[\left(\frac{\diff \log \alpha_t}{\diff t} - \left(1 + \tau^2(t)\right)  \frac{\diff \lambda_t}{\diff t}  \right) \xx_t\right] \diff t \\
& \qquad + e^{- \int_{t_0}^t \left(\frac{\diff \log \alpha_v}{\diff v} - \left(1 + \tau^2(v)\right)  \frac{\diff \lambda_v}{\diff v}  \right) \diff v} \left[\left(1 + \tau^2(t)\right)  \alpha_{t}\bx_{\btheta}(\bx_{t}, t)  \frac{\diff \lambda_t}{\diff t} \right] \diff t \\
& \qquad +  e^{- \int_{t_0}^t \left(\frac{\diff \log \alpha_v}{\diff v} - \left(1 + \tau^2(v)\right)  \frac{\diff \lambda_v}{\diff v}  \right) \diff v} \tau(t)\sigma_t \sqrt{-2 \frac{\diff \lambda_t}{\diff t}}  \diff \rbm\\
&=  e^{- \int_{t_0}^t \left(\frac{\diff \log \alpha_v}{\diff v} - \left(1 + \tau^2(v)\right)  \frac{\diff \lambda_v}{\diff v}  \right) \diff v} \left[\left(1 + \tau^2(t)\right)  \alpha_{t}\bx_{\btheta}(\bx_{t}, t)  \frac{\diff \lambda_t}{\diff t} \right] \diff t \\
& \qquad +  e^{- \int_{t_0}^t \left(\frac{\diff \log \alpha_v}{\diff v} - \left(1 + \tau^2(v)\right)  \frac{\diff \lambda_v}{\diff v}  \right) \diff v} \tau(t)\sigma_t \sqrt{-2 \frac{\diff \lambda_t}{\diff t}}  \diff \rbm.
\end{aligned}
\end{equation}
Integrating both sides from $s$ to $t$
\begin{equation}
\label{eq: appendix integrate y_t}
\begin{aligned}
\boldsymbol{y}_t = \boldsymbol{y}_s & + \int_s^t e^{- \int_{t_0}^u \left(\frac{\diff \log \alpha_v}{\diff v} - \left(1 + \tau^2(v)\right)  \frac{\diff \lambda_v}{\diff v}  \right) \diff v} \left[\left(1 + \tau^2(u)\right)  \alpha_{u}\bx_{\btheta}(\bx_{u}, u)  \frac{\diff \lambda_u}{\diff u} \right] \diff u \\
&+ \int_s^t e^{- \int_{t_0}^u \left(\frac{\diff \log \alpha_v}{\diff v} - \left(1 + \tau^2(v)\right)  \frac{\diff \lambda_v}{\diff v}  \right) \diff v} \tau(u)\sigma_u \sqrt{-2 \frac{\diff \lambda_u}{\diff u}}  \diff \bar{\boldsymbol{w}}_u.
\end{aligned}
\end{equation}
Substituting the definition of $\boldsymbol{y}_t$, $\boldsymbol{y}_s$ into Eq.~\eqref{eq: appendix integrate y_t}, we obtain Eq.~\eqref{eq: appendix exact solution data prediction model}
\begin{equation}
\label{eq: appendix integrate x_t}
\begin{aligned}
\boldsymbol{x}_t &= e^{ \int_{s}^t \left(\frac{\diff \log \alpha_u}{\diff u} - \left(1 + \tau^2(u)\right)  \frac{\diff \lambda_u}{\diff u}  \right) \diff u} \boldsymbol{x}_s\\
& \qquad + \int_s^t e^{- \int_{t}^u \left(\frac{\diff \log \alpha_v}{\diff v} - \left(1 + \tau^2(v)\right)  \frac{\diff \lambda_v}{\diff v}  \right) \diff v} \left[\left(1 + \tau^2(u)\right)  \alpha_{u}\bx_{\btheta}(\bx_{u}, u)  \frac{\diff \lambda_u}{\diff u} \right] \diff u \\
& \qquad + \int_s^t e^{- \int_{t}^u \left(\frac{\diff \log \alpha_v}{\diff v} - \left(1 + \tau^2(v)\right)  \frac{\diff \lambda_v}{\diff v}  \right) \diff v} \tau(u)\sigma_u \sqrt{-2 \frac{\diff \lambda_u}{\diff u}}  \diff \bar{\boldsymbol{w}}_u. \\
&=  e^{ \int_{s}^t \left(\frac{\diff \log \sigma_u}{\diff u} - \tau^2(u)  \frac{\diff \lambda_u}{\diff u}  \right) \diff u} \boldsymbol{x}_s\\
& \qquad + \int_s^t e^{- \int_{t}^u \left(\frac{\diff \log \sigma_v}{\diff v} - \tau^2(v)  \frac{\diff \lambda_v}{\diff v}  \right) \diff v} \left[\left(1 + \tau^2(u)\right)  \alpha_{u}\bx_{\btheta}(\bx_{u}, u)  \frac{\diff \lambda_u}{\diff u} \right] \diff u \\
& \qquad + \int_s^t e^{- \int_{t}^u \left(\frac{\diff \log \sigma_v}{\diff v} - \tau^2(v)  \frac{\diff \lambda_v}{\diff v}  \right) \diff v} \tau(u)\sigma_u \sqrt{-2 \frac{\diff \lambda_u}{\diff u}}  \diff \bar{\boldsymbol{w}}_u \\
&= \frac{\sigma_t}{\sigma_s} e^{ -\int_{s}^t   \tau^2(u)  \frac{\diff \lambda_u}{\diff u} \diff u} \boldsymbol{x}_s\\
& \qquad + \int_s^t  \frac{\sigma_t}{\sigma_u} e^{- \int_{u}^t  \tau^2(v)  \frac{\diff \lambda_v}{\diff v}  \diff v} \left[\left(1 + \tau^2(u)\right)  \alpha_{u}\bx_{\btheta}(\bx_{u}, u)  \frac{\diff \lambda_u}{\diff u} \right] \diff u \\
& \qquad + \int_s^t \frac{\sigma_t}{\sigma_u} e^{- \int_{u}^t  \tau^2(v)  \frac{\diff \lambda_v}{\diff v}  \diff v} \tau(u)\sigma_u \sqrt{-2 \frac{\diff \lambda_u}{\diff u}}  \diff \bar{\boldsymbol{w}}_u. \\
&= \frac{\sigma_t}{\sigma_s} e^{ -\int_{\lambda_s}^{\lambda_t}   \tau^2(\Tilde{\lambda})   \diff \Tilde{\lambda}} \boldsymbol{x}_s\\
& \qquad + \sigma_t \int_{\lambda_s}^{\lambda_t}  e^{- \int_{\lambda}^{\lambda_t}  \tau^2(\Tilde{\lambda})  \diff \Tilde{\lambda}}  \left(1 + \tau^2(\lambda)\right) e^{\lambda} \bx_{\btheta}(\bx_{\lambda}, \lambda)   \diff \lambda \\
& \qquad + \sigma_t \int_s^t  e^{- \int_{u}^t  \tau^2(v)  \frac{\diff \lambda_v}{\diff v}  \diff v} \tau(u) \sqrt{-2 \frac{\diff \lambda_u}{\diff u}}  \diff \bar{\boldsymbol{w}}_u. \\
\end{aligned}
\end{equation}
The last term $\sigma_t \int_s^t  e^{- \int_{u}^t  \tau^2(v)  \frac{\diff \lambda_v}{\diff v}  \diff v} \tau(u) \sqrt{-2 \frac{\diff \lambda_u}{\diff u}}  \diff \bar{\boldsymbol{w}}_u$ of Eq.~\eqref{eq: appendix integrate x_t} is the \textit{Itô} integral term. It follows a Gaussian distribution, which can be directly derived from two basic facts~\citep{oksendal2013stochastic}: first, the definition of \textit{Itô} integral is the limitation in $L^2$ space; second, the limit of Gaussian Process in $L^2$ space is still Gaussian. Then we can compute the mean as:
\begin{equation}
\label{eq: appendix calculate mean}
\Expectation \left[\sigma_t \int_s^t  e^{- \int_{u}^t  \tau^2(v)  \frac{\diff \lambda_v}{\diff v}  \diff v} \tau(u) \sqrt{-2 \frac{\diff \lambda_u}{\diff u}}  \diff \bar{\boldsymbol{w}}_u\right] = 0
\end{equation}
and the variance is
\begin{equation}
\label{eq: appendix calculate variance}
\begin{aligned}
&\mathrm{Var}  \left[\sigma_t \int_s^t  e^{- \int_{u}^t  \tau^2(v)  \frac{\diff \lambda_v}{\diff v}  \diff v} \tau(u) \sqrt{-2 \frac{\diff \lambda_u}{\diff u}}  \diff \bar{\boldsymbol{w}}_u\right] \\
=& \Expectation \left[\left(\sigma_t \int_s^t  e^{- \int_{u}^t  \tau^2(v)  \frac{\diff \lambda_v}{\diff v}  \diff v} \tau(u) \sqrt{-2 \frac{\diff \lambda_u}{\diff u}}  \diff \bar{\boldsymbol{w}}_u \right)^2\right] - \left( \Expectation \left[\sigma_t \int_s^t  e^{- \int_{u}^t  \tau^2(v)  \frac{\diff \lambda_v}{\diff v}  \diff v} \tau(u) \sqrt{-2 \frac{\diff \lambda_u}{\diff u}}  \diff \bar{\boldsymbol{w}}_u \right] \right)^2 \\
=& \Expectation \left[\left(\sigma_t \int_t^s  e^{- \int_{u}^t  \tau^2(v)  \frac{\diff \lambda_v}{\diff v}  \diff v} \tau(u) \sqrt{-2 \frac{\diff \lambda_u}{\diff u}}  \diff \boldsymbol{w}_u \right)^2\right] - 0 \\
=& \sigma_t^2 \Expectation \left[  \int_t^s \left( e^{- \int_{u}^t  \tau^2(v)  \frac{\diff \lambda_v}{\diff v}  \diff v} \tau(u) \sqrt{-2 \frac{\diff \lambda_u}{\diff u}} \right)^2 \diff u \right] \\
=& \sigma_t^2  \int_t^s  e^{- \int_{u}^t  2 \tau^2(v)  \frac{\diff \lambda_v}{\diff v}  \diff v} \tau^2(u) \left(-2 \frac{\diff \lambda_u}{\diff u} \right) \diff u  \\
=& \sigma_t^2  \int_{\lambda_s}^{\lambda_t}  2e^{- \int_{\lambda}^{\lambda_t}  2 \tau^2(\Tilde{\lambda})  \diff \Tilde{\lambda} } \tau^2(\lambda) \diff \lambda\\
\end{aligned}
\end{equation}
The expectation equals zero because the Itô integral is a martingale~\citep{oksendal2013stochastic}. The computation of variance uses the \textit{Itô Isometry}, which is a crucial fact of Itô integral. We can further simplify the result by using the change of variable $P(\lambda) = e^{\int_{\lambda_t}^{\lambda} 2\tau^2(\Tilde{\lambda})\diff \Tilde{\lambda}}$.
\begin{equation}
\label{eq: appendix calculate mean variance}
\begin{aligned}
&\mathrm{Var}   \left[\sigma_t \int_s^t  e^{- \int_{u}^t  \tau^2(v)  \frac{\diff \lambda_v}{\diff v}  \diff v} \tau(u) \sqrt{-2 \frac{\diff \lambda_u}{\diff u}}  \diff \bar{\boldsymbol{w}}_u\right] \\
= & \sigma_t^2  \int_{\lambda_s}^{\lambda_t}  2e^{- \int_{\lambda}^{\lambda_t}  2 \tau^2(\Tilde{\lambda})  \diff \Tilde{\lambda} } \tau^2(\lambda) \diff \lambda\\
= & \sigma_t^2   \int_{P(\lambda_s)}^{P(\lambda_t)} P(\lambda) \frac{\diff P(\lambda)}{P(\lambda)}  \\
= & \sigma_t^2 \left(P(\lambda_t) - P(\lambda_s)\right) \\
= & \sigma_t^2 (1 - e^{-2 \int_{\lambda_s}^{\lambda_t} \tau^2(\lambda) \diff \lambda})
\end{aligned}
\end{equation}
\end{proof}

\subsubsection{Noise Prediction Reparameterization}

After approximating $\nablaxx\log p_t(\xx_t)$ with $\boldsymbol{s}_{\btheta}(\xx_t, t)$ and reparameterizing $\boldsymbol{s}_{\btheta}(\xx_t, t)$ with $- \boldsymbol{\epsilon}_{\btheta}(\bx_{t}, t)) / \sigma_{t}$, Eq.~\eqref{eq: appendix variance controlled sde} becomes
\begin{equation}
\label{eq: appendix variance controlled sde noise para}
\diff \xx_t = \left[f(t) \xx_t + \left(\frac{1 + \tau^2(t)}{2\sigma_{t}}\right) g^2(t)\boldsymbol{\epsilon}_{\btheta}(\bx_{t}, t)\right] \diff t + \tau(t)g(t)  \diff \rbm.
\end{equation}
Applying change-of-variable with log-SNR $\lambda_t = \log(\alpha_t/\sigma_t)$ and substituting the following relationship
\begin{equation}
\label{eq: appendix noise f g lambda}
f(t) = \frac{\diff \log \alpha_t}{\diff t}, \hspace{4mm} g^2(t) = \frac{\diff \sigma^2_t}{\diff t} -  2 \frac{\diff \log \alpha_t}{\diff t} \sigma^2_t = -2 \sigma^2_t \frac{\diff \lambda_t}{\diff t},
\end{equation}
Eq.~\eqref{eq: appendix variance controlled sde noise para} becomes
\begin{equation}
\label{eq: appendix variance controlled sde noise para2}
\begin{aligned}
\diff \xx_t & = \left[\frac{\diff \log \alpha_t}{\diff t} \xx_t - \left(1 + \tau^2(t)\right) \sigma_t \boldsymbol{\epsilon}_{\btheta}(\bx_{t}, t) \frac{\diff \lambda_t}{\diff t} \right] \diff t + \tau(t)\sigma_t \sqrt{-2 \frac{\diff \lambda_t}{\diff t}}  \diff \rbm. \\
\end{aligned}
\end{equation}

Eq \eqref{eq: appendix variance controlled sde noise para} is the formulation of \textit{noist prediction model}. Similar with Proposition \ref{prop: exact solution data prediction model}, Eq. \eqref{eq: appendix variance controlled sde noise para2} can be solved analytically, which is shown in the following propositon:
\begin{proposition}
Given $\xx_s$ for any time $s > 0$, the solution $\xx_t$ at time $t \in [0,s]$ of \eqref{eq: appendix variance controlled sde noise para2} is
\begin{equation}
\label{eq: appendix exact solution noise prediction model}
\begin{aligned}
&\xx_t = \frac{\alpha_t}{\alpha_s}  \boldsymbol{x}_s +  \alpha_t \boldsymbol{F}_{\btheta}(s,t) + \alpha_t \boldsymbol{G}(s,t), \\
&\boldsymbol{F}_{\btheta}(s,t) = \int_{\lambda_s}^{\lambda_t} e^{-\lambda} \left(1 + \tau^2(\lambda)\right)  \boldsymbol{\epsilon}_{\btheta}(\bx_{\lambda}, \lambda) \diff \lambda \\
&\boldsymbol{G}(s,t) = \int_s^t e^{-\lambda_u} \tau(u) \sqrt{-2 \frac{\diff \lambda_u}{\diff u}}  \diff \bar{\boldsymbol{w}}_u,
\end{aligned}
\end{equation}
where $\boldsymbol{G}(s,t)$ is an Itô integral \citep{oksendal2013stochastic} with the special property 
\begin{equation}
 \label{eq: appendix Analytical variance of reverse SDEs noise}
\alpha_t  \boldsymbol{G}(s,t) \sim \Normal\Bigl(\boldzero, \alpha_t^2  \int_{\lambda_s}^{\lambda_t}  2e^{-2\lambda} \tau^2(\lambda) \diff \lambda \Bigl).
\end{equation}
\end{proposition}

\begin{proof}
Define $\boldsymbol{y}_t = e^{- \int_{t_0}^t \frac{\diff \log \alpha_v}{\diff v}  \diff v}  \xx_t $, where $t_0 \in [0,T]$ is a constant. Differentiate $\boldsymbol{y}_t$ with respect to $t$, we get
\begin{equation}
\label{eq: appendix diff y_t noise}
\begin{aligned}
\diff \boldsymbol{y}_t & = -\frac{\diff \log \alpha_t}{\diff t}   e^{- \int_{t_0}^t \frac{\diff \log \alpha_v}{\diff v}  \diff v}  \xx_t \diff t + e^{- \int_{t_0}^t \frac{\diff \log \alpha_v}{\diff v}  \diff v}  \diff \xx_t \\
& = -\frac{\diff \log \alpha_t}{\diff t}   e^{- \int_{t_0}^t \frac{\diff \log \alpha_v}{\diff v}  \diff v}  \xx_t \diff t + e^{- \int_{t_0}^t \frac{\diff \log \alpha_v}{\diff v}  \diff v}  \frac{\diff \log \alpha_t}{\diff t} \xx_t  \diff t\\
& - e^{- \int_{t_0}^t \frac{\diff \log \alpha_v}{\diff v}  \diff v}  \left(1 + \tau^2(t)\right) \sigma_t \boldsymbol{\epsilon}_{\btheta}(\bx_{t}, t) \frac{\diff \lambda_t}{\diff t} \diff t  + e^{- \int_{t_0}^t \frac{\diff \log \alpha_v}{\diff v}  \diff v} \tau(t)\sigma_t \sqrt{-2 \frac{\diff \lambda_t}{\diff t}}  \diff \rbm \\
&= - e^{- \int_{t_0}^t \frac{\diff \log \alpha_v}{\diff v}  \diff v}  \left(1 + \tau^2(t)\right) \sigma_t \boldsymbol{\epsilon}_{\btheta}(\bx_{t}, t) \frac{\diff \lambda_t}{\diff t} \diff t  + e^{- \int_{t_0}^t \frac{\diff \log \alpha_v}{\diff v}  \diff v} \tau(t)\sigma_t \sqrt{-2 \frac{\diff \lambda_t}{\diff t}}  \diff \rbm.
\end{aligned}
\end{equation}
Integrating both sides from $s$ to $t$
\begin{equation}
\label{eq: appendix integrate y_t noise}
\begin{aligned}
\boldsymbol{y}_t = \boldsymbol{y}_s & - \int_s^t e^{- \int_{t_0}^u \frac{\diff \log \alpha_v}{\diff v}  \diff v}  \left(1 + \tau^2(u)\right) \sigma_u \boldsymbol{\epsilon}_{\btheta}(\bx_{u}, u) \frac{\diff \lambda_u}{\diff u} \diff u \\
&+ \int_s^t e^{- \int_{t_0}^u \frac{\diff \log \alpha_v}{\diff v}  \diff v} \tau(u)\sigma_u \sqrt{-2 \frac{\diff \lambda_u}{\diff u}}  \diff \bar{\boldsymbol{w}}_u.
\end{aligned}
\end{equation}
Substituting the definition of $\boldsymbol{y}_t$, $\boldsymbol{y}_s$ into Eq.~\eqref{eq: appendix integrate y_t noise}, we obtain
\begin{equation}
\label{eq: appendix integrate x_t noise}
\begin{aligned}
\boldsymbol{x}_t &= e^{ \int_{s}^t \frac{\diff \log \alpha_u}{\diff u}  \diff u} \boldsymbol{x}_s  + \int_s^t e^{- \int_{t}^u \frac{\diff \log \alpha_v}{\diff v}  \diff v} \left(1 + \tau^2(u)\right) \sigma_u \boldsymbol{\epsilon}_{\btheta}(\bx_{u}, u) \frac{\diff \lambda_u}{\diff u} \diff u \\
& \qquad + \int_s^t e^{- \int_{t}^u \frac{\diff \log \alpha_v}{\diff v}  \diff v} \tau(u)\sigma_u \sqrt{-2 \frac{\diff \lambda_u}{\diff u}}  \diff \bar{\boldsymbol{w}}_u \\
&= \frac{\alpha_t}{\alpha_s}  \boldsymbol{x}_s  + \int_s^t \frac{\alpha_t}{\alpha_u} \left(1 + \tau^2(u)\right) \sigma_u \boldsymbol{\epsilon}_{\btheta}(\bx_{u}, u) \frac{\diff \lambda_u}{\diff u} \diff u + \int_s^t \frac{\alpha_t}{\alpha_u} \tau(u)\sigma_u \sqrt{-2 \frac{\diff \lambda_u}{\diff u}}  \diff \bar{\boldsymbol{w}}_u \\
&= \frac{\alpha_t}{\alpha_s}  \boldsymbol{x}_s  + \alpha_t \int_{\lambda_s}^{\lambda_t} e^{-\lambda} \left(1 + \tau^2(\lambda)\right)  \boldsymbol{\epsilon}_{\btheta}(\bx_{\lambda}, \lambda) \diff \lambda + \alpha_t \int_s^t e^{-\lambda_u} \tau(u) \sqrt{-2 \frac{\diff \lambda_u}{\diff u}}  \diff \bar{\boldsymbol{w}}_u.
\end{aligned}
\end{equation}
The Itô integral term $\alpha_t \int_s^t e^{-\lambda_u} \tau(u) \sqrt{-2 \frac{\diff \lambda_u}{\diff u}}  \diff \bar{\boldsymbol{w}}_u$ follows a Gaussian distribution. Following the derivation in Proposition \ref{prop: exact solution data prediction model}, the mean of the \textit{Itô} integral term is:
\begin{equation}
\label{eq: appendix calculate mean noise}
\Expectation \left[\alpha_t \int_s^t e^{-\lambda_u} \tau(u) \sqrt{-2 \frac{\diff \lambda_u}{\diff u}}  \diff \bar{\boldsymbol{w}}_u\right] = 0
\end{equation}
and the expectation is
\begin{equation}
\label{eq: appendix calculate variance noise}
\begin{aligned}
&\mathrm{Var}   \left[\alpha_t \int_s^t e^{-\lambda_u} \tau(u) \sqrt{-2 \frac{\diff \lambda_u}{\diff u}}  \diff \bar{\boldsymbol{w}}_u\right] \\
= &\Expectation \left[\left(\alpha_t \int_s^t e^{-\lambda_u} \tau(u) \sqrt{-2 \frac{\diff \lambda_u}{\diff u}}  \diff \bar{\boldsymbol{w}}_u \right)^2\right] 
-  \left( \Expectation \left[\alpha_t \int_s^t e^{-\lambda_u} \tau(u) \sqrt{-2 \frac{\diff \lambda_u}{\diff u}}  \diff \bar{\boldsymbol{w}}_u \right] \right)^2 \\
= &\Expectation \left[\left(\alpha_t \int_t^s  e^{-\lambda_u} \tau(u) \sqrt{-2 \frac{\diff \lambda_u}{\diff u}}  \diff \boldsymbol{w}_u \right)^2\right] - 0 \\
= & \alpha_t^2 \Expectation \left[  \int_t^s  \left( e^{-\lambda_u} \tau(u) \sqrt{-2 \frac{\diff \lambda_u}{\diff u}} \right)^2 \diff u \right] \\
= & \alpha_t^2  \int_t^s  e^{-2\lambda_u} \tau^2(u) \diff u \left( -2 \frac{\diff \lambda_u}{\diff u} \right) \diff u \\
= & \alpha_t^2  \int_{\lambda_s}^{\lambda_t}  2e^{-2\lambda} \tau^2(\lambda) \diff \lambda\\
\end{aligned}
\end{equation}
\end{proof}

\subsubsection{Comparison between Data and Noise Reparameterizations}
\label{appendix: data or noise reparameterization}
In Table~\ref{tab:data_prediction} we perform an ablation study on data and noise reparameterizations, the experiment results show that under the same magnitude of stochasticity, the proposed \textit{SA-Solver} in data reparameterization has a better convergence which leads to better FID results under the same NFEs. In this subsection, we provide a theoretical view of this phenomenon.

\begin{corollary}
For any bounded measurable function $\tau(t)$, the following inequality holds
\begin{equation}
\label{eq: appendix variance inequality data noise}
\sigma_t^2 \left(1 - e^{-2 \int_{\lambda_s}^{\lambda_t} \tau^2(\Tilde{\lambda}) \diff \Tilde{\lambda}} \right) \leq \alpha_t^2  \int_{\lambda_s}^{\lambda_t}  2e^{-2\lambda} \tau^2(\lambda) \diff \lambda.
\end{equation}    
\end{corollary}
\begin{proof}
It's equivalent to show that
\begin{equation}
1 - e^{-2 \int_{\lambda_s}^{\lambda_t} \tau^2(\Tilde{\lambda}) \diff \Tilde{\lambda}} \leq e^{2\lambda_t}  \int_{\lambda_s}^{\lambda_t}  2e^{-2\lambda} \tau^2(\lambda) \diff \lambda.
\end{equation}
From the basic inequality $1 - e^{-x} \leq x$, we have
\begin{equation}
1 - e^{-2 \int_{\lambda_s}^{\lambda_t} \tau^2(\Tilde{\lambda}) \diff \Tilde{\lambda}} \leq  2 \int_{\lambda_s}^{\lambda_t} \tau^2(\lambda) \diff \lambda.
\end{equation}
Thus it's sufficient to show that
\begin{equation}
e^{2\lambda_t}  \int_{\lambda_s}^{\lambda_t}  2e^{-2\lambda} \tau^2(\lambda) \diff \lambda \geq 2 \int_{\lambda_s}^{\lambda_t} \tau^2(\lambda) \diff \lambda,
\end{equation}
which is true because
\begin{equation}
\int_{\lambda_s}^{\lambda_t}  2\left(e^{2\left(\lambda_t-\lambda\right)} - 1\right) \tau^2(\lambda) \diff \lambda \geq 0,
\end{equation}

\end{proof}
This corollary indicates that the same SDE under two different reparameterizations has different properties under the effect of the exponential integrator. Specifically, in the numerical scheme, the data reparameterization will inject smaller noise in each step's updation. We speculate that this is the reason that the data reparameterization has a better convergence, shown as in Table~\ref{tab:data_prediction}.

\section{Derivations and Proofs for \textit{SA-Solver}}
\label{appendix: Derivations and Proofs for SA-Solver}
\subsection{Preliminary}
We will first review some basic concepts and formulas in the numerical solutions of SDEs~\citep{kloeden1992Numerical}. Suppose we have an \textit{Itô} SDE $\diff \xx_t = f(\xx_t, t) \diff t + g(\xx_t, t) \diff \bm$ and time steps $\left\{ t_i \right\}_{i=0}^{M}, t_i \in [0, T]$ to numerically solve the SDE. For a random variable $Z$, we define the $L_1$ norm $\left\| Z \right\|_{L_1} = \Expectation \left[ |Z| \right]$, the $L_2$ norm $\left\| Z \right\|_{L_2} = \Expectation \left[ |Z|^2 \right]^\frac{1}{2}$, where $|\cdot|$ is the Euclidean norm. Denote $h = \underset{1 \leq i \leq M}{\max} \left( t_i - t_{i-1} \right)$. 
\begin{definition}
We shall say that a time-discrete approximation $\xx_0, \cdots, \xx_M$, where $\xx_i$ is a numerical approximation of $\xx_{t_i}$, converges strongly with order $\gamma > 0$, if there exists a positive constant $C$, which does not depend on $h$ and a $h_0 > 0$ such that
\begin{equation}
\label{eq: appendix strong order}
\underset{0 \leq i \leq M}{\max}  \left\| \xx_{t_i} - \xx_i \right \|_{L_1}  \leq C h^\gamma, \hspace{4mm} \forall\ h \leq h_0.
\end{equation}  
\end{definition}
\begin{definition}
We say it is mean-square convergent with order $\gamma > 0$, if there exists a positive constant $C$, which does not depend on $h$ and a $h_0 > 0$ such that
\begin{equation}
\label{eq: appendix strong order mean square}
\underset{0 \leq i \leq M}{\max}  \left\| \xx_{t_i} - \xx_i \right \|_{L_2}  \leq C h^\gamma, \hspace{4mm} \forall\ h \leq h_0.
\end{equation}    
\end{definition}
\begin{remark}
To prove the strong convergence order $\gamma$ of a numerical scheme, it's sufficient to show the mean-square convergence order $\gamma$. This is from \textit{Hölder Inequality} $  \Expectation \left[ |Z| \right] \leq \Expectation \left[ |Z|^2 \right]^{\frac{1}{2}}\Expectation\left[ |1|^2 \right]^{\frac{1}{2}} \leq \Expectation \left[ |Z|^2 \right]^{\frac{1}{2}} $. Thus $\underset{0 \leq i \leq M}{\max}  \left\| \xx_{t_i} - \xx_i \right \|_{L_1} \leq \underset{0 \leq i \leq M}{\max}  \left\| \xx_{t_i} - \xx_i \right \|_{L_2}$.
\end{remark}

We also need the following definition and assumptions, which usually holds in practical diffusion models.
\begin{definition}
A function $h: \Real^d \times [0,T] \rightarrow \Real^d$ satisfies a linear growth condition if there exists a constant $K$ such that
\begin{equation}
\label{eq: appendix linear growth condition}
|h(\xx, t)| \leq K (1+|\xx|^2)^{\frac{1}{2}}
\end{equation}
\end{definition}
\begin{assumption}
\label{assump1}
The data prediction model $\xx_{\btheta}$ and its derivatives such as $\partial_t \xx_{\btheta}$, $\nablaxx \xx_{\btheta}$ and $\Delta \xx_{\btheta}$ satisfy the linear growth condition.
\end{assumption}
\begin{assumption}
\label{assump2}
The data prediction model $\xx_{\btheta}$ satisfies a uniform Lipschitz condition with respect to $\xx$
\begin{equation}
\label{eq: appendix lipschitz}
|\xx_{\btheta}(\xx_1, t) - \xx_{\btheta}(\xx_2, t)| \leq L|\xx_1 - \xx_2|, \hspace{4mm}\forall x, y \in \Real^d, t\in[0,T]
\end{equation}
\end{assumption}

\subsection{Outline of the Proof}
In the remaining part of this section, we will focus on our variance controlled SDE
\begin{equation}
\label{eq: appendix controllable variance reverse SDEs}
\begin{aligned}
\diff \xx_t =& \left[\left(\frac{\diff \log \alpha_t}{\diff t} - \left(1 + \tau^2(t)\right)  \frac{\diff \lambda_t}{\diff t}  \right) \xx_t + \left(1 + \tau^2(t)\right)  \alpha_{t}\bx_{\btheta}(\bx_{t}, t)  \frac{\diff \lambda_t}{\diff t} \right] \diff t \\
&\qquad + \tau(t)\sigma_t \sqrt{-2 \frac{\diff \lambda_t}{\diff t}}  \diff \rbm.
\end{aligned}
\end{equation}

Consider the general case of the numerical scheme as follows:
\begin{equation}
\label{eq: appendix sde scheme}
\begin{aligned}
   \xx_{i + 1} =& \frac{\sigma_{t_{i + 1}}}{\sigma_{t_{i}}} e^{-\int_{\lambda_{t_i}}^{\lambda_{t_{i+1}}} \tau^2(\lambda_u) \diff \lambda_u} \xx_{i} + \sum_{j=-1}^{s-1} b_{i-j} \xx_{\btheta}(\xx_{i-j}, t_{i-j})\\
   &+ \sigma_{t_{i + 1}} \int_{t_i}^{t_{i+1}}  e^{- \int_{u}^{t_{i + 1}}  \tau^2(\lambda)  \diff \lambda} \tau(u) \sqrt{-2 \frac{\diff \lambda_u}{\diff u}}  \diff \bar{\boldsymbol{w}}_u. 
\end{aligned}
\end{equation}
in which Eq. \eqref{s-step SA-Corrector} and Eq \eqref{s-step SA-Predictor} are the special case of this scheme. We will provide proof of the mean-square convergence order of the numerical scheme $\underset{0 \leq i \leq M}{\max}  \left\| \xx_{t_i} - \xx_i \right \|_{L_2}$.
We define the local error of the numerical scheme Eq.~\eqref{eq: appendix sde scheme} for the approximation of the SDE Eq.~\eqref{eq: appendix controllable variance reverse SDEs} as
\begin{equation}
\label{eq: appendix local error}
\begin{aligned}
L_{i+1} = \xx_{t_{i + 1}} - \xx_{i+1} = \xx_{t_{i + 1}} &- \frac{\sigma_{t_{i + 1}}}{\sigma_{t_{i}}} e^{-\int_{\lambda_{t_i}}^{\lambda_{t_{i+1}}} \tau^2(\lambda) \diff \lambda} \xx_{t_{i}} - \sum_{j=-1}^{s-1} b_{i-j} \xx_{\btheta}(\xx_{t_{i-j}}, t_{i-j})\\
&- \sigma_{t_{i + 1}} \int_{t_i}^{t_{i+1}}  e^{- \int_{u}^{t_{i + 1}}  \tau^2(\lambda)  \diff \lambda} \tau(u) \sqrt{-2 \frac{\diff \lambda_u}{\diff u}}  \diff \bar{\boldsymbol{w}}_u.
\end{aligned}
\end{equation}
$L_{i+1}$ can be decomposed into $R_{i+1}$ and $S_{i+1}$. Then the mean-square convergence can be derived, which is summarized in the following theorem proved by~\citep{2006multistepsde}:

\begin{theorem}[\citep{2006multistepsde}, Theorem 1]
\label{thm: appendix basic convergence}
The mean-square convergent of $\xx_i$ is bounded by
\begin{equation}
\label{eq: appendix stability inequality}
\underset{0 \leq i \leq M}{\max}  \left\| \xx_{t_i} - \xx_i \right \|_{L_2} \leq S\left\{ \underset{0 \leq i \leq s-1}{\max}  \left\| D_i \right \|_{L_2} + \underset{s \leq i \leq M}{\max} \left(      
 \frac{\left\| R_i \right \|_{L_2}}{h} + \frac{\left\| S_i \right \|_{L_2}}{h^{\frac{1}{2}}} \right) \right\}.
\end{equation}
\end{theorem}

In Eq. \eqref{eq: appendix stability inequality}, $D_i, i = 0,\cdots,s-1$ are the initial error which we do not consider. Given Theorem \ref{thm: appendix basic convergence}, to show the convergence order $\mathcal{O}(\underset{0 \leq t \leq T}{\max}\tau(t) h + h^s)$ of our $s$-step \textit{SA-Predictor} and the convergence order $\mathcal{O}(\underset{0 \leq t \leq T}{\max}\tau(t) h + h^{s+1})$ of our $s$-step \textit{SA-Corrector}, we just need to prove the following lemmas.

\begin{lemma}[Convergence rate of $s$-step \textit{SA-Predictor}]
\label{lem: appendix s-step SA-Predictor}
For
\begin{equation}
\label{appendix s-step SA-Predictor}
\begin{aligned}
\xx_{t_{i + 1}} &= \frac{\sigma_{t_{i + 1}}}{\sigma_{t_{i}}} e^{-\int_{\lambda_{t_i}}^{\lambda_{t_{i+1}}} \tau^2(\lambda_u) \diff \lambda_u} \xx_{t_i} + \sum_{j=0}^{s-1} b_{i-j} \xx_{\btheta}(\xx_{t_{i-j}}, t_{i-j}) + \Tilde{\sigma}_i \bxi, \hspace{4mm} \bxi \sim \Normal(\boldzero, \Identity), \\
\Tilde{\sigma}_i &= \sigma_{t_{i + 1}}\sqrt{1 - e^{-2\int_{\lambda_{t_i}}^{\lambda_{t_{i + 1}}} \tau^2(\lambda) \diff \lambda}  } \\
b_{i-j} &= \sigma_{t_{i + 1}}   \int_{\lambda_{t_i}}^{\lambda_{t_{i+1}}} e^{-\int_{\lambda_u}^{\lambda_{t_{i + 1}}} \tau^2(\lambda_v) \diff \lambda_v }\left(1+\tau^2\left(\lambda_u\right)\right) e^{\lambda_u}    l_{i-j}(\lambda_u) \diff \lambda_u,\hspace{4mm} \forall\ 0\leq j \leq s-1
\end{aligned}
\end{equation}
There exists an decomposition of local error $L_i$ such that $L_i = R_i + S_i$ and
\begin{equation}
\left\| R_i \right \|_{L_2} \leq h^{s+1}, \left\| S_i \right \|_{L_2} \leq \underset{0 \leq t \leq T}{\max}\tau(t) h^{\frac{3}{2}},
\end{equation}
\end{lemma}

\begin{lemma}[Convergence rate of $s$-step \textit{SA-Corrector}]
\label{lem: appendix s-step SA-Corrector}
For
\begin{equation}
\label{appendix s-step SA-Corrector}
\begin{aligned}
\xx_{t_{i + 1}} &= \frac{\sigma_{t_{i + 1}}}{\sigma_{t_{i}}} e^{-\int_{\lambda_{t_i}}^{\lambda_{t_{i+1}}} \tau^2(\lambda_u) \diff \lambda_u} \xx_{t_i} +\hat{b}_{i+1} \xx_{\btheta}(\xx_{t_{i+1}}^{p}, t_{i+1}) + \sum_{j=0}^{\hat{s}-1} \hat{b}_{i-j} \xx_{\btheta}(\xx_{t_{i-j}}, t_{i-j}) + \Tilde{\sigma}_i \bxi,  \\
\Tilde{\sigma}_i &= \sigma_{t_{i + 1}}\sqrt{1 - e^{-2\int_{\lambda_{t_i}}^{\lambda_{t_{i + 1}}} \tau^2(\lambda) \diff \lambda}  } \\
\hat{b}_{i-j} &= \sigma_{t_{i + 1}}   \int_{\lambda_{t_i}}^{\lambda_{t_{i+1}}} e^{-\int_{\lambda_u}^{\lambda_{t_{i + 1}}} \tau^2(\lambda_v) \diff \lambda_v }\left(1+\tau^2\left(\lambda_u\right)\right) e^{\lambda_u}    \hat{l}_{i-j}(\lambda_u) \diff \lambda_u,\hspace{4mm} \forall\ 0\leq j \leq s-1 \\
\hat{b}_{i+1} &= \sigma_{t_{i + 1}}   \int_{\lambda_{t_i}}^{\lambda_{t_{i+1}}} e^{-\int_{\lambda_u}^{\lambda_{t_{i + 1}}} \tau^2(\lambda_v) \diff \lambda_v }\left(1+\tau^2\left(\lambda_u\right)\right) e^{\lambda_u}    \hat{l}_{i+1}(\lambda_u) \diff \lambda_u
\end{aligned}
\end{equation}
There exists an decomposition of local error $L_i$ such that $L_i = R_i + S_i$ and
\begin{equation}
\left\| R_i \right \|_{L_2} \leq h^{s+2}, \left\| S_i \right \|_{L_2} \leq \underset{0 \leq t \leq T}{\max}\tau(t) h^{\frac{3}{2}},
\end{equation}
\end{lemma}

Lemma \ref{lem: appendix s-step SA-Predictor} and \ref{lem: appendix s-step SA-Corrector} will be proved in Sec. \ref{sec: appendix proof of convergence}.



\subsection{Lemmas for the Proof}

To better analyze the local error here, we state the following definitions and results from~\citep{2007improvedmultistepsde}. For a continuous function $y:\Real^d \times [0,T] \rightarrow \Real^d$, a general multiple Wiener integral over the subinterval $[t, t+h] \subset [0,T]$ is given by
\begin{equation}
\label{eq: appendix multiple wiener integral}
I_{r_1r_2\cdots r_j}^{t,t+h}(y) = \int_t^{t+h} \int_t^{s_1}\cdots\int_t^{s_{j-1}} y(\xx_{s_j}, s_j) \diff w_{r_1}(s_j)\cdots\diff w_{r_j}(s_1),
\end{equation}
where $r_i \in \left\{0,1,\cdots,d\right\}$ and $\diff w_0(s) = \diff s$. Then we have the following lemma.
\begin{lemma}[Bound of Wiener Integral]
\label{lemma: appendix lemma}
For any function $y:\Real^d \times [0,T] \rightarrow \Real^d$ that satisfies a growth condition in the form $|y(\xx,t)|\leq K(1+|\xx|^2)^{\frac{1}{2}}$, for any $\xx \in \Real^d$, and any $t \in [0,T]$, $h >0$ such that $t+h \in [0,T]$, we have that
\begin{equation}
\label{eq: appendix lemma eq1}
\Expectation\left[ I_{r_1r_2\cdots r_j}^{t,t+h}(y) | \mathcal{F}_t \right] = 0\hspace{4mm}\textit{if $r_i\neq0$ for some $i\in \left\{1,\cdots,j\right\}$},
\end{equation}
\begin{equation}
\label{eq: appendix lemma eq2}
\left\| I_{r_1r_2\cdots r_j}^{t,t+h}(y) \right\|_{L_2} = \mathcal{O}\left(h^{l_1 + \frac{l_2}{2}}\right),
\end{equation}
where $l_1$ is the number of zero indices and $l_2$ is the number of non-zero indices $r_i$.
\end{lemma}
\begin{lemma}[Property of Lagrange interpolation polynomial]
\label{lemma: appendix lemma lagrange}
For $s+1$ points $(t_{i+1}, y_{i+1}), (t_i, y_i), \cdots, (t_{i-(s-1)}, y_{i-(s-1)})$, the Lagrange interpolation polynomial is
\begin{equation}
\label{eq: appendix Lagrange interpolation}
L(t) = \sum_{k=i-(s-1)}^{i+1} l_k(t) y_k.
\end{equation}
Then the following s+1 equalities hold
\begin{equation}
\label{eq: appendix Lagrange equality}
\begin{aligned}
\sum_{k=i-(s-1)}^{i+1} l_k(u) &= 1,  \\
\sum_{k=i-(s-1)}^{i+1} l_k(u)\int_{t_{i-(s-1)}}^{t_{k}}  \diff u_2 &= \int_{t_{i-(s-1)}}^{u}  \diff u_2, \\
\vdots \\
\sum_{k=i-(s-1)}^{i+1} l_k(u)\int_{t_{i-(s-1)}}^{t_{k}} \int_{t_{i-(s-1)}}^{u_2} \cdots \int_{t_{i-(s-1)}}^{u_s} \diff u_{s+1}\cdots \diff u_3 \diff u_2 &= \\ 
\int_{t_{i-(s-1)}}^{u} \int_{t_{i-(s-1)}}^{u_2} & \cdots \int_{t_{i-(s-1)}}^{u_s} \diff u_{s+1}\cdots \diff u_3 \diff u_2 \\
\end{aligned}
\end{equation}
\end{lemma}
\begin{proof}

For the first equality, consider $y_k \equiv 1$ for $i-(s-1)\leq k \leq i+1$. The Lagrange interpolation polynomial for these $y_k$s is a constant function $L(t) \equiv 1$. We have $L(u) = \sum_{k=i-(s-1)}^{i+1} l_k(u) = 1$.

For the second equality, consider $y_k = \int_{t_{i-(s-1)}}^{t_{k}} \diff u_2$. The Lagrange interpolation polynomial for these $y_k$s is a polynomial of degree 1 $L(t) = t - t_{i-(s-1)}$. We have $L(u) = \sum_{k=i-(s-1)}^{i+1} l_k(u) \int_{t_{i-(s-1)}}^{t_{k}}  \diff u_2 = u - t_{i-(s-1)} = \int_{t_{i-(s-1)}}^{u}  \diff u_2$.

For equalities from the third to the last, without loss of generality, we prove the $p-th$ equality, where $3\leq p \leq s+1$. Consider $y_k = \int_{t_{i-(s-1)}}^{t_{k}} \int_{t_{i-(s-1)}}^{u_2} \cdots \int_{t_{i-(s-1)}}^{u_{p-1}} \diff u_{p}\cdots \diff u_3 \diff u_2 $. The Lagrange interpolation polynomial for these $y_k$s is a polynomial of degree $p-1$ $L(t) = \int_{t_{i-(s-1)}}^{t} \int_{t_{i-(s-1)}}^{u_2} \cdots \int_{t_{i-(s-1)}}^{u_{p-1}} \diff u_{p}\cdots \diff u_3 \diff u_2$. We have $L(u) = \int_{t_{i-(s-1)}}^{u} \int_{t_{i-(s-1)}}^{u_2} \cdots \int_{t_{i-(s-1)}}^{u_{p-1}} \diff u_{p}\cdots \diff u_3 \diff u_2 =  \int_{t_{i-(s-1)}}^{u} \int_{t_{i-(s-1)}}^{u_2} \cdots \int_{t_{i-(s-1)}}^{u_{p-1}} \diff u_{p}\cdots \diff u_3 \diff u_2$.
\end{proof}

\subsection{Proof of Lemma \ref{lem: appendix s-step SA-Predictor} (for Theorem.~\ref{thm: coefficients and convergence of predictor}) and Lemma \ref{lem: appendix s-step SA-Corrector} (for Theorem.~\ref{thm: coefficients and convergence of corrector})}
\label{sec: appendix proof of convergence}

To simplify the notation, we will introduce two operators which will appear in the Itô formula. Suppose we have an \textit{Itô} SDE $\diff \xx_t = f(\xx_t, t) \diff t + g(\xx_t, t) \diff \bm$ and $h(\xx,t)$ is a twice continuously differentiable function. Let $\Gamma_0(\cdot) = \partial_t (\cdot) + \nablaxx (\cdot) f$, $\Gamma_1(\cdot) = \frac{g^2}{2}\Delta(\cdot)$ and $\Gamma_2(\cdot) = \nablaxx(\cdot)g$ in which $\nablaxx$ is the Jacobian matrix, and $\Delta$ is the Laplacian operator. With the notation here, we can express the Itô formula for $h(\xx,t)$ as
\begin{equation}
\label{eq: appendix simplify notation}
h(\xx_t, t) = h(\xx_s, s) + \int_s^t \left( \Gamma_0(h) + \Gamma_1(h) \right) \diff t + \int_s^t \Gamma_2(h) \diff \bar{\boldsymbol{w}}_t.
\end{equation}

Given the above lemmas, we will analyze the local error $L_{i+1}$ step by step. Inspired by Theorem \ref{thm: appendix basic convergence}, for data-prediction reparameterization model, $L_{i+i}$ can be estimated by decomposing the terms step by step. The first step of decomposition is summarized as the following lemma:

\begin{lemma}[First step of estimating local error $L_{i+1}$ in data-prediction reparameterization model]
\label{lem: first step decomposition}
Given the exact solution of data prediction model
\begin{equation}
\begin{aligned}
&\xx_t = \frac{\sigma_t}{\sigma_s} e^{-\int_{\lambda_s}^{\lambda_t} \tau^2(\Tilde{\lambda}) \diff \Tilde{\lambda}} \xx_s +  \sigma_t \boldsymbol{F}_{\btheta}(s,t) + \sigma_t \boldsymbol{G}(s,t), \\
&\boldsymbol{F}_{\btheta}(s,t) = \int_{\lambda_s}^{\lambda_t} e^{-\int_{\lambda}^{\lambda_t} \tau^2(\Tilde{\lambda}) \diff \Tilde{\lambda} }\left(1+\tau^2\left(\lambda\right)\right) e^{\lambda} \xx_{\btheta}\left(\xx_{\lambda}, \lambda\right) \diff \lambda \\
&\boldsymbol{G}(s,t) = \int_s^t e^{-\int_{\lambda_u}^{\lambda_t} \tau^2(\Tilde{\lambda}) \diff \Tilde{\lambda} } \tau(u) \sqrt{-2\frac{\diff \lambda_u}{\diff u}} \diff \bar{\boldsymbol{w}}_u,
\end{aligned}
\end{equation}
With proper $b_k, k \in [i-(s-1), i+1]$, The local error $L_{i+1}$ in Eq. \eqref{eq: appendix local error} is 
\begin{equation}
L_{i+1} = R^{(1)}_{i+1} + S^{(1)}_{i+1}
\end{equation}
where
\begin{equation}
\label{eq: lem appendix remaining local error}
\begin{aligned}
S^{(1)}_{i+1} =& \mathcal{O}\left(\underset{0 \leq t \leq T}{\max}\tau(t) h^{\frac{3}{2}} \right) \\
R^{(1)}_{i+1}=& \sum_{k = i-(s-1)}^{i-1} \sigma_{t_{i + 1}}\left( \int_{t_i}^{t_{i+1}} e^{- \int_{\lambda_u}^{\lambda_{t_{i+1}}}  \tau^2(\lambda)  \diff \lambda}  \left(1 + \tau^2(u)\right) e^{\lambda_u}  \frac{\diff \lambda_u}{ \diff u}  \diff u \right) \times\\
&\quad\left(\int_{t_k}^{t_{k+1}} \Gamma_0(\xx_{\btheta}) \diff t\right) \\
&+ \sigma_{t_{i + 1}} \int_{t_i}^{t_{i+1}} e^{- \int_{\lambda_u}^{\lambda_{t_{i+1}}}  \tau^2(\lambda)  \diff \lambda}  \left(1 + \tau^2(u)\right) e^{\lambda_u} \left( \int_{t_i}^{u} \Gamma_0(\xx_{\btheta}) \diff t \right) \frac{\diff \lambda_u}{ \diff u}  \diff u \\
&- \sum_{j=-1}^{s-1} b_{i-j}  \sum_{k = i-(s-1)}^{i-j-1} \int_{t_k}^{t_{k+1}} \Gamma_0(\xx_{\btheta}) \diff t.
\end{aligned}
\end{equation}
\end{lemma}

\begin{proof}
The difference between Eq. \eqref{eq: appendix sde scheme} and Eq. \eqref{eq: appendix local error} is that $\xx_{j}$ is our numerical approximation, while $\xx_{t_j}$ is the exact solution of SDE Eq.~\eqref{eq: appendix controllable variance reverse SDEs} at time $t = t_j$. Substitute the exact solution Eq.~\eqref{eq: appendix integrate x_t} of $\xx_{t_{i + 1}}$, we have
\begin{equation}
\label{eq: appendix local error2}
\begin{aligned}
L_{i+1} = & \frac{\sigma_{t_{i + 1}}}{\sigma_{t_{i}}} e^{ -\int_{\lambda_{t_i}}^{\lambda_{t_{i + 1}}}   \tau^2(\lambda)  \diff \lambda} \boldsymbol{x}_{t_{i}} + \sigma_{t_{i + 1}} \int_{t_i}^{t_{i+1}}  e^{- \int_{u}^{t_{i + 1}}  \tau^2(\lambda)  \diff \lambda} \tau(u) \sqrt{-2 \frac{\diff \lambda_u}{\diff u}}  \diff \bar{\boldsymbol{w}}_u \\
&  + \sigma_{t_{i + 1}} \int_{\lambda_{t_i}}^{\lambda_{t_{i+1}}}  e^{- \int_{\lambda_u}^{\lambda_{t_{i+1}}}  \tau^2(\lambda)  \diff \lambda}  \left(1 + \tau^2(\lambda_u)\right) e^{\lambda_u} \bx_{\btheta}(\bx_{\lambda_u}, \lambda_u)   \diff \lambda_u\\
&- \sigma_{t_{i + 1}} \int_{t_i}^{t_{i+1}}  e^{- \int_{u}^{t_{i + 1}}  \tau^2(\lambda)  \diff \lambda} \tau(u) \sqrt{-2 \frac{\diff \lambda_u}{\diff u}}  \diff \bar{\boldsymbol{w}}_u\\
&- \frac{\sigma_{t_{i + 1}}}{\sigma_{t_{i}}} e^{-\int_{\lambda_{t_i}}^{\lambda_{t_{i+1}}} \tau^2(\lambda_u) \diff \lambda_u} \xx_{t_{i}} - \sum_{j=-1}^{s-1} b_{i-j} \xx_{\btheta}(\xx_{t_{i-j}}, t_{i-j})\\
= & \sigma_{t_{i + 1}} \int_{t_i}^{t_{i+1}} e^{- \int_{\lambda_u}^{\lambda_{t_{i+1}}}  \tau^2(\lambda)  \diff \lambda}  \left(1 + \tau^2(u)\right) e^{\lambda_u} \bx_{\btheta}(\bx_{u}, u) \frac{\diff \lambda_u}{ \diff u}  \diff u \\
&- \sum_{j=-1}^{s-1} b_{i-j} \xx_{\btheta}(\xx_{t_{i-j}}, t_{i-j}).
\end{aligned}
\end{equation}
Let $f(\xx,t) = \left(\frac{\diff \log \alpha_t}{\diff t} - \left(1 + \tau^2(t)\right)  \frac{\diff \lambda_t}{\diff t}  \right) \xx + \left(1 + \tau^2(t)\right)  \alpha_{t}\bx_{\btheta}(\bx, t)  \frac{\diff \lambda_t}{\diff t}$, $g(t) = \tau(t)\sigma_t \sqrt{-2 \frac{\diff \lambda_t}{\diff t}}$.
By Itô's formula~\citep{oksendal2013stochastic}, we have
\begin{equation}
\label{eq: appendix ito formula 1}
\begin{aligned}
&\xx_{\btheta}(\xx_{u}, u) = \xx_{\btheta}(\xx_{t_{i-(s-1)}}, t_{i-(s-1)}) + \sum_{k = i-(s-1)}^{i-1} \int_{t_k}^{t_{k+1}} \left( \Gamma_0(\xx_{\btheta}) + \Gamma_1(\xx_{\btheta}) \right) \diff t\\
&+ \int_{t_i}^{u} \left( \Gamma_0(\xx_{\btheta}) + \Gamma_1(\xx_{\btheta}) \right) \diff t  + \sum_{k = i-(s-1)}^{i-1} \int_{t_k}^{t_{k+1}} \Gamma_2(\xx_{\btheta}) \diff \bar{\boldsymbol{w}}_t + \int_{t_i}^{u} \Gamma_2(\xx_{\btheta}) \diff \bar{\boldsymbol{w}}_t,
\end{aligned}
\end{equation}
\begin{equation}
\label{eq: appendix ito formula 2}
\begin{aligned}
&\xx_{\btheta}(\xx_{t_{i-j}}, t_{i-j}) = \xx_{\btheta}(\xx_{t_{i-(s-1)}}, t_{i-(s-1)}) + \sum_{k = i-(s-1)}^{i-j-1} \int_{t_k}^{t_{k+1}} \left( \Gamma_0(\xx_{\btheta}) + \Gamma_1(\xx_{\btheta}) \right) \diff t\\
&+ \sum_{k = i-(s-1)}^{i-j-1} \int_{t_k}^{t_{k+1}} \Gamma_2(\xx_{\btheta}) \diff \bar{\boldsymbol{w}}_t,
\end{aligned}
\end{equation}

Substituting Eq.~\eqref{eq: appendix ito formula 1} and Eq.~\eqref{eq: appendix ito formula 2} into Eq.~\eqref{eq: appendix local error2}, we have

\begin{equation}
\label{eq: appendix local error3}
\begin{aligned}
&L_{i+1}\\
=& \left(\sigma_{t_{i + 1}} \int_{t_i}^{t_{i+1}} e^{- \int_{\lambda_u}^{\lambda_{t_{i+1}}}  \tau^2(\lambda)  \diff \lambda}  \left(1 + \tau^2(u)\right) e^{\lambda_u}  \frac{\diff \lambda_u}{ \diff u}  \diff u - \sum_{j=-1}^{s-1} b_{i-j} \right) \xx_{\btheta}(\xx_{t_{i-(s-1)}}, t_{i-(s-1)})\\
&+ \sum_{k = i-(s-1)}^{i-1} \sigma_{t_{i + 1}}\left( \int_{t_i}^{t_{i+1}} e^{- \int_{\lambda_u}^{\lambda_{t_{i+1}}}  \tau^2(\lambda)  \diff \lambda}  \left(1 + \tau^2(u)\right) e^{\lambda_u}  \frac{\diff \lambda_u}{ \diff u}  \diff u \right) \times\\
&\quad\left(\int_{t_k}^{t_{k+1}} \left( \Gamma_0(\xx_{\btheta}) + \Gamma_1(\xx_{\btheta}) \right) \diff t\right) \\
&+ \sum_{k = i-(s-1)}^{i-1} \sigma_{t_{i + 1}}\left( \int_{t_i}^{t_{i+1}} e^{- \int_{\lambda_u}^{\lambda_{t_{i+1}}}  \tau^2(\lambda)  \diff \lambda}  \left(1 + \tau^2(u)\right) e^{\lambda_u}  \frac{\diff \lambda_u}{ \diff u}  \diff u \right) \times\left(  \int_{t_k}^{t_{k+1}} \Gamma_2(\xx_{\btheta}) \diff \bar{\boldsymbol{w}}_t \right) \\
&+ \sigma_{t_{i + 1}} \int_{t_i}^{t_{i+1}} e^{- \int_{\lambda_u}^{\lambda_{t_{i+1}}}  \tau^2(\lambda)  \diff \lambda}  \left(1 + \tau^2(u)\right) e^{\lambda_u} \left( \int_{t_i}^{u} \left( \Gamma_0(\xx_{\btheta}) + \Gamma_1(\xx_{\btheta}) \right) \diff t \right) \frac{\diff \lambda_u}{ \diff u}  \diff u \\
&+ \sigma_{t_{i + 1}} \int_{t_i}^{t_{i+1}} e^{- \int_{\lambda_u}^{\lambda_{t_{i+1}}}  \tau^2(\lambda)  \diff \lambda}  \left(1 + \tau^2(u)\right) e^{\lambda_u} \left( \int_{t_i}^{u} \Gamma_2(\xx_{\btheta}) \diff \bar{\boldsymbol{w}}_t \right) \frac{\diff \lambda_u}{ \diff u}  \diff u \\
&- \sum_{j=-1}^{s-1} b_{i-j} \left( \sum_{k = i-(s-1)}^{i-j-1} \int_{t_k}^{t_{k+1}} \left( \Gamma_0(\xx_{\btheta}) + \Gamma_1(\xx_{\btheta}) \right) \diff t+ \sum_{k = i-(s-1)}^{i-j-1} \int_{t_k}^{t_{k+1}} \Gamma_2(\xx_{\btheta}) \diff \bar{\boldsymbol{w}}_t \right)
\end{aligned}
\end{equation}
We will divide the local error $L_{i+1}$ into distinct terms. The first term has a coefficient
\begin{equation}
\label{eq: appendix matching coefficient first}
\sigma_{t_{i + 1}} \int_{t_i}^{t_{i+1}} e^{- \int_{\lambda_u}^{\lambda_{t_{i+1}}}  \tau^2(\lambda)  \diff \lambda}  \left(1 + \tau^2(u)\right) e^{\lambda_u}  \frac{\diff \lambda_u}{ \diff u}  \diff u - \sum_{j=-1}^{s-1} b_{i-j}.
\end{equation} 
By Lemma~\ref{lemma: appendix lemma lagrange}, $b_k$ constructed by the integral of Lagrange polynomial in Eq.~\eqref{appendix s-step SA-Predictor} and Eq.~\eqref{appendix s-step SA-Corrector} satisfies $b_k = \mathcal{O}(h)$ and the coefficient~\eqref{eq: appendix matching coefficient first} is zero. Furthermore, we have $g(t) = \tau(t)\sigma_t \sqrt{-2 \frac{\diff \lambda_t}{\diff t}} = \mathcal{O}\left(\underset{0 \leq t \leq T}{\max}\tau(t)\right)$. By Lemma~\ref{lemma: appendix lemma}, we have the following estimations
\begin{equation}
\label{eq: appendix estimation}
\begin{aligned}
&\sum_{k = i-(s-1)}^{i-1} \sigma_{t_{i + 1}}\left( \int_{t_i}^{t_{i+1}} e^{- \int_{\lambda_u}^{\lambda_{t_{i+1}}}  \tau^2(\lambda)  \diff \lambda}  \left(1 + \tau^2(u)\right) e^{\lambda_u}  \frac{\diff \lambda_u}{ \diff u}  \diff u \right) \times \left(\int_{t_k}^{t_{k+1}} \Gamma_1(\xx_{\btheta})  \diff t\right)\\
= &\mathcal{O}\left(\underset{0 \leq t \leq T}{\max}\tau^2(t) h^2 \right),\\
&\sum_{k = i-(s-1)}^{i-1} \sigma_{t_{i + 1}}\left( \int_{t_i}^{t_{i+1}} e^{- \int_{\lambda_u}^{\lambda_{t_{i+1}}}  \tau^2(\lambda)  \diff \lambda}  \left(1 + \tau^2(u)\right) e^{\lambda_u}  \frac{\diff \lambda_u}{ \diff u}  \diff u \right) \times\left( \int_{t_k}^{t_{k+1}} \Gamma_2(\xx_{\btheta}) \diff \bar{\boldsymbol{w}}_t \right)\\
= &\mathcal{O}\left(\underset{0 \leq t \leq T}{\max}\tau(t) h^{\frac{3}{2}} \right),\\
&\sigma_{t_{i + 1}} \int_{t_i}^{t_{i+1}} e^{- \int_{\lambda_u}^{\lambda_{t_{i+1}}}  \tau^2(\lambda)  \diff \lambda}  \left(1 + \tau^2(u)\right) e^{\lambda_u} \left( \int_{t_i}^{u}  \Gamma_1(\xx_{\btheta})  \diff t \right) \frac{\diff \lambda_u}{ \diff u}  \diff u = \mathcal{O}\left(\underset{0 \leq t \leq T}{\max}\tau^2(t) h^{2} \right),\\
&\sigma_{t_{i + 1}} \int_{t_i}^{t_{i+1}} e^{- \int_{\lambda_u}^{\lambda_{t_{i+1}}}  \tau^2(\lambda)  \diff \lambda}  \left(1 + \tau^2(u)\right) e^{\lambda_u} \left( \int_{t_i}^{u} \Gamma_2(\xx_{\btheta}) \diff \bar{\boldsymbol{w}}_t \right) \frac{\diff \lambda_u}{ \diff u}  \diff u  = \mathcal{O}\left(\underset{0 \leq t \leq T}{\max}\tau(t) h^{\frac{3}{2}} \right),\\
&\sum_{j=-1}^{s-1} b_{i-j}  \sum_{k = i-(s-1)}^{i-j-1} \int_{t_k}^{t_{k+1}}  \Gamma_1(\xx_{\btheta})  \diff t = \mathcal{O}\left(\underset{0 \leq t \leq T}{\max}\tau^2(t) h^2 \right),\\
&\sum_{j=-1}^{s-1} b_{i-j} \sum_{k = i-(s-1)}^{i-j-1} \int_{t_k}^{t_{k+1}} \Gamma_2(\xx_{\btheta}) \diff \bar{\boldsymbol{w}}_t = \mathcal{O}\left(\underset{0 \leq t \leq T}{\max}\tau(t) h^{\frac{3}{2}} \right),
\end{aligned}
\end{equation}
and the summation of the above terms is $S^{(1)}_{i+1} = \mathcal{O}\left(\underset{0 \leq t \leq T}{\max}\tau(t) h^{\frac{3}{2}} \right)$.

The remaining terms of local error are
\begin{equation}
\label{eq: appendix remaining local error}
\begin{aligned}
R^{(1)}_{i+1}=& \sum_{k = i-(s-1)}^{i-1} \sigma_{t_{i + 1}}\left( \int_{t_i}^{t_{i+1}} e^{- \int_{\lambda_u}^{\lambda_{t_{i+1}}}  \tau^2(\lambda)  \diff \lambda}  \left(1 + \tau^2(u)\right) e^{\lambda_u}  \frac{\diff \lambda_u}{ \diff u}  \diff u \right) \times \left(\int_{t_k}^{t_{k+1}} \Gamma_0(\xx_{\btheta}) \diff t\right) \\
&+ \sigma_{t_{i + 1}} \int_{t_i}^{t_{i+1}} e^{- \int_{\lambda_u}^{\lambda_{t_{i+1}}}  \tau^2(\lambda)  \diff \lambda}  \left(1 + \tau^2(u)\right) e^{\lambda_u} \left( \int_{t_i}^{u} \Gamma_0(\xx_{\btheta}) \diff t \right) \frac{\diff \lambda_u}{ \diff u}  \diff u \\
&- \sum_{j=-1}^{s-1} b_{i-j}  \sum_{k = i-(s-1)}^{i-j-1} \int_{t_k}^{t_{k+1}} \Gamma_0(\xx_{\btheta}) \diff t,
\end{aligned}
\end{equation}
which completes the proof.
\end{proof}

The remaining problem is to estimate the $R^{(1)}_{i+1}$ in Eq. \eqref{eq: lem appendix remaining local error}. We can further expand the term $\Gamma_0(\xx_{\btheta})$ as following
\begin{equation}
\label{eq: appendix expansion}
\begin{aligned}
&\Gamma_0(\xx_{\btheta})(\xx_t, t) \\
= &\Gamma_0(\xx_{\btheta})(\xx_{t_{i-(s-1)}}, t_{i-(s-1)}) + \int_{t_{i-(s-1)}}^t \left( \Gamma_0\Gamma_0(\xx_{\btheta}) + \Gamma_1\Gamma_0(\xx_{\btheta}) \right) \diff t + \int_{t_{i-(s-1)}}^t  \Gamma_2\Gamma_0(\xx_{\btheta}) \diff \bar{\boldsymbol{w}}_t.
\end{aligned}
\end{equation}

Substituting the expansion of $\Gamma_0(\xx_{\btheta})$, we perform the approximation of $\Tilde{L}_{i+1}$, which is summarized with the following lemma:

\begin{lemma}[Second step of estimating $L_{i+1}$ in data-prediction reparameterization model]
\label{lem: second step of decomposition}
$R^{(1)}_{i+1}$ in Eq. \eqref{eq: lem appendix remaining local error} can be decomposed as
\begin{equation}
R^{(1)}_{i+1} = R^{(2)}_{i+1} + S^{(2)}_{i+1},
\end{equation}
where
\begin{equation}
\label{eq: lem appendix remaining local error3}
\begin{aligned}
& S^{(2)}_{i+1} = \mathcal{O}\left(\underset{0 \leq t \leq T}{\max}\tau(t) h^{\frac{5}{2}} \right) \\
& R^{(2)}_{i+1} \\
=& \sigma_{t_{i + 1}}  \int_{t_i}^{t_{i+1}} e^{- \int_{\lambda_u}^{\lambda_{t_{i+1}}}  \tau^2(\lambda)  \diff \lambda}  \left(1 + \tau^2(u)\right) e^{\lambda_u} \left(\int_{t_{i-(s-1)}}^{u}  \diff u_2\right) \frac{\diff \lambda_u}{ \diff u}  \diff u \cdot \Gamma_0(\xx_{\btheta})(\xx_{t_{i-(s-1)}}, t_{i-(s-1)})\\
+& \sigma_{t_{i + 1}}  \int_{t_i}^{t_{i+1}} e^{- \int_{\lambda_u}^{\lambda_{t_{i+1}}}  \tau^2(\lambda)  \diff \lambda}  \left(1 + \tau^2(u)\right) e^{\lambda_u}\left(\int_{t_{i-(s-1)}}^{u} \int_{t_{i-(s-1)}}^{u_2}  \Gamma_0\Gamma_0(\xx_{\btheta})\left(\xx_{u_3}, u_3\right) \diff u_3 \diff u_2\right) \frac{\diff \lambda_u}{ \diff u}  \diff u  \\
-& \sum_{j=-1}^{s-1} b_{i-j} \int_{t_{i-(s-1)}}^{t_{i-j}}  \diff u_2 \times \Gamma_0(\xx_{\btheta})(\xx_{t_{i-(s-1)}}, t_{i-(s-1)})\\
-& \sum_{j=-1}^{s-1} b_{i-j} \int_{t_{i-(s-1)}}^{t_{i-j}} \int_{t_{i-(s-1)}}^{u_2}  \Gamma_0\Gamma_0(\xx_{\btheta})\left(\xx_{u_3}, u_3\right) \diff u_3 \diff u_2 \\
\end{aligned}    
\end{equation}
\end{lemma}

\begin{proof}
We start with decomposing the term $R^{(1)}_{i+1}$
\begin{equation}
\label{eq: appendix remaining local error2}
\begin{aligned}
& R^{(1)}_{i+1} \\
=& \sigma_{t_{i + 1}}  \int_{t_i}^{t_{i+1}} e^{- \int_{\lambda_u}^{\lambda_{t_{i+1}}}  \tau^2(\lambda)  \diff \lambda}  \left(1 + \tau^2(u)\right) e^{\lambda_u} \left(\int_{t_{i-(s-1)}}^{u} \Gamma_0(\xx_{\btheta})(\xx_{u_2}, u_2) \diff u_2\right) \frac{\diff \lambda_u}{ \diff u}  \diff u \\
&- \sum_{j=-1}^{s-1} b_{i-j} \int_{t_{i-(s-1)}}^{t_{i-j}} \Gamma_0(\xx_{\btheta})(\xx_{u_2}, u_2) \diff u_2\\
=& \sigma_{t_{i + 1}}  \int_{t_i}^{t_{i+1}} e^{- \int_{\lambda_u}^{\lambda_{t_{i+1}}}  \tau^2(\lambda)  \diff \lambda}  \left(1 + \tau^2(u)\right) e^{\lambda_u} \left(\int_{t_{i-(s-1)}}^{u}  \diff u_2\right) \frac{\diff \lambda_u}{ \diff u}  \diff u \cdot \Gamma_0(\xx_{\btheta})(\xx_{t_{i-(s-1)}}, t_{i-(s-1)})\\
&+ \sigma_{t_{i + 1}}  \int_{t_i}^{t_{i+1}} e^{- \int_{\lambda_u}^{\lambda_{t_{i+1}}}  \tau^2(\lambda)  \diff \lambda}  \left(1 + \tau^2(u)\right) e^{\lambda_u}\\
&\quad\left(\int_{t_{i-(s-1)}}^{u} \int_{t_{i-(s-1)}}^{u_2} \left( \Gamma_0\Gamma_0(\xx_{\btheta}) + \Gamma_1\Gamma_0(\xx_{\btheta}) \right)\left(\xx_{u_3}, u_3\right) \diff u_3 \diff u_2\right) \frac{\diff \lambda_u}{ \diff u}  \diff u  \\
&+ \sigma_{t_{i + 1}}  \int_{t_i}^{t_{i+1}} e^{- \int_{\lambda_u}^{\lambda_{t_{i+1}}}  \tau^2(\lambda)  \diff \lambda}  \left(1 + \tau^2(u)\right) e^{\lambda_u}\\
&\quad\left(\int_{t_{i-(s-1)}}^{u} \int_{t_{i-(s-1)}}^{u_2}  \Gamma_2\Gamma_0(\xx_{\btheta})\left(\xx_{u_3}, u_3\right) \diff \bar{\boldsymbol{w}}_{u_3} \diff u_2\right) \frac{\diff \lambda_u}{ \diff u}  \diff u  \\
&- \sum_{j=-1}^{s-1} b_{i-j} \int_{t_{i-(s-1)}}^{t_{i-j}}  \diff u_2 \times \Gamma_0(\xx_{\btheta})(\xx_{t_{i-(s-1)}}, t_{i-(s-1)})\\
&- \sum_{j=-1}^{s-1} b_{i-j} \int_{t_{i-(s-1)}}^{t_{i-j}} \int_{t_{i-(s-1)}}^{u_2} \left( \Gamma_0\Gamma_0(\xx_{\btheta}) + \Gamma_1\Gamma_0(\xx_{\btheta}) \right)\left(\xx_{u_3}, u_3\right) \diff u_3 \diff u_2 \\
&- \sum_{j=-1}^{s-1} b_{i-j} \int_{t_{i-(s-1)}}^{t_{i-j}} \int_{t_{i-(s-1)}}^{u_2}  \Gamma_2\Gamma_0(\xx_{\btheta})\left(\xx_{u_3}, u_3\right) \diff \bar{\boldsymbol{w}}_{u_3} \diff u_2.
\end{aligned}
\end{equation}

We further estimate the terms with $\Gamma_1$ and $\Gamma_2$.
\begin{equation}
\label{eq: appendix estimation2}
\begin{aligned}
&\sigma_{t_{i + 1}}  \int_{t_i}^{t_{i+1}} e^{- \int_{\lambda_u}^{\lambda_{t_{i+1}}}  \tau^2(\lambda)  \diff \lambda}  \left(1 + \tau^2(u)\right) e^{\lambda_u}\\
&\left(\int_{t_{i-(s-1)}}^{u} \int_{t_{i-(s-1)}}^{u_2}  \Gamma_1\Gamma_0(\xx_{\btheta}) \left(\xx_{u_3}, u_3\right) \diff u_3 \diff u_2\right) \frac{\diff \lambda_u}{ \diff u}  \diff u  \\
= &\mathcal{O}\left(\underset{0 \leq t \leq T}{\max}\tau^2(t) h^3 \right),\\
&\sigma_{t_{i + 1}}  \int_{t_i}^{t_{i+1}} e^{- \int_{\lambda_u}^{\lambda_{t_{i+1}}}  \tau^2(\lambda)  \diff \lambda}  \left(1 + \tau^2(u)\right) e^{\lambda_u}\\
&\left(\int_{t_{i-(s-1)}}^{u} \int_{t_{i-(s-1)}}^{u_2}  \Gamma_2\Gamma_0(\xx_{\btheta})\left(\xx_{u_3}, u_3\right) \diff \bar{\boldsymbol{w}}_{u_3} \diff u_2\right) \frac{\diff \lambda_u}{ \diff u}  \diff u  \\
= &\mathcal{O}\left(\underset{0 \leq t \leq T}{\max}\tau(t) h^{\frac{5}{2}} \right),\\
&\sum_{j=-1}^{s-1} b_{i-j} \int_{t_{i-(s-1)}}^{t_{i-j}} \int_{t_{i-(s-1)}}^{u_2}  \Gamma_1\Gamma_0(\xx_{\btheta}) \left(\xx_{u_3}, u_3\right) \diff u_3 \diff u_2 =\mathcal{O}\left(\underset{0 \leq t \leq T}{\max}\tau^2(t) h^{3} \right),\\
&\sum_{j=-1}^{s-1} b_{i-j} \int_{t_{i-(s-1)}}^{t_{i-j}} \int_{t_{i-(s-1)}}^{u_2}  \Gamma_2\Gamma_0(\xx_{\btheta})\left(\xx_{u_3}, u_3\right) \diff \bar{\boldsymbol{w}}_{u_3} \diff u_2 = \mathcal{O}\left(\underset{0 \leq t \leq T}{\max}\tau(t) h^{\frac{5}{2}} \right).
\end{aligned}
\end{equation}

The summation of the above terms is $S^{(2)} = \mathcal{O}\left(\underset{0 \leq t \leq T}{\max}\tau(t) h^{\frac{5}{2}} \right)$. Compared with $S^{(1)}$, this term can be omitted.

The remaining local error is
\begin{equation}
\label{eq: appendix remaining local error3}
\begin{aligned}
R^{(2)}_{i+1}
=& \sigma_{t_{i + 1}}  \int_{t_i}^{t_{i+1}} e^{- \int_{\lambda_u}^{\lambda_{t_{i+1}}}  \tau^2(\lambda)  \diff \lambda}  \left(1 + \tau^2(u)\right) e^{\lambda_u} \left(\int_{t_{i-(s-1)}}^{u}  \diff u_2\right) \frac{\diff \lambda_u}{ \diff u}  \diff u\\
&\times \Gamma_0(\xx_{\btheta})(\xx_{t_{i-(s-1)}}, t_{i-(s-1)})\\
&+ \sigma_{t_{i + 1}}  \int_{t_i}^{t_{i+1}} e^{- \int_{\lambda_u}^{\lambda_{t_{i+1}}}  \tau^2(\lambda)  \diff \lambda}  \left(1 + \tau^2(u)\right) e^{\lambda_u}\\
&\left(\int_{t_{i-(s-1)}}^{u} \int_{t_{i-(s-1)}}^{u_2}  \Gamma_0\Gamma_0(\xx_{\btheta})\left(\xx_{u_3}, u_3\right) \diff u_3 \diff u_2\right) \frac{\diff \lambda_u}{ \diff u}  \diff u  \\
&- \sum_{j=-1}^{s-1} b_{i-j} \int_{t_{i-(s-1)}}^{t_{i-j}}  \diff u_2 \times \Gamma_0(\xx_{\btheta})(\xx_{t_{i-(s-1)}}, t_{i-(s-1)})\\
&- \sum_{j=-1}^{s-1} b_{i-j} \int_{t_{i-(s-1)}}^{t_{i-j}} \int_{t_{i-(s-1)}}^{u_2}  \Gamma_0\Gamma_0(\xx_{\btheta})\left(\xx_{u_3}, u_3\right) \diff u_3 \diff u_2 \\
\end{aligned}
\end{equation}  
which completes the proof.
\end{proof}

\begin{remark}
With Lemma \ref{lem: first step decomposition} and \ref{lem: second step of decomposition}, the local error $L_{i+1}$ can be decomposed to the term $S_{i+1} = S^{(1)}_{i+1} + S^{(2)}_{i+1}$ and the term $R^{(2)}_{i+1}$. It is clear that $S_{i+1} = \mathcal{O}\left(\underset{0 \leq t \leq T}{\max}\tau(t) h^{\frac{3}{2}} \right)$, and we will show that given $b_{i-j}$ constructed by integral of Lagrange polynomial in Eq.~\eqref{appendix s-step SA-Predictor} and Eq.~\eqref{appendix s-step SA-Corrector}, $R^{(2)}_{i+1} = \mathcal{O}\left( h^{3} \right)$. 

By Lemma~\ref{lemma: appendix lemma lagrange}, $b_k$ constructed by the integral of Lagrange polynomial in Eq.~\eqref{appendix s-step SA-Predictor} and Eq.~\eqref{appendix s-step SA-Corrector} satisfies that the coefficient for $\Gamma_0(\xx_{\btheta})(\xx_{t_{i-(s-1)}}, t_{i-(s-1)})$
\begin{equation}
\label{eq: appendix matching coefficient}
    \sigma_{t_{i + 1}}  \int_{t_i}^{t_{i+1}} e^{- \int_{\lambda_u}^{\lambda_{t_{i+1}}}  \tau^2(\lambda)  \diff \lambda}  \left(1 + \tau^2(u)\right) e^{\lambda_u} \left(\int_{t_{i-(s-1)}}^{u}  \diff u_2\right) \frac{\diff \lambda_u}{ \diff u}  \diff u - \sum_{j=-1}^{s-1} b_{i-j} \int_{t_{i-(s-1)}}^{t_{i-j}}  \diff u_2,
\end{equation}
equals zero. And the remaining term in $R^{(2)}_{i+1}$ is $\mathcal{O}(h^3)$.
\end{remark}

\begin{remark}
We will show that the local error can be further decomposed such that $L_{i+1} = R^{(s)}_{i+1} + \sum_{j=1}^{s} S^{(j)}_{i+1}$. In this case $S_{i+1} = \sum_{j=1}^{s} S^{(j)}_{i+1}$ is the term such that $S_{i+1} = \mathcal{O}\left(\underset{0 \leq t \leq T}{\max}\tau(t) h^{\frac{3}{2}} \right)$, and we will show that by our constructed $b_{i-j}$, $R^{(s)}_{i+1} = \mathcal{O}\left( h^{s+1} \right)$. 
\end{remark}

\begin{lemma}[$j-th$ step of estimating $L_{i+1}$ in data-prediction reparameterization model]
\label{lem: jth step of decomposition}
For $j \leq s+1$, $R^{(j-1)}_{i+1}$ in Eq. \eqref{eq: lem appendix remaining local error} can be decomposed as
\begin{equation}
R^{(j-1)}_{i+1} = R^{(j)}_{i+1} + S^{(j)}_{i+1},
\end{equation}
where
\begin{equation}
\label{eq: lem appendix remaining local error4}
\begin{aligned}
& S^{(j)}_{i+1} = \mathcal{O}\left(\underset{0 \leq t \leq T}{\max}\tau(t) h^{\frac{2j+1}{2}} \right) \\
& R^{(j)}_{i+1} \\
=& \sigma_{t_{i + 1}}  \int_{t_i}^{t_{i+1}} e^{- \int_{\lambda_u}^{\lambda_{t_{i+1}}}  \tau^2(\lambda)  \diff \lambda}  \left(1 + \tau^2(u)\right) e^{\lambda_u} \left(\int_{t_{i-(s-1)}}^{u} \int_{t_{i-(s-1)}}^{u_2} \cdots \int_{t_{i-(s-1)}}^{u_{j-1}} \diff u_{j}\cdots \diff u_3 \diff u_2\right) \frac{\diff \lambda_u}{ \diff u}  \diff u\\
&\quad \cdot \overbrace{\Gamma_0\cdots\Gamma_0}^{j-1}(\xx_{\btheta})(\xx_{t_{i-(s-1)}}, t_{i-(s-1)})\\
+& \sigma_{t_{i + 1}}  \int_{t_i}^{t_{i+1}} e^{- \int_{\lambda_u}^{\lambda_{t_{i+1}}}  \tau^2(\lambda)  \diff \lambda} \left(1 + \tau^2(u)\right) e^{\lambda_u}\\
&\quad  \left(\int_{t_{i-(s-1)}}^{u} \int_{t_{i-(s-1)}}^{u_2} \cdots \int_{t_{i-(s-1)}}^{u_{j}}  \overbrace{\Gamma_0\cdots\Gamma_0}^{j}(\xx_{\btheta})\left(\xx_{u_{j+1}}, u_{j+1}\right) \diff u_{j+1}\cdots \diff u_3 \diff u_2\right) \frac{\diff \lambda_u}{ \diff u}  \diff u  \\
-& \sum_{j=-1}^{s-1} b_{i-j} \int_{t_{i-(s-1)}}^{t_{i-j}} \int_{t_{i-(s-1)}}^{u_2} \cdots \int_{t_{i-(s-1)}}^{u_{j-1}} \diff u_{j}\cdots \diff u_3 \diff u_2 \cdot \overbrace{\Gamma_0\cdots\Gamma_0}^{j-1}(\xx_{\btheta})(\xx_{t_{i-(s-1)}}, t_{i-(s-1)})\\
-& \sum_{j=-1}^{s-1} b_{i-j} \int_{t_{i-(s-1)}}^{t_{i-j}} \int_{t_{i-(s-1)}}^{u_2} \cdots \int_{t_{i-(s-1)}}^{u_{j}}  \overbrace{\Gamma_0\cdots\Gamma_0}^{j}(\xx_{\btheta})\left(\xx_{u_{j+1}}, u_{j+1}\right) \diff u_{j+1}\cdots \diff u_3 \diff u_2 \\
\end{aligned}    
\end{equation}
Furthermore, given that $b_k$ is constructed by the integral of Lagrange polynomial in Eq.~\eqref{appendix s-step SA-Predictor} and Eq.~\eqref{appendix s-step SA-Corrector}, $R^{(j)}_{i+1} = \mathcal{O}(h^{j+1})$
\end{lemma}

\paragraph{Sketch of the proof}(1) Use the Itô formula Eq.~\eqref{eq: appendix simplify notation} to expand $\overbrace{\Gamma_0\cdots\Gamma_0}^{j-1}(\xx_{\btheta})$. (2) Use Lemma~\ref{lemma: appendix lemma} to estimate the stochastic term $S^{(j)}$. For the remaining term $R^{(j)}$, by Lemma~\ref{lemma: appendix lemma lagrange}, $b_k$ constructed by the integral of Lagrange polynomial in Eq.~\eqref{appendix s-step SA-Predictor} and Eq.~\eqref{appendix s-step SA-Corrector} satisfies that the coefficient before $\overbrace{\Gamma_0\cdots\Gamma_0}^{j-1}(\xx_{\btheta})(\xx_{t_{i-(s-1)}}, t_{i-(s-1)})$
\begin{equation}
\label{eq: appendix matching coefficient 2}
\begin{aligned}
&\sigma_{t_{i + 1}}  \int_{t_i}^{t_{i+1}} e^{- \int_{\lambda_u}^{\lambda_{t_{i+1}}}  \tau^2(\lambda)  \diff \lambda}  \left(1 + \tau^2(u)\right) e^{\lambda_u} \left(\int_{t_{i-(s-1)}}^{u} \int_{t_{i-(s-1)}}^{u_2} \cdots \int_{t_{i-(s-1)}}^{u_{j-1}} \diff u_{j}\cdots \diff u_3 \diff u_2\right) \frac{\diff \lambda_u}{ \diff u}  \diff u \\
&- \sum_{j=-1}^{s-1} b_{i-j} \int_{t_{i-(s-1)}}^{t_{i-j}} \int_{t_{i-(s-1)}}^{u_2} \cdots \int_{t_{i-(s-1)}}^{u_{j-1}} \diff u_{j}\cdots \diff u_3 \diff u_2.
\end{aligned}
\end{equation}
equals zero. And the remaining term in $R^{(j)}_{i+1}$ is $\mathcal{O}(h^{j+1})$. 

The process can be repeated until the coefficient before $\overbrace{\Gamma_0\cdots\Gamma_0}^{s}(\xx_{\btheta})(\xx_{t_{i-(s-1)}}, t_{i-(s-1)})$ is
\begin{equation}
\label{eq: appendix matching coefficient 3}
\begin{aligned}
&\sigma_{t_{i + 1}}  \int_{t_i}^{t_{i+1}} e^{- \int_{\lambda_u}^{\lambda_{t_{i+1}}}  \tau^2(\lambda)  \diff \lambda}  \left(1 + \tau^2(u)\right) e^{\lambda_u} \left(\int_{t_{i-(s-1)}}^{u} \int_{t_{i-(s-1)}}^{u_2} \cdots \int_{t_{i-(s-1)}}^{u_s} \diff u_{s+1}\cdots \diff u_3 \diff u_2\right) \frac{\diff \lambda_u}{ \diff u}  \diff u \\
&- \sum_{j=-1}^{s-1} b_{i-j} \int_{t_{i-(s-1)}}^{t_{i-j}} \int_{t_{i-(s-1)}}^{u_2} \cdots \int_{t_{i-(s-1)}}^{u_s} \diff u_{s+1}\cdots \diff u_3 \diff u_2.
\end{aligned}
\end{equation}
which equals zero. And the remaining term $R^{s+1}_{i+1}$ is $\mathcal{O}(h^{s+2})$.

We conclude with the proof of Lemma \ref{lem: appendix s-step SA-Predictor} and \ref{lem: appendix s-step SA-Corrector}.

\paragraph{Proof for Lemma \ref{lem: appendix s-step SA-Corrector} (Convergence for $s$-step \textit{SA-Corrector})} The stochastic term $S_{i+1}$ can be estimated as $\mathcal{O}\left(\underset{0 \leq t \leq T}{\max}\tau(t) h^{\frac{3}{2}} \right)$. Lemma~\ref{lemma: appendix lemma lagrange} prove that with $b_{i-j}$ defined in Theorem~\ref{thm: coefficients and convergence of corrector}, the coefficients of Eq.~\eqref{eq: appendix matching coefficient first}, Eq.~\eqref{eq: appendix matching coefficient}, Eq.~\eqref{eq: appendix matching coefficient 2} and Eq.~\eqref{eq: appendix matching coefficient 3} equal zero. Thus the deterministic term $R_{i+1}$ can be estimated as $\mathcal{O}(h^{s+2})$. The proof is completed.

\paragraph{Proof for Lemma \ref{lem: appendix s-step SA-Predictor} (Convergence for $s$-step \textit{SA-Predictor})} The stochastic term $S_{i+1}$ can be estimated as $\mathcal{O}\left(\underset{0 \leq t \leq T}{\max}\tau(t) h^{\frac{3}{2}} \right)$ from Eq.~\eqref{eq: appendix estimation} and Eq.~\eqref{eq: appendix estimation2}. Lemma~\ref{lemma: appendix lemma lagrange} prove that with $b_{i-j}$ defined in Theorem~\ref{thm: coefficients and convergence of predictor}, the coefficients of Eq.~\eqref{eq: appendix matching coefficient first}, Eq.~\eqref{eq: appendix matching coefficient}, Eq.~\eqref{eq: appendix matching coefficient 2} and Eq.~\eqref{eq: appendix matching coefficient 3} equal zero except for the last term. This is because in $s$-step \textit{SA-Predictor} we only have $s$ points in contrast to $s+1$ points in $s$-step \textit{SA-Corrector}, for which we can only obtain the first s equalities in Lemma~\ref{lemma: appendix lemma lagrange}. Thus the deterministic term $R_{i+1}$ can be estimated as $\mathcal{O}(h^{s+1})$. The proof is completed.

\subsection{Relationship with Existing Samplers}

\subsubsection{Relationship with DDIM}
\label{appendix: proof of corollary 5.3}
DDIM~\citep{song2020denoising} generates samples through the following process:
\begin{equation}
\label{eq: appendix DDIM sample}
\xx_{t_{i+1}} = \alpha_{t_{i+1}} \left( \frac{\xx_{t_i} - \sigma_{t_i} \boldsymbol{\epsilon}_{\btheta}(\xx_{t_i}, t_i) }{\alpha_{t_i}} \right) + \sqrt{1 - \alpha_{t_{i+1}}^2 - \hat{\sigma}_{t_i}^2} \boldsymbol{\epsilon}_{\btheta}(\xx_{t_i}, t_i) + \hat{\sigma}_{t_i} \boldsymbol{\xi}, 
\end{equation}
where $\boldsymbol{\xi} \sim \Normal(\boldzero, \Identity)$, $\hat{\sigma}_{t_i}$ is a variable parameter. In practice, DDIM introduces a parameter $\eta$ such that when $\eta = 0$, the sampling process becomes deterministic and when $\eta = 1$, the sampling process coincides with original DDPM~\citep{ho2020denoising}. Specifically, $\hat{\sigma}_{t_i} = \eta \sqrt{\frac{1 - \alpha_{t_{i+1}}^2}{1 - \alpha_{t_{i}}^2} \left( 1 - \frac{\alpha_{t_{i}}^2}{\alpha_{t_{i+1}}^2} \right)}$.
\paragraph{Corollary~\ref{corol: relationship with ddim}}
For any $\eta$ in DDIM, there exists a $\tau_{\eta}(t):\Real \rightarrow \Real$ which is a piecewise constant function such that DDIM-$\eta$ coincides with our $1$-step \textit{SA-Predictor} when $\tau(t) = \tau_{\eta}(t)$ with data parameterization of our variance-controlled diffusion SDE.
\begin{proof}
Our $1$-step \textit{SA-Predictor} when $\tau(t) = \tau, t \in [t_i, t_{i+1}]$ with data parameterization of our variance-controlled diffusion SDE is
\begin{equation}
\label{eq: appendix SA-P1}
\begin{aligned}
\xx_{t_{i+1}} =& \frac{\sigma_{t_{i+1}}}{\sigma_{t_{i}}} e^{-\tau^2\left(\lambda_{t_{i+1}} - \lambda_{t_{i}}\right)} \xx_{t_{i}} + \alpha_{t_{i+1}} \left( 1 - e^{-\left(1+\tau^2\right)\left(\lambda_{t_{i+1}} - \lambda_{t_{i}}\right)} \right) \xx_{\btheta}(\xx_{t_i}, t_i) \\
&+ \sigma_{t_{i+1}} \sqrt{1 - e^{-2\tau^2\left(\lambda_{t_{i+1}} - \lambda_{t_{i}}\right)}}  \xi.
\end{aligned}
\end{equation}
DDIM-$\eta$ generates samples through the following process
\begin{equation}
\label{eq: appendix DDIM sample2}
\xx_{t_{i+1}} = \alpha_{t_{i+1}} \xx_{\btheta}\left(\xx_{t_i}, t_i\right) + \sqrt{1 - \alpha_{t_{i+1}}^2 - \hat{\sigma}_{t_i}^2} \boldsymbol{\epsilon}_{\btheta}(\xx_{t_i}, t_i) + \hat{\sigma}_{t_i} \boldsymbol{\xi}, \hat{\sigma}_{t_i} = \eta \sqrt{\frac{1 - \alpha_{t_{i+1}}^2}{1 - \alpha_{t_{i}}^2} \left( 1 - \frac{\alpha_{t_{i}}^2}{\alpha_{t_{i+1}}^2} \right)}.
\end{equation}
If we substitute $\hat{\sigma}_{t_i}$ with $\sigma_{t_{i+1}} \sqrt{1 - e^{-2\tau^2\left(\lambda_{t_{i+1}} - \lambda_{t_{i}}\right)}}$, we can verify that $\sqrt{1 - \alpha_{t_{i+1}}^2 - \hat{\sigma}_{t_i}^2} = \sigma_{t_{i+1}}e^{-\tau^2 (\lambda_{t_{i+1}} - \lambda_{t_i})}$. The DDIM-$\eta$ then becomes
\begin{equation}
\label{eq: appendix DDIM sample3}
\begin{aligned}
\xx_{t_{i+1}} =& \alpha_{t_{i+1}} \xx_{\btheta}\left(\xx_{t_i}, t_i\right) +\sigma_{t_{i+1}}e^{-\tau^2 \left(\lambda_{t_{i+1}} - \lambda_{t_i}\right)} \left( \frac{\xx_{t_i} - \alpha_{t_i}\xx_{\btheta}(\xx_{t_i}, t_i)}{\sigma_{t_i}} \right) \\
&+ \sigma_{t_{i+1}} \sqrt{1 - e^{-2\tau^2\left(\lambda_{t_{i+1}} - \lambda_{t_{i}}\right)}} \boldsymbol{\xi}\\
=& \frac{\sigma_{t_{i+1}}}{\sigma_{t_i}}e^{-\tau^2 \left(\lambda_{t_{i+1}} - \lambda_{t_i}\right)} \xx_{t_i}   + \left(\alpha_{t_{i+1}} - \frac{\alpha_{t_i}}{\sigma_{t_i}} \sigma_{t_{i+1}}e^{-\tau^2 \left(\lambda_{t_{i+1}} - \lambda_{t_i}\right)} \right) \xx_{\btheta}\left(\xx_{t_i}, t_i\right) \\
&+ \sigma_{t_{i+1}} \sqrt{1 - e^{-2\tau^2\left(\lambda_{t_{i+1}} - \lambda_{t_{i}}\right)}} \boldsymbol{\xi}\\
=& \frac{\sigma_{t_{i+1}}}{\sigma_{t_{i}}} e^{-\tau^2\left(\lambda_{t_{i+1}} - \lambda_{t_{i}}\right)} \xx_{t_{i}} + \alpha_{t_{i+1}} \left( 1 - e^{-\left(1+\tau^2\right)\left(\lambda_{t_{i+1}} - \lambda_{t_{i}}\right)} \right) \xx_{\btheta}(\xx_{t_i}, t_i) \\
&+ \sigma_{t_{i+1}} \sqrt{1 - e^{-2\tau^2\left(\lambda_{t_{i+1}} - \lambda_{t_{i}}\right)}}  \boldsymbol{\xi},
\end{aligned}
\end{equation}
which is exactly the same with our $1$-step \textit{SA-Predictor}. To find the $\tau_{\eta}$, we solve the relationship
\begin{equation}
\label{eq: appendix eta tau relation}
\eta \sqrt{\frac{1 - \alpha_{t_{i+1}}^2}{1 - \alpha_{t_{i}}^2} \left( 1 - \frac{\alpha_{t_{i}}^2}{\alpha_{t_{i+1}}^2} \right)} = \sigma_{t_{i+1}} \sqrt{1 - e^{-2\tau_{\eta}^2\left(\lambda_{t_{i+1}} - \lambda_{t_{i}}\right)}}.
\end{equation}
The relationship between $\tau$ and $\eta$ is
\begin{equation}
\label{eq: appendix eta tau relation3}
\eta = \sigma_{t_i} \sqrt{\frac{1 - e^{-2\tau_{\eta}^2\left(\lambda_{t_{i+1}} - \lambda_{t_{i}}\right)}}{1 - \frac{\alpha_{t_{i}}^2}{\alpha_{t_{i+1}}^2}}}, \tau_{\eta} = \sqrt{\frac{\log \left( 1 - \frac{\eta^2}{\sigma_{t_i}^2} \left( 1 - \frac{\alpha_{t_i}^2}{\alpha_{t_{i+1}}^2} \right) \right)}{-2\left(\lambda_{t_{i+1}} - \lambda_{t_{i}}\right)}}.
\end{equation}
\end{proof}
In a concurrent paper~\citep{lu2023dpmsolver++}, Lu \textit{et al.} prove the result that their SDE-DPM-Solver++1 coincides with DDIM with a special $\eta$. Their result is a special case of Corollary~\ref{corol: relationship with ddim} when $\tau_\eta \equiv 1$ and $\eta$ take a special value, while our result holds for arbitrary $\eta$.
\subsubsection{Relationship with DPM-Solver++(2M)} 
DPM-Solver++~\citep{lu2023dpmsolver++} is a high-order solver which solves diffusion ODEs for guided sampling. DPM-Solver++(2M) is equivalent to the 2-step Adams-Bashforth scheme combined with the exponential integrator. While our 2-step \textit{SA-Predictor} is also equivalent to the 2-step Adams-Bashforth scheme combined with the exponential integrator when $\tau(t) \equiv 0$. Thus DPM-Solver++(2M) is a special case of our $2$-step \textit{SA-Predictor} when $\tau(t) \equiv 0$.
\subsubsection{Relationship with UniPC} 
UniPC~\citep{zhao2023unipc} is a unified predictor-corrector framework for solving diffusion ODEs. Specifically, UniPC-p uses a p-step Adams-Bashforth scheme combined with the exponential integrator as a predictor and a p-step Adams-Moulton scheme combined with the exponential integrator as a corrector. While our p-step \textit{SA-Predictor} is also equivalent to the p-step Adams-Bashforth scheme combined with the exponential integrator when $\tau(t) \equiv 0$ and our p-step \textit{SA-Corrector} is also equivalent to the p-step Adams-Moulton scheme combined with the exponential integrator when $\tau(t) \equiv 0$. Thus UniPC-p is a special case of our \textit{SA-Solver} when $\tau(t) \equiv 0$ with predictor step $p$, corrector step $p$ in Algorithm~\ref{alg: SA-Solver}.


\section{Selection on the Magnitude of Stochasticity}
\label{appendix: selection of stochasticity}

In this section, we will show that we choose $\tau(t) \equiv 1$ in a number of NFEs. We will show that under certain conditions, the upper bound of KL divergence between the marginal distribution and the true distribution can be minimized when $\tau(t) \equiv 1$. 

Let $p_t(\xx)$ denotes the marginal distribution of $\xx_t$, by Proposition~\ref{prop: controllable variance reverse SDEs}, we know that for any bounded measurable function $\tau(t):[0,T]\rightarrow\Real$, the following Reverse SDEs
\begin{equation}
\label{eq: appendix controllable variance reverse SDEs v2}
\diff \xx_t = \left[f(t) \xx_t - \left(\frac{1 + \tau^2(t)}{2}\right) g^2(t)\nablaxx\log p_t(\xx_t)\right] \diff t + \tau(t)g(t)  \diff \rbm,\hspace{4mm}\xx_T \sim p_T(\xx_T),
\end{equation}
have the same marginal probability distributions. In practice, we substitute $\nablaxx\log p_t(\xx_t)$ with $\boldsymbol{s}_{\btheta}(\xx_t, t)$ and substitute $p_T(\xx_T)$ wiht $\pi$ to sample the reverse SDE.
\begin{equation}
\label{eq: appendix controllable variance reverse SDEs practical}
\diff \xx_t^{\btheta} = \left[f(t) \xx_t^{\btheta} - \left(\frac{1 + \tau^2(t)}{2}\right) g^2(t)\boldsymbol{s}_{\btheta}(\xx_t^{\btheta}, t)\right] \diff t + \tau(t)g(t)  \diff \rbm,\hspace{4mm}\xx_T^{\btheta} \sim \pi,
\end{equation}

where $\pi$ is a known distribution, specifically here a Gaussian. We have the following theorem under the Assumption in Appendix A in~\citep{song2020score}.
\begin{theorem}
\label{thm: appendix tau 1 optimal}
Let $p = p_0$ be the data distribution, which is the distribution if we sample from the ground truth reverse SDE~\eqref{eq: appendix controllable variance reverse SDEs} at time $0$. Let $p_{\btheta}^{\tau(t)}$ be the distribution if we sample from the practical reverse SDE~\eqref{eq: appendix controllable variance reverse SDEs practical} at time $0$. Under the assumptions above, we have
\begin{equation}
\begin{aligned}
&D_{KL}\left(p \| p_{\btheta}^{\tau(t)}\right)\\
\leq& D_{KL}\left(p_T \| \pi\right) + \frac{1}{8} \int_0^T \Expectation_{p_t(\xx)}\left[\left(\tau(t) + \frac{1}{\tau(t)}\right)^2g^2(t) \| \nablaxx \log p_t(\xx) - \boldsymbol{s}_{\btheta}(\xx, t) \|^2 \right]\diff t.    
\end{aligned}
\end{equation}
This evidence lower bound (ELBO) is minimized when $\tau(t) \equiv 1$.
\end{theorem}
\begin{proof}
Denote the path measure of Eq.~\eqref{eq: appendix controllable variance reverse SDEs v2} and Eq.~\eqref{eq: appendix controllable variance reverse SDEs practical} as $\boldsymbol{\mu}$ and $\boldsymbol{\nu}$ respectively. Both $\boldsymbol{\mu}$ and $\boldsymbol{\nu}$ are uniquely determined by the corresponding SDEs due to assumptions. Consider a Markov kernel $K\left( \left\{ \bz_t \right\}_{t \in [0,T]}, \by \right) = \delta(\bz_0 = \by)$. Thus we have the following result
\begin{equation}
\int K\left( \left\{ \xx_t \right\}_{t \in [0,T]}, \xx \right) \diff \boldsymbol{\mu}\left(\left\{ \xx_t \right\}_{t \in [0,T]}\right)= p_0(\xx),
\end{equation}
\begin{equation}
\int K\left( \left\{ \xx_t^{\btheta} \right\}_{t \in [0,T]}, \xx \right) \diff \boldsymbol{\nu}\left(\left\{ \xx_t^{\btheta} \right\}_{t \in [0,T]}\right)= p_{\btheta}^{\tau(t)}(\xx).
\end{equation}
By data processing inequality for KL divergence
\begin{equation}
\begin{aligned}
&D_{KL}\left(p \| p_{\btheta}^{\tau(t)}\right) = D_{KL}\left(p_0 \| p_{\btheta}^{\tau(t)}\right) \\
=& D_{KL}\left(\int K\left( \left\{ \xx_t \right\}_{t \in [0,T]}, \xx \right) \diff \mu\left(\left\{ \xx_t \right\}_{t \in [0,T]}\right) \Bigl\| \int K\left( \left\{ \xx_t^{\btheta} \right\}_{t \in [0,T]}, \xx \right) \diff \nu\left(\left\{ \xx_t^{\btheta} \right\}_{t \in [0,T]}\right) \right)\\
\leq& D_{KL}\left(\boldsymbol{\mu} \| \boldsymbol{\nu}\right).
\end{aligned}
\end{equation}
By the chain rule of KL divergence, we have
\begin{equation}
D_{KL}\left(\boldsymbol{\mu} \| \boldsymbol{\nu}\right) = D_{KL}\left(p_T \| \pi\right) + \Expectation_{\bz\sim p_T}\left[ D_{KL}\left(\boldsymbol{\mu}(\cdot | \xx_T = \bz) \| \boldsymbol{\nu}(\cdot | \xx_T^{\btheta} = \bz)\right) \right].
\end{equation}
By Girsanov Thoerem, $D_{KL}\left(\boldsymbol{\mu}(\cdot | \xx_T = \bz) \| \boldsymbol{\nu}(\cdot | \xx_T^{\btheta} = \bz)\right)$ can be computed as
\begin{equation}
\begin{aligned}
&D_{KL}\left(\boldsymbol{\mu}(\cdot | \xx_T = \bz) \| \boldsymbol{\nu}(\cdot | \xx_T^{\btheta} = \bz)\right) \\
=& \Expectation_{\mu} \left[ \int_0^T \frac{1}{2}\left(\tau(t) + \frac{1}{\tau(t)}\right)g(t)\left( \nablaxx \log p_t(\xx) - \boldsymbol{s}_{\btheta}(\xx, t) \right) \diff \bar{\boldsymbol{w}}_t\right]\\
&+\Expectation_{\mu}\left[ \frac{1}{2}\int_0^T \frac{1}{4}\left(\tau(t) + \frac{1}{\tau(t)}\right)^2g^2(t) \| \nablaxx \log p_t(\xx) - \boldsymbol{s}_{\btheta}(\xx, t) \|^2 \diff t\right] \\
=&  \frac{1}{8} \int_0^T \Expectation_{p_t(\xx)}\left[\left(\tau(t) + \frac{1}{\tau(t)}\right)^2g^2(t) \| \nablaxx \log p_t(\xx) - \boldsymbol{s}_{\btheta}(\xx, t) \|^2 \right]\diff t
\end{aligned}
\end{equation}
\end{proof}

\section{Implementation Details}

For our 2-step SA-Predictor and 1-step SA-Corrector, we find that the coefficient will degenerate to a simple case.

For 2-step SA-Predictor, assume on $[t_i, t_{i+1}]$, $\tau(t) = \tau$ is a constant,
\begin{equation}
b_i = e^{-\lambda_{t_{i+1}} \tau^2} \sigma_{t_{i+1}} (1+\tau^2) \int_{\lambda_{t_i}}^{\lambda_{t_{i+1}}} e^{(1+\tau^2)\lambda} \frac{\lambda - \lambda_{t_{i-1}}}{\lambda_{t_i} - \lambda_{t_{i-1}}} \diff \lambda,
\end{equation}

\begin{equation}
b_{i-1} = e^{-\lambda_{t_{i+1}} \tau^2} \sigma_{t_{i+1}} (1+\tau^2) \int_{\lambda_{t_i}}^{\lambda_{t_{i+1}}} e^{(1+\tau^2)\lambda} \frac{\lambda - \lambda_{t_{i}}}{\lambda_{t_{i-1}} - \lambda_{t_{i}}} \diff \lambda,
\end{equation}
we have
\begin{equation}
\begin{aligned}
  \xx_{t_{i + 1}} =& \frac{\sigma_{t_{i + 1}}}{\sigma_{t_{i}}} e^{-\int_{\lambda_{t_i}}^{\lambda_{t_{i+1}}} \tau^2(\Tilde{\lambda}) \diff \Tilde{\lambda}} \xx_{t_i} +  b_{i} \xx_{\btheta}(\xx_{t_{i}}, t_{i}) + b_{i-1} \xx_{\btheta}(\xx_{t_{i-1}}, t_{i-1}) + \Tilde{\sigma}_i \bxi  \\
  =& \frac{\sigma_{t_{i + 1}}}{\sigma_{t_{i}}} e^{-\int_{\lambda_{t_i}}^{\lambda_{t_{i+1}}} \tau^2(\Tilde{\lambda}) \diff \Tilde{\lambda}} \xx_{t_i} +  (b_{i} +b_{i-1}) \xx_{\btheta}(\xx_{t_{i}}, t_{i}) \\
  &- b_{i-1}(\xx_{\btheta}(\xx_{t_{i}}, t_{i}) -\xx_{\btheta}(\xx_{t_{i-1}}, t_{i-1}) )+ \Tilde{\sigma}_i \bxi.
\end{aligned}
\end{equation}

Let $h = \lambda_{t_{i+1}} - \lambda_{t_i}$, we have

\begin{equation}
\begin{aligned}
b_i + b_{i-1} &= e^{-\lambda_{t_{i+1}} \tau^2} \sigma_{t_{i+1}} (1+\tau^2) \int_{\lambda_{t_i}}^{\lambda_{t_{i+1}}} e^{(1+\tau^2)\lambda} \diff \lambda\\
&= \alpha_{t_{i+1}}(1 - e^{-h(1+\tau^2)}) \\
b_{i-1} &= \alpha_{t_{i+1}} \frac{e^{-(1+\tau^2)h} + (1+\tau^2)h - 1}{(1+\tau^2)(\lambda_{t_i} - \lambda_{t_{i-1}})} \\
&= \frac{\alpha_{t_{i+1}}}{\lambda_{t_i} - \lambda_{t_{i-1}}} \frac{1 - (1+\tau^2)h + \frac{1}{2} (1+\tau^2)^2h^2 + \mathcal{O}(h^3)+ (1+\tau^2)h - 1}{1+\tau^2}\\
&= \frac{\alpha_{t_{i+1}}}{\lambda_{t_i} - \lambda_{t_{i-1}}} \frac{1}{2}(1+\tau^2)h^2 + \mathcal{O}(h^3)\\
\end{aligned}
\end{equation}
Thus we implement $\hat{b}_{i-1}$ as $\frac{\alpha_{t_{i+1}}}{\lambda_{t_i} - \lambda_{t_{i-1}}} \frac{1}{2}(1+\tau^2)h^2$ and $\hat{b}_{i}$ as $\alpha_{t_{i+1}}(1 - e^{-h(1+\tau^2)}) - \hat{b}_{i-1}$. Note that substituting $b_i, b_{i-1}$ as $\hat{b}_{i}, \hat{b}_{i-1}$ will maintain the convergence order result of 2-step SA-Predictor since the modified term is $\mathcal{O}(h^3)$. The implementation detail for 1-step SA-Corrector is technically the same.

\section{Experiment Details}
\label{appendix: experiment details}

\subsection{Details on $\tau(t)$, Predictor Steps and Corrector Steps}

\paragraph{CIFAR10 32x32}
For the CIFAR10 experiment in Section~\ref{sec: sota result}, we use the pretrained baseline-unconditional-VE model\footnote{\url{https://nvlabs-fi-cdn.nvidia.com/edm/pretrained/baseline/baseline-cifar10-32x32-uncond-ve.pkl}}from~\citep{Karras2022edm}. It's an unconditional model with VE noise schedule. To fairly compare with results in~\citep{Karras2022edm}, we use a piecewise constant function $\tau(t)$ inspired by~\citep{Karras2022edm}. Concretely, denoting $\sigma^{EDM}_t = \frac{\sigma_t}{\alpha_t}$, our $\tau(t)$ is set to be a constant $\tau$ in the interval $[(\sigma^{EDM})^{-1}(0.05), (\sigma^{EDM})^{-1}(1)]$ and to be zero outside the interval. We find empirically that this piecewise constant function setting makes our \textit{SA-Solver} converge better, especially in large noise scale cases. We use a 3-step SA-Predictor and a 3-step SA-Corrector. For the CIFAR10 experiment in Section~\ref{sec: ablationpc} and~\ref{sec: inaccurate score}, we also use the piecewise constant function $\tau(t)$ as above. The predictor steps and corrector steps vary to verify the effectiveness of our proposed method in Section~\ref{sec: ablationpc}, while they are both set to be 3-steps in Section~\ref{sec: inaccurate score}.
\paragraph{ImageNet 64x64}
For the ImageNet 64x64 experiment in Section~\ref{sec: sota result}, we use the pretrained model\footnote{\url{https://openaipublic.blob.core.windows.net/diffusion/jul-2021/64x64_diffusion.pt}} from~\citep{dhariwal2021diffusion}. It's a conditional model with VP cosine noise schedule. To fairly compare with results in~\citep{Karras2022edm}, we use a piecewise constant function $\tau(t)$ inspired by~\citep{Karras2022edm}. Concretely, denoting $\sigma^{EDM}_t = \frac{\sigma_t}{\alpha_t}$, our $\tau(t)$ is set to be a constant $\tau$ in the interval $[(\sigma^{EDM})^{-1}(0.05), (\sigma^{EDM})^{-1}(50)]$ and to be zero outside the interval. We find empirically that this piecewise constant function setting makes our \textit{SA-Solver} converge better, especially in large noise scale cases. We use a 3-step SA-Predictor and a 3-step SA-Corrector.
\paragraph{Other experiments}
For other experiments, we use a constant function $\tau(t) \equiv \tau$. It's generally not the optimal choice for each individual task, thus further fine-grained tuning has the potential to improve the results. We aim to report the result of our \textit{SA-Solver} without extra hyperparameter tuning. We use a 3-step SA-Predictor and a 3-step SA-Corrector under 20 NFEs and 2-step SA-Predictor and a 1-step SA-Corrector beyond 20 NFEs.

\subsection{Details on Pretrained Models and Settings}

\paragraph{CIFAR10 32x32}
For the CIFAR10 experiment, we use the pretrained baseline-unconditional-VE model\footnote{\url{https://nvlabs-fi-cdn.nvidia.com/edm/pretrained/baseline/baseline-cifar10-32x32-uncond-ve.pkl}}from~\citep{Karras2022edm}. It's an unconditional model with VE noise schedule. To fairly compare with results in~\citep{Karras2022edm}, we follow the time step schedule in it. Specifically, we set $\sigma_{min} = 0.02$ and $\sigma_{max} = 80$ and select the step by $\sigma_{i} = (\sigma_{max}^{\frac{1}{7}} + \frac{i}{N-1} (\sigma_{min}^{\frac{1}{7}} - \sigma_{max}^{\frac{1}{7}}))^7$ for \textit{SA-Solver} and UniPC. We directly report the results of the deterministic sampler and stochastic sampler of EDM. To make it a strong baseline, we report the results of the optimal setting for 4 hyper-parameters $\left\{S_{churn}, S_{tmin}, S_{tmax}, S_{noise}\right\}$ and report its lowest observed FID. While for \textit{SA-Solver} and UniPC, we report the averaged observed FID.
\paragraph{ImageNet 64x64}
For the ImageNet 64x64 experiment, we use the pretrained model\footnote{\url{https://openaipublic.blob.core.windows.net/diffusion/jul-2021/64x64_diffusion.pt}} from~\citep{dhariwal2021diffusion}. It's a conditional model with VP cosine noise schedule. To fairly compare with results in~\citep{Karras2022edm}, we follow the time step schedule in it and use conditional sampling. Specifically, we set $\sigma_{min} = 0.0064$ and $\sigma_{max} = 80$ and select the step by $\sigma_{i} = (\sigma_{max}^{\frac{1}{7}} + \frac{i}{N-1} (\sigma_{min}^{\frac{1}{7}} - \sigma_{max}^{\frac{1}{7}}))^7$ for \textit{SA-Solver}, UniPC, DPM-Solver and DDIM. We directly report the results of the deterministic sampler and stochastic sampler of EDM. To make it a strong baseline, we report the results of the optimal setting for 4 hyper-parameters $\left\{S_{churn}, S_{tmin}, S_{tmax}, S_{noise}\right\}$ and report its lowest observed FID. While for \textit{SA-Solver}, UniPC, DPM-Solver and DDIM, we report the averaged observed FID.
\paragraph{ImageNet 256x256}
For the ImageNet 256x256 experiment, we use three different pretrained models:LDM\footnote{\url{https://ommer-lab.com/files/latent-diffusion/nitro/cin/model.ckpt}}(VP, handcrafted noise schedule) from~\citep{Rombach_2022_CVPR}, DiT-XL/2\footnote{\url{https://dl.fbaipublicfiles.com/DiT/models/DiT-XL-2-256x256.pt}}(VP, linear noise schedule) from~\citep{peebles2022scalable}, Min-SNR\footnote{\url{https://github.com/TiankaiHang/Min-SNR-Diffusion-Training/releases/download/v0.0.0/ema_0.9999_xl.pt}}(VP, cosine noise schedule) from~\citep{hang2023efficient}.
We use classifier-free guidance of scale $s = 1.5$ and a uniform time step schedule because it's the most common setting for guided sampling for ImageNet 256x256.
\paragraph{ImageNet 512x512}
For the ImageNet 256x256 experiment, we use the pre-trained model: DiT-XL/2\footnote{\url{https://dl.fbaipublicfiles.com/DiT/models/DiT-XL-2-512x512.pt}} from~\citep{peebles2022scalable}.
We use classifier-free guidance of scale $s = 1.5$ and a uniform time step schedule following the settings of DiT~\citep{peebles2022scalable}.
\paragraph{LSUN Bedroom 256x256}
For the LSUN Bedroom 256x256 experiment, we use the pretrained model\footnote{\url{https://openaipublic.blob.core.windows.net/diffusion/jul-2021/lsun_bedroom.pt}} from~\citep{dhariwal2021diffusion}.
We use unconditional sampling and a uniform lambda step schedule from~\citep{lu2022dpmsa}.

\section{Additional Results}
We include the detailed FID results in Figure~\ref{fig:ablations}, Figure~\ref{fig:sotas} and Figure~\ref{fig:ins} in the~\cref{tb: results cifar10,tb: results cifar10 sasolver,tb: results imagenet64,tb: results imagenet64 sasolver,tb: results imagenet256,tb: results cifar10 inaccurate score,tb: results imagenet256 inaccurate score,tb: ablation cifar10,tb: ablation imagenet64,tb: ablation imagenet256,tb: ablation lsun 256}. The ablation study shows that stochasticity indeed helps improve sample quality. We find that for small NFEs, the magnitude of stochasticity should be small while for large NFEs, large magnitude of stochasticity helps improve sample quality. It can also be observed that in latent space, SDE converges faster as in Table~\ref{tb: ablation imagenet256}. With only 10 NFEs, $\tau = 0.6$ is better than $\tau = 0$. With 20 NFEs, our \textit{SA-Solver} can achieve 3.87 FID, which outperforms all ODE samplers even with far more steps.

\begin{table}[h]
  \caption{Sample quality measured by FID ↓ on CIFAR10 32x32 dataset (VE-baseline model from~\citep{Karras2022edm}) varying the number of function evaluations (NFE). For the results from EDM$^\dagger$, we reported its lowest observed FID.}
  \label{tb: results cifar10}
  \centering
  \begin{tabular}{lccccccc}
    \toprule
    Method $\backslash$ NFE & 11 & 15 & 23 & 31 & 47 & 63 & 95\\
    \midrule
    DDIM($\eta=0$) & 18.28 & 12.23 & 7.93 & 6.45 & 5.27 & 4.83 & 4.42\\
    DPM-Solver & 9.26 & 5.13 & 4.52 & 4.30 & 4.02 & 3.97 & 3.94\\
    UniPC & \bf6.42 & 5.02 & 4.19 & 4.00 & 3.91 & 3.90 & 3.89\\
    EDM(ODE)$^\dagger$ & 13.46 & 5.62 & 4.04 & 3.82 & 3.79 & 3.80 & 3.79\\
    \midrule
    EDM(SDE)$^\dagger$ & 23.94 & 8.94 & 4.73 & 3.95 & 3.59 & 3.36 & 3.06\\
    SA-Solver & 6.46 & \bf4.91 & \bf3.77 & \bf3.40 & \bf2.92 & \bf2.74 & \bf2.63\\
    \bottomrule
  \end{tabular}
\end{table}

\begin{table}[h]
  \caption{Sample quality measured by FID ↓ on CIFAR10 32x32 dataset (VE-baseline model from~\citep{Karras2022edm}) varying the number of function evaluations (NFE) and the magnitude of stochasticity ($\tau$).}
  \label{tb: results cifar10 sasolver}
  \centering
  \begin{tabular}{lccccccc}
    \toprule
    SA-Solver $\backslash$ NFE & 11 & 15 & 23 & 31 & 47 & 63 & 95\\
    \midrule
    $\tau = 0.0$ & \bf6.46 & 5.06 & 4.22 & 4.02 & 3.93 & 3.92 & 3.91\\
    $\tau = 0.2$ & 6.54 & 5.01 & 4.14 & 3.95 & 3.89 & 3.84 & 3.83\\
    $\tau = 0.4$ & 6.79 & \bf4.91 & 4.03 & 3.81 & 3.76 & 3.74 & 3.67\\
    $\tau = 0.6$ & 7.34 & 4.91 & 3.85 & 3.65 & 3.60 & 3.56 & 3.57\\
    $\tau = 0.8$ & 8.61 & 5.28 & \bf3.77 & 3.48 & 3.45 & 3.43 & 3.50\\
    $\tau = 1.0$ & 10.89 & 6.52 & 3.98 & \bf3.40 & 3.21 & 3.25 & 3.29\\
    $\tau = 1.2$ & 14.49 & 9.33 & 5.19 & 3.69 & 3.00 & 3.03 & 3.07\\
    $\tau = 1.4$ & 20.19 & 13.76 & 7.60 & 4.91 & \bf2.92 & 2.86 & 2.93\\
    $\tau = 1.6$ & 27.90 & 20.51 & 11.89 & 8.07 & 3.25 & \bf2.74 & 2.80\\
    $\tau = 1.8$ & 36.26 & 29.43 & 18.13 & 14.00 & 4.60 & 2.83 & \bf2.63\\
    \bottomrule
  \end{tabular}
\end{table}

\begin{table}[h]
  \caption{Sample quality measured by FID ↓ on ImageNet 64x64 dataset (model from~\citep{dhariwal2021diffusion}) varying the number of function evaluations (NFE). For the results from EDM$^\dagger$, we reported its lowest observed FID.}
  \label{tb: results imagenet64}
  \centering
  \begin{tabular}{lccccccc}
    \toprule
    Method $\backslash$ NFE & 15 & 23 & 31 & 47 & 63 & 95\\
    \midrule
    DDIM($\eta=0$) & 8.48 & 5.39 & 4.27 & 3.46 & 3.17 & 2.95\\
    DPM-Solver & 3.49 & 3.04 & 2.88 & 2.80 & 2.76 & 2.74\\
    UniPC & 3.51 & 2.84 & 2.75 & 2.72 & 2.71 & 2.72\\
    EDM(ODE)$^\dagger$ & 4.78 & 3.12 & 2.84 & 2.73 & 2.73 & 2.67\\
    \midrule
    EDM(SDE)$^\dagger$ & 8.94 & 4.30 & 3.40 & 2.72 & 2.44 & 2.22 \\
    SA-Solver & \bf3.41 & \bf2.61 & \bf2.23 & \bf1.95 & \bf1.88 & \bf1.81\\
    \bottomrule
  \end{tabular}
\end{table}

\begin{table}[h]
  \caption{Sample quality measured by FID ↓ on ImageNet 64x64 dataset (model from~\citep{dhariwal2021diffusion}) varying the number of function evaluations (NFE) and the magnitude of stochasticity ($\tau$).}
  \label{tb: results imagenet64 sasolver}
  \centering
  \begin{tabular}{lccccccc}
    \toprule
    SA-Solver $\backslash$ NFE & 15 & 23 & 31 & 47 & 63 & 95\\
    \midrule
    $\tau = 0.0$ & 3.48 & 2.72 & 2.72 & 2.66 & 2.64 & 2.71\\
    $\tau = 0.2$ & \bf3.41 & 2.80 & 2.63 & 2.63 & 2.64 & 2.60\\
    $\tau = 0.4$ & 3.52 & 2.70 & 2.51 & 2.51 & 2.49 & 2.49\\
    $\tau = 0.6$ & 3.98 & \bf2.61 & 2.44 & 2.39 & 2.34 & 2.35\\
    $\tau = 0.8$ & 5.80 & 2.68 & 2.32 & 2.24 & 2.19 & 2.21 \\
    $\tau = 1.0$ & 10.06 & 3.38 & \bf2.23& 2.09 & 2.08 & 2.08\\
    $\tau = 1.2$ & 18.39 & 5.52 & 2.52 & \bf1.95 & 1.97 & 2.00\\
    $\tau = 1.4$ & 32.42 & 10.37 & 3.83 & 2.05 & 1.89 & 1.89\\
    $\tau = 1.6$ & 52.31 & 19.64 & 7.10 & 2.60 & \bf1.88 & \bf1.81\\
    \bottomrule
  \end{tabular}
\end{table}

\begin{table}[h]
  \caption{Sample quality measured by FID ↓ on CIFAR10 32x32 dataset (model trained by ourselves; see Section~\ref{sec: inaccurate score}) varying the sampling method and the training epoch.}
  \label{tb: results cifar10 inaccurate score}
  \centering
  \begin{tabular}{lccccccc}
    \toprule
    method (NFE = 31) $\backslash$ epoch & 1250 & 1300 & 1350 & 1400 & 1450 & 1500\\
    \midrule
    DDIM & 39.32 & 29.79 & 19.59 & 12.98 & 8.63 & 6.64\\
    DPM-Solver & 30.57 & 22.11 & 13.85 & 8.85 & 5.68 & 4.55\\
    EDM(ODE) &27.51 & 19.82 & 12.37 & 8.03 & 5.33 & 4.32\\
    SA-Solver($\tau = 0.6$) & 20.55 & 14.89 & 9.71 & 6.55 & 4.61 & 4.08\\
    SA-Solver($\tau = 1.0$) & \bf13.62 & \bf10.01 & \bf6.79 & \bf4.81 & \bf3.70 & \bf3.47 \\
    \bottomrule
  \end{tabular}
\end{table}

\begin{table}[h]
  \caption{Sample quality measured by FID ↓ on ImageNet 256x256 dataset (model trained by ourselves; see Section~\ref{sec: inaccurate score}) varying the sampling method and the training epoch.}
  \label{tb: results imagenet256 inaccurate score}
  \centering
  \begin{tabular}{lcccccc}
    \toprule
    method (NFE = 40) $\backslash$ epoch & 50 & 100 & 150 & 200 & 250\\
    \midrule
    DDIM & 19.40 & 9.61 & 6.75 & 5.86 & 5.12 \\
    DPM-Solver & 18.75 & 8.96 & 6.15 & 5.28 & 4.62 \\
    SA-Solver($\tau = 0.4$) & 17.93 & 8.39 & 5.69 & 4.84 & 4.24 \\
    SA-Solver($\tau = 0.8$) & \bf16.57 & \bf7.54 & \bf5.15 & \bf4.48 & \bf3.99  \\
    \bottomrule
  \end{tabular}
\end{table}

\begin{table}[h]
  \caption{Sample quality measured by FID ↓ on ImageNet 256x256 dataset(model from~\citep{Rombach_2022_CVPR}) varying the number of function evaluations (NFE).}
  \label{tb: results imagenet256}
  \centering
  \begin{tabular}{lccccccc}
    \toprule
    Method $\backslash$ NFE & 5 & 10 & 20 & 40 & 60 & 80 & 100\\
    \midrule
    DDIM($\eta=0$) & 58.68 & 16.32 & 6.82 & 4.71 & 4.45 & 4.28 & 4.23\\
    DPM-Solver & 166.88 & 6.19 & 5.51 & 4.17 & 4.18 & 4.21 & 4.15\\
    UniPC & 12.79 & 4.96 & 4.21 & 4.14 & 4.12 & 4.09 & 4.10\\
    \midrule
    DDIM($\eta=1$) & 138.91 & 50.05 & 14.60 & 6.09 & 4.56 & 4.12 & 3.87\\
    SA-Solver & \bf11.46 & \bf4.82 & \bf3.88 & \bf3.47 & \bf3.37 & \bf3.37 & \bf3.33\\
    \bottomrule
  \end{tabular}
\end{table}

\begin{table}[h]
  \caption{Ablation study on the effect of the magnitude of stochasticity using \textit{SA-Solver}. Sample quality measured by FID ↓ on CIFAR10 32x32 dataset(model from~\citep{Karras2022edm}) varying the number of function evaluations (NFE) and the magnitude of stochasticity($\tau$).}
  \label{tb: ablation cifar10}
  \centering
  \begin{tabular}{lccccccc}
    \toprule
    $\tau$ $\backslash$ NFE & 15 & 23 & 31 & 47 & 63 & 95 & 127\\
    \midrule
    0 & \bf4.84 & 4.11 & 3.94 & 3.86 & 3.88 & 3.87 & 3.87\\
    0.2 & 4.96 & 4.04 & 3.84 & 3.75 & 3.74 & 3.79 & 3.75\\
    0.4 & 5.27 & 4.00 & 3.87 & 3.64 & 3.71 & 3.70 & 3.62\\
    0.6 & 6.05 & \bf3.95 & 3.61 & 3.49 & 3.46 & 3.53 & 3.43\\
    0.8 & 7.40 & 4.12 & 3.53 & 3.28 & 3.37 & 3.32 & 3.30\\
    1.0 & 10.00 & 4.49 & \bf3.41 & 3.18 & 3.24 & 3.17 & 3.15\\
    1.2 & 13.58 & 5.14 & 3.59 & 3.10 & 3.02 & 2.97 & 3.05\\
    1.4 & 17.88 & 6.55 & 3.94 & \bf3.04 & \bf3.01 & \bf2.89 & 2.95\\
    1.6 & 22.42 & 8.44 & 4.69 & 3.20 & 3.02 & 2.94 & \bf2.89\\
    \bottomrule
  \end{tabular}
\end{table}

\begin{table}[h]
  \caption{Ablation study on the effect of the magnitude of stochasticity using \textit{SA-Solver}. Sample quality measured by FID ↓ on ImageNet 64x64 dataset(model from~\citep{dhariwal2021diffusion}) varying the number of function evaluations (NFE) and the magnitude of stochasticity($\tau$).}
  \label{tb: ablation imagenet64}
  \centering
  \begin{tabular}{lccccccc}
    \toprule
    $\tau$ $\backslash$ NFE & 20 & 40 & 60 & 80 & 100\\
    \midrule
    0 & \bf3.30 & 2.83 & 2.78 & 2.79 & 2.82 \\
    0.2 & 3.32 & 2.77 & 2.72 & 2.74 & 2.79\\
    0.4 & 3.37 & 2.68 & 2.63 & 2.62 & 2.59\\
    0.6 & 3.61 & 2.57 & 2.49 & 2.49 & 2.47\\
    0.8 & 4.19 & \bf2.51 & 2.40 & 2.34 & 2.30\\
    1.0 & 5.55 & 2.54 & 2.32 & 2.21 & 2.20\\
    1.2 & 7.93 & 2.77 & \bf2.29 & \bf2.14 & 2.14\\
    1.4 & 11.55 & 3.20 & 2.40 & \bf2.14 & \bf2.08\\
    1.6 & 16.15 & 3.97 & 2.60 & 2.20 & 2.09\\
    \bottomrule
  \end{tabular}
\end{table}

\begin{table}[h]
  \caption{Ablation study on the effect of the magnitude of stochasticity using \textit{SA-Solver}. Sample quality measured by FID ↓ on ImageNet 256x256 dataset(model from~\citep{Rombach_2022_CVPR}) varying the number of function evaluations (NFE) and the magnitude of stochasticity($\tau$).}
  \label{tb: ablation imagenet256}
  \centering
  \begin{tabular}{lccccccc}
    \toprule
    $\tau$ $\backslash$ NFE & 5 & 10 & 20 & 40 & 60 & 80 & 100\\
    \midrule
    0 & \bf11.46 & 5.04 & 4.30 & 4.16 & 4.12 & 4.10 & 4.16\\
    0.2 & 11.88 & 4.89 & 4.29 & 4.05 & 4.02 & 4.01 & 4.03\\
    0.4 & 12.69 & 4.84 & 4.14 & 3.86 & 3.84 & 3.83 & 3.84\\
    0.6 & 14.84 & \bf4.82 & 3.99 & 3.63 & 3.62 & 3.63 & 3.61\\
    0.8 & 18.82 & 5.09 & \bf3.87 & 3.55 & 3.50 & 3.47 & 3.47\\
    1.0 & 25.96 & 6.06 & 3.88 & \bf3.47 & 3.41 & 3.39 & 3.38\\
    1.2 & 37.20 & 8.23 & 3.92 & \bf3.47 & \bf3.37 & \bf3.37 & \bf3.33\\
    1.4 & 53.03 & 12.93 & 4.08 & 3.53 & 3.40 & 3.38 & 3.36\\
    1.6 & 71.30 & 24.08 & 4.43 & 3.56 & 3.44 & 3.45 & 3.33\\
    \bottomrule
  \end{tabular}
\end{table}

\begin{table}[h]
  \caption{Ablation study on the effect of the magnitude of stochasticity using \textit{SA-Solver}. Sample quality measured by FID ↓ on LSUN Bedroom 256x256 dataset(model from~\citep{dhariwal2021diffusion}) varying the number of function evaluations (NFE) and the magnitude of stochasticity($\tau$).}
  \label{tb: ablation lsun 256}
  \centering
  \begin{tabular}{lccccccc}
    \toprule
    $\tau$ $\backslash$ NFE & 20 & 40 & 60 & 80 & 100\\
    \midrule
    0 & 3.60 & 3.14 & 3.06 & 3.09 & 3.07\\
    0.2 & \bf3.51 & 3.12 & 3.00 & 2.99 & 2.99\\
    0.4 & 3.70 & 3.09 & 2.97 & 3.03 & 3.16\\
    0.6 & 4.10 & \bf3.08 & \bf2.95 & 2.99 & 3.03\\
    0.8 & 4.75 & 3.11 & 2.97 & 2.89 & 2.99\\
    1.0 & 6.18 & 3.28 & 2.98 & 2.90 & \bf2.91\\
    1.2 & 8.54 & 3.53 & 3.12 & \bf2.86 & 3.00\\
    1.4 & 12.14 & 4.25 & 3.24 & 2.98 & 2.93\\
    1.6 & 16.63 & 5.50 & 3.75 & 3.18 & 3.10\\
    \bottomrule
  \end{tabular}
\end{table}

\section{Additional Samples}
We include additional samples in this section. In Figure~\ref{figure: cifar10} and Figure~\ref{figure: imagenet64} we compare samples of our proposed \textit{SA-Solver} with other diffusion samplers. In Figure~\ref{figure: lsun} and Figure~\ref{figure: dit512}, we compare samples of our proposed \textit{SA-Solver} under different NFEs and $\tau$. In Figure~\ref{figure: t2i1} and Figure~\ref{figure: t2i2}, we compare samples of our proposed \textit{SA-Solver} with other diffusion samplers on text-to-image tasks. Our \textit{SA-Solver} can generate more diverse samples with more details.

\begin{figure}[!htb]
    \centering
    \setlength{\tabcolsep}{1pt}
    \begin{tabular}{c c c c c}
     & NFE = 15 & NFE = 23 & NFE = 47 & NFE = 95 \\
    \raisebox{5.0\height}{DDIM($\eta=0$)} &
    \begin{subfigure}{0.2\textwidth}
        \includegraphics[width=\textwidth]{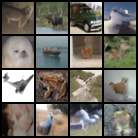}
    \end{subfigure} &
    \begin{subfigure}{0.2\textwidth}
        \includegraphics[width=\textwidth]{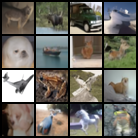}
    \end{subfigure} &
    \begin{subfigure}{0.2\textwidth}
        \includegraphics[width=\textwidth]{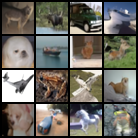}
    \end{subfigure} &
    \begin{subfigure}{0.2\textwidth}
        \includegraphics[width=\textwidth]{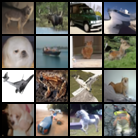}
    \end{subfigure} \\
    \raisebox{5.0\height}{DPM-Solver} &
    \begin{subfigure}{0.2\textwidth}
        \includegraphics[width=\textwidth]{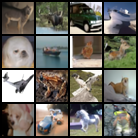}
    \end{subfigure} &
    \begin{subfigure}{0.2\textwidth}
        \includegraphics[width=\textwidth]{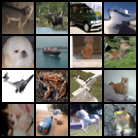}
    \end{subfigure} &
    \begin{subfigure}{0.2\textwidth}
        \includegraphics[width=\textwidth]{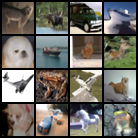}
    \end{subfigure} &
    \begin{subfigure}{0.2\textwidth}
        \includegraphics[width=\textwidth]{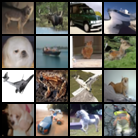}
    \end{subfigure} \\
    \raisebox{5.0\height}{UniPC} &
    \begin{subfigure}{0.2\textwidth}
        \includegraphics[width=\textwidth]{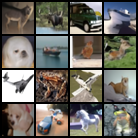}
    \end{subfigure} &
    \begin{subfigure}{0.2\textwidth}
        \includegraphics[width=\textwidth]{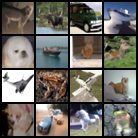}
    \end{subfigure} &
    \begin{subfigure}{0.2\textwidth}
        \includegraphics[width=\textwidth]{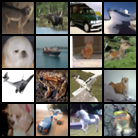}
    \end{subfigure} &
    \begin{subfigure}{0.2\textwidth}
        \includegraphics[width=\textwidth]{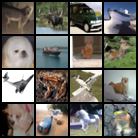}
    \end{subfigure} \\
        \raisebox{5.0\height}{EDM(ODE)} &
    \begin{subfigure}{0.2\textwidth}
        \includegraphics[width=\textwidth]{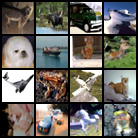}
    \end{subfigure} &
    \begin{subfigure}{0.2\textwidth}
        \includegraphics[width=\textwidth]{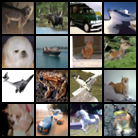}
    \end{subfigure} &
    \begin{subfigure}{0.2\textwidth}
        \includegraphics[width=\textwidth]{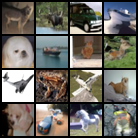}
    \end{subfigure} &
    \begin{subfigure}{0.2\textwidth}
        \includegraphics[width=\textwidth]{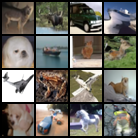}
    \end{subfigure} \\
        \raisebox{5.0\height}{EDM(SDE)} &
    \begin{subfigure}{0.2\textwidth}
        \includegraphics[width=\textwidth]{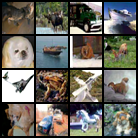}
    \end{subfigure} &
    \begin{subfigure}{0.2\textwidth}
        \includegraphics[width=\textwidth]{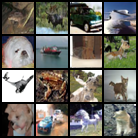}
    \end{subfigure} &
    \begin{subfigure}{0.2\textwidth}
        \includegraphics[width=\textwidth]{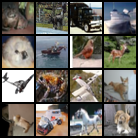}
    \end{subfigure} &
    \begin{subfigure}{0.2\textwidth}
        \includegraphics[width=\textwidth]{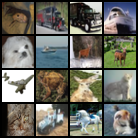}
    \end{subfigure} \\
            \raisebox{5.0\height}{SA-Solver(ours)} &
    \begin{subfigure}{0.2\textwidth}
        \includegraphics[width=\textwidth]{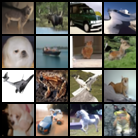}
    \end{subfigure} &
    \begin{subfigure}{0.2\textwidth}
        \includegraphics[width=\textwidth]{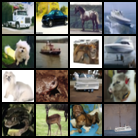}
    \end{subfigure} &
    \begin{subfigure}{0.2\textwidth}
        \includegraphics[width=\textwidth]{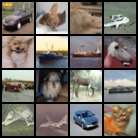}
    \end{subfigure} &
    \begin{subfigure}{0.2\textwidth}
        \includegraphics[width=\textwidth]{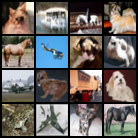}
    \end{subfigure} \\
    \end{tabular}
    \caption{Samples by DDIM, DPM-Solver, UniPC, EDM(ODE), EDM(SDE) and our SA-Solver with 15, 23, 47, 95 NFEs with the same random seed from CIFAR10 32x32 VE baseline model~\citep{Karras2022edm}}
\label{figure: cifar10}
\end{figure}

\begin{figure}[!htb]
    \centering
    \setlength{\tabcolsep}{1pt}
    \begin{tabular}{c c c c c}
     & NFE = 15 & NFE = 23 & NFE = 47 & NFE = 95 \\
    \raisebox{5.0\height}{DDIM($\eta=0$)} &
    \begin{subfigure}{0.2\textwidth}
        \includegraphics[width=\textwidth]{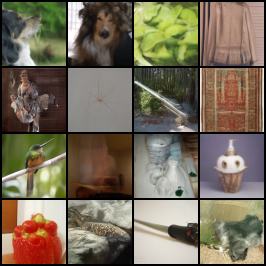}
    \end{subfigure} &
    \begin{subfigure}{0.2\textwidth}
        \includegraphics[width=\textwidth]{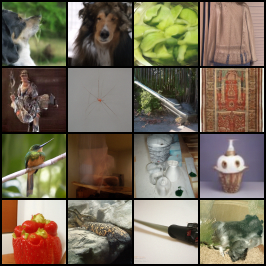}
    \end{subfigure} &
    \begin{subfigure}{0.2\textwidth}
        \includegraphics[width=\textwidth]{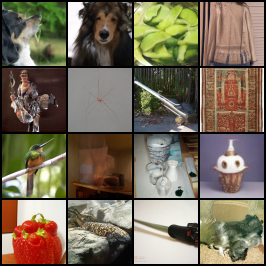}
    \end{subfigure} &
    \begin{subfigure}{0.2\textwidth}
        \includegraphics[width=\textwidth]{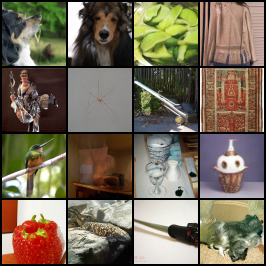}
    \end{subfigure} \\
    \raisebox{5.0\height}{DPM-Solver} &
    \begin{subfigure}{0.2\textwidth}
        \includegraphics[width=\textwidth]{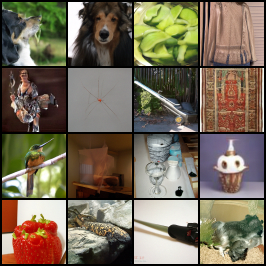}
    \end{subfigure} &
    \begin{subfigure}{0.2\textwidth}
        \includegraphics[width=\textwidth]{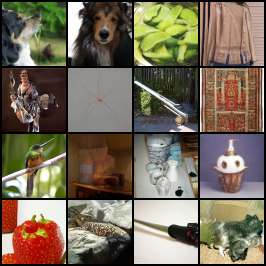}
    \end{subfigure} &
    \begin{subfigure}{0.2\textwidth}
        \includegraphics[width=\textwidth]{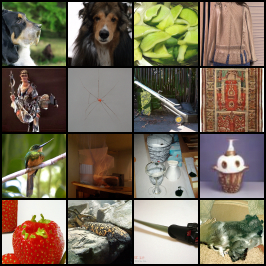}
    \end{subfigure} &
    \begin{subfigure}{0.2\textwidth}
        \includegraphics[width=\textwidth]{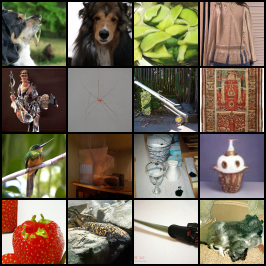}
    \end{subfigure} \\
    \raisebox{5.0\height}{UniPC} &
    \begin{subfigure}{0.2\textwidth}
        \includegraphics[width=\textwidth]{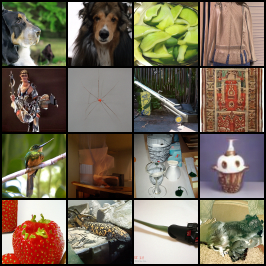}
    \end{subfigure} &
    \begin{subfigure}{0.2\textwidth}
        \includegraphics[width=\textwidth]{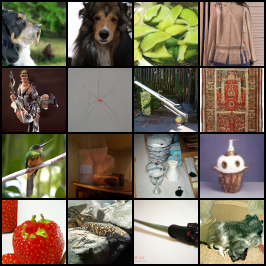}
    \end{subfigure} &
    \begin{subfigure}{0.2\textwidth}
        \includegraphics[width=\textwidth]{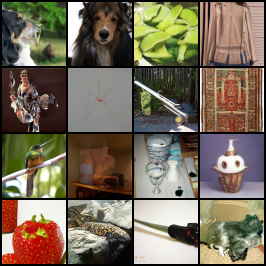}
    \end{subfigure} &
    \begin{subfigure}{0.2\textwidth}
        \includegraphics[width=\textwidth]{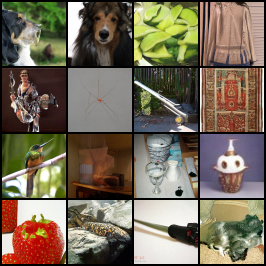}
    \end{subfigure} \\
    \raisebox{5.0\height}{SA-Solver(ours)} &
    \begin{subfigure}{0.2\textwidth}
        \includegraphics[width=\textwidth]{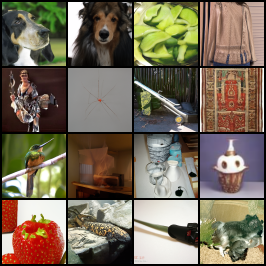}
    \end{subfigure} &
    \begin{subfigure}{0.2\textwidth}
        \includegraphics[width=\textwidth]{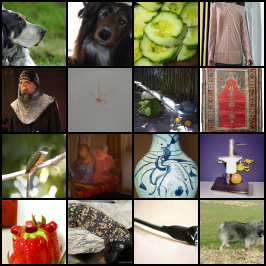}
    \end{subfigure} &
    \begin{subfigure}{0.2\textwidth}
        \includegraphics[width=\textwidth]{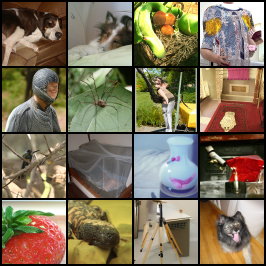}
    \end{subfigure} &
    \begin{subfigure}{0.2\textwidth}
        \includegraphics[width=\textwidth]{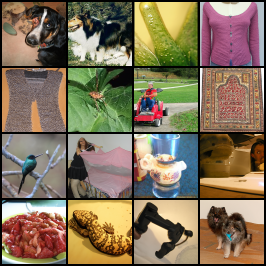}
    \end{subfigure} \\
    \end{tabular}
    \caption{Samples by DDIM, DPM-Solver, UniPC, and our SA-Solver with 15, 23, 47, 95 NFEs with the same random seed from ImageNet 64x64 model~\citep{Karras2022edm}(conditional sampling)}
\label{figure: imagenet64}
\end{figure}

\begin{figure}[!htb]
    \centering
    \setlength{\tabcolsep}{1pt}
    \begin{tabular}{c c c c c}
     & NFE = 20 & NFE = 40 & NFE = 60 & NFE = 100\\
    \raisebox{5.0\height}{SA-Solver($\eta=0$)} &
    \begin{subfigure}{0.2\textwidth}
        \includegraphics[width=\textwidth]{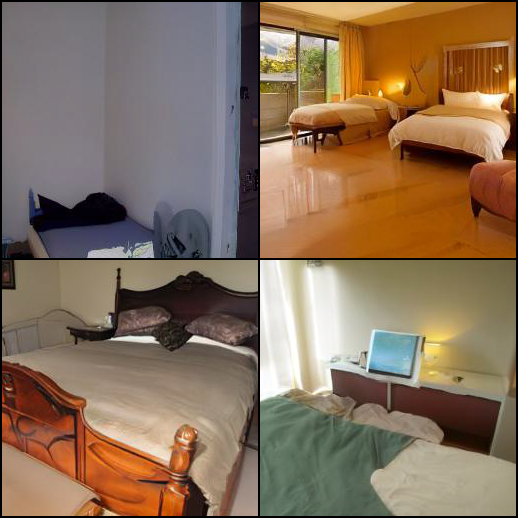}
    \end{subfigure} &
    \begin{subfigure}{0.2\textwidth}
        \includegraphics[width=\textwidth]{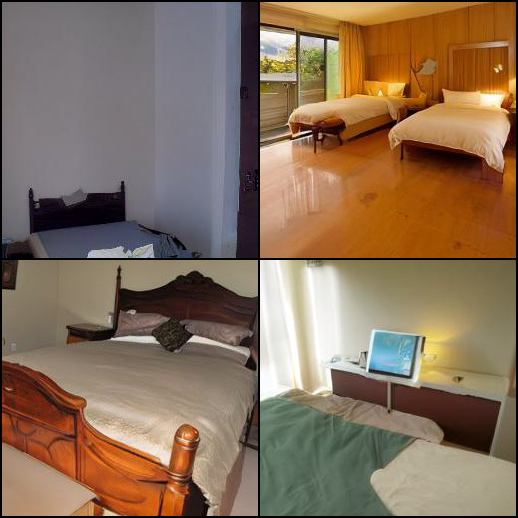}
    \end{subfigure} &
    \begin{subfigure}{0.2\textwidth}
        \includegraphics[width=\textwidth]{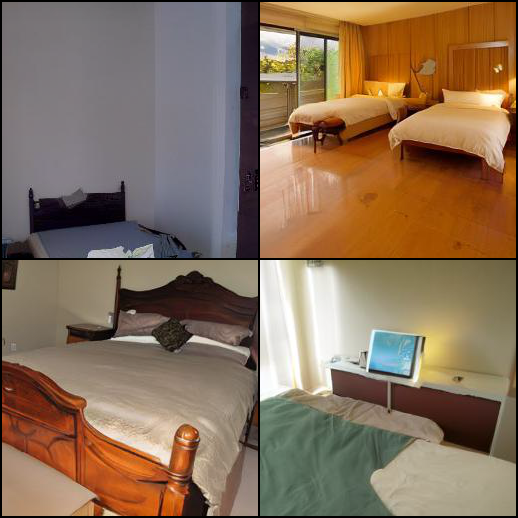}
    \end{subfigure} &
    \begin{subfigure}{0.2\textwidth}
        \includegraphics[width=\textwidth]{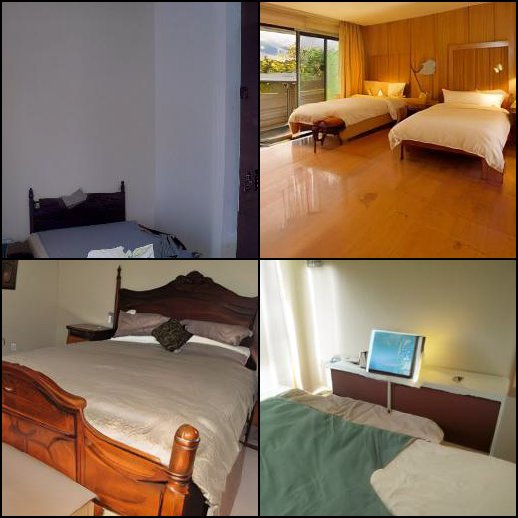}
    \end{subfigure} \\
    \raisebox{5.0\height}{SA-Solver($\eta=0.2$)} &
    \begin{subfigure}{0.2\textwidth}
        \includegraphics[width=\textwidth]{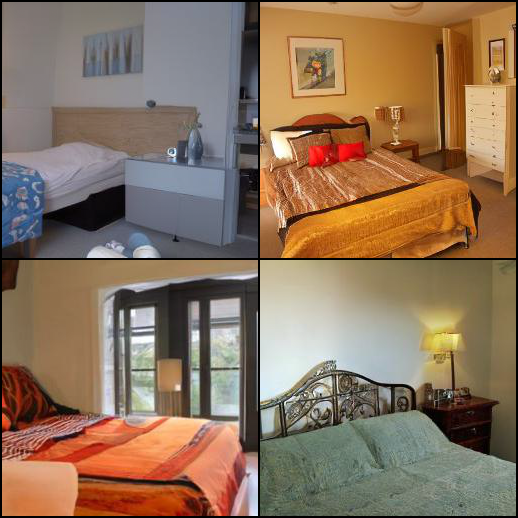}
    \end{subfigure} &
    \begin{subfigure}{0.2\textwidth}
        \includegraphics[width=\textwidth]{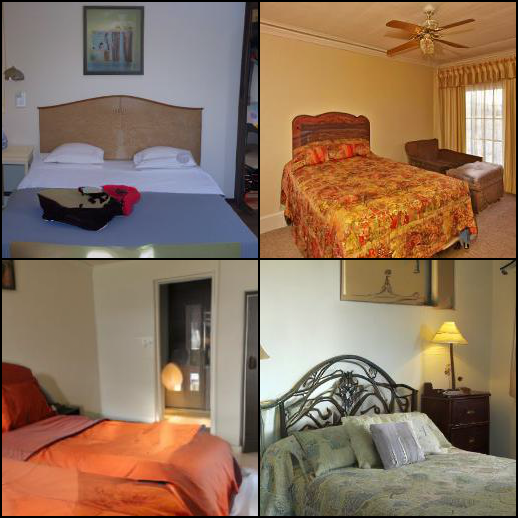}
    \end{subfigure} &
    \begin{subfigure}{0.2\textwidth}
        \includegraphics[width=\textwidth]{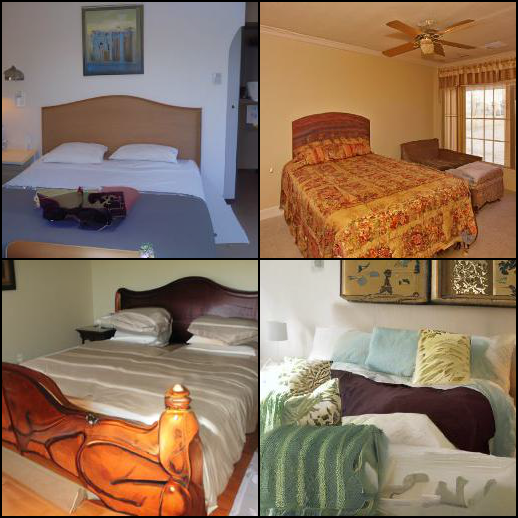}
    \end{subfigure} &
    \begin{subfigure}{0.2\textwidth}
        \includegraphics[width=\textwidth]{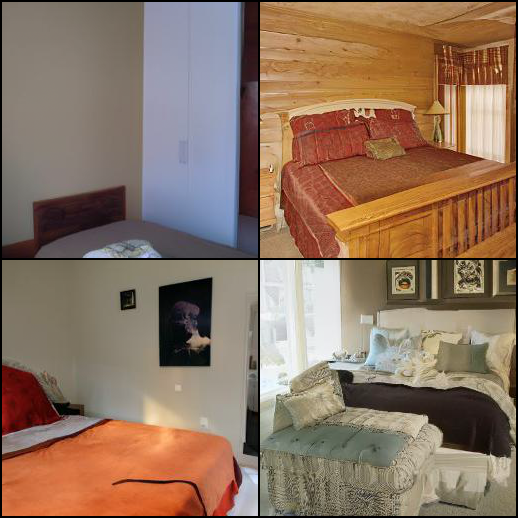}
    \end{subfigure} \\
    \raisebox{5.0\height}{SA-Solver($\eta=0.4$)} &
    \begin{subfigure}{0.2\textwidth}
        \includegraphics[width=\textwidth]{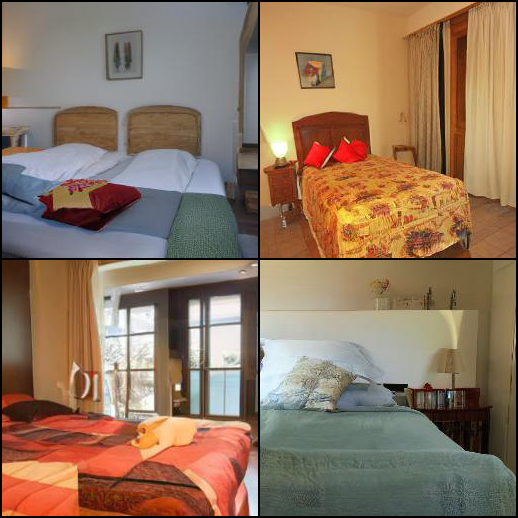}
    \end{subfigure} &
    \begin{subfigure}{0.2\textwidth}
        \includegraphics[width=\textwidth]{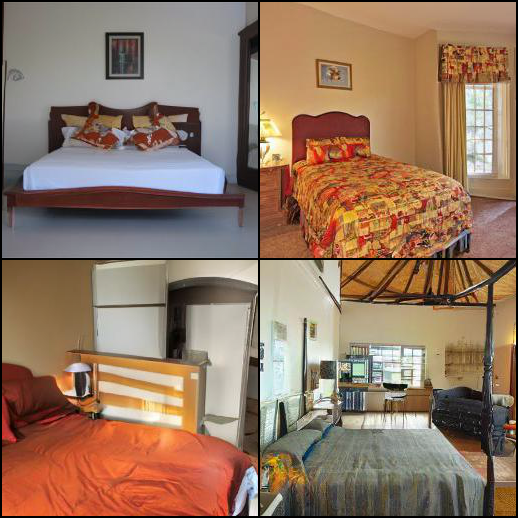}
    \end{subfigure} &
    \begin{subfigure}{0.2\textwidth}
        \includegraphics[width=\textwidth]{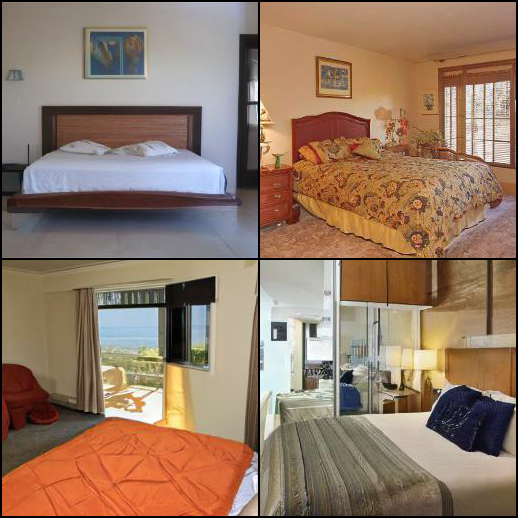}
    \end{subfigure} &
    \begin{subfigure}{0.2\textwidth}
        \includegraphics[width=\textwidth]{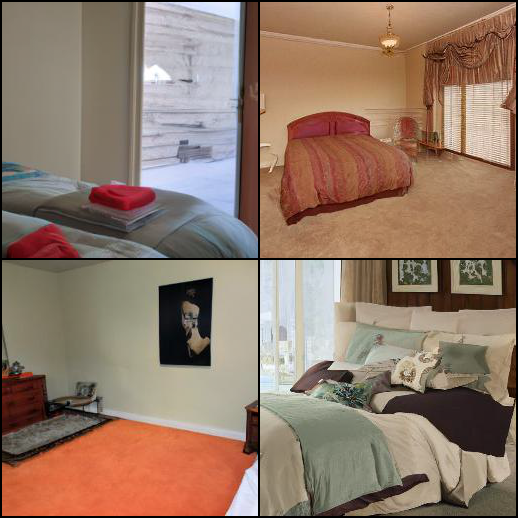}
    \end{subfigure} \\
    \raisebox{5.0\height}{SA-Solver($\eta=0.6$)} &
    \begin{subfigure}{0.2\textwidth}
        \includegraphics[width=\textwidth]{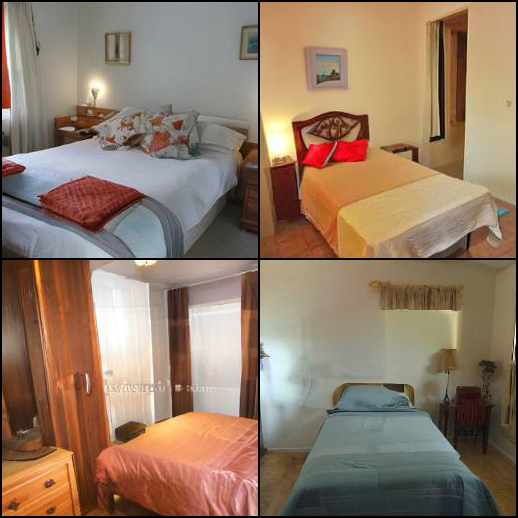}
    \end{subfigure} &
    \begin{subfigure}{0.2\textwidth}
        \includegraphics[width=\textwidth]{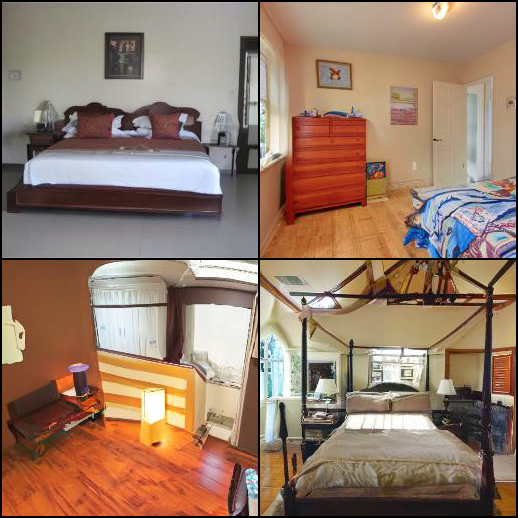}
    \end{subfigure} &
    \begin{subfigure}{0.2\textwidth}
        \includegraphics[width=\textwidth]{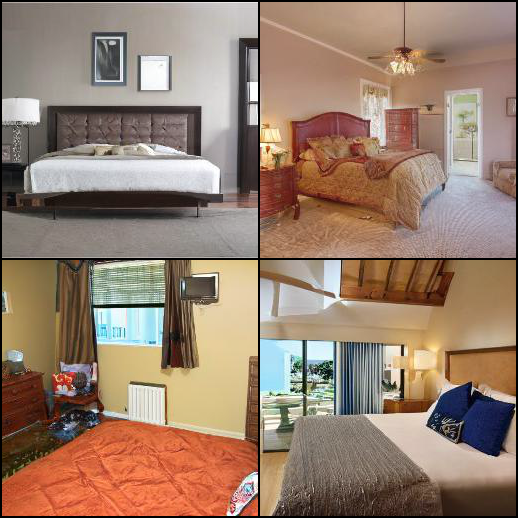}
    \end{subfigure} &
    \begin{subfigure}{0.2\textwidth}
        \includegraphics[width=\textwidth]{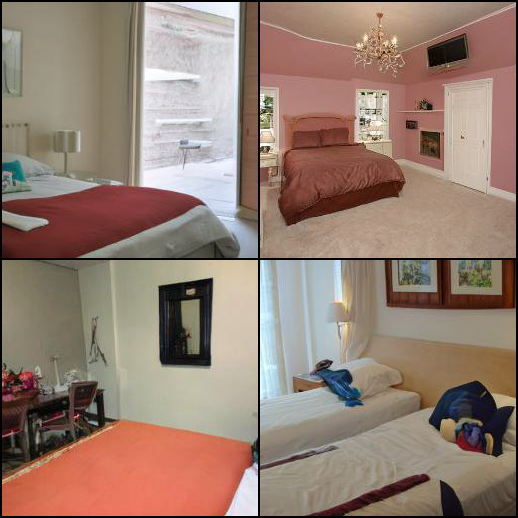}
    \end{subfigure} \\
    \raisebox{5.0\height}{SA-Solver($\eta=0.8$)} &
    \begin{subfigure}{0.2\textwidth}
        \includegraphics[width=\textwidth]{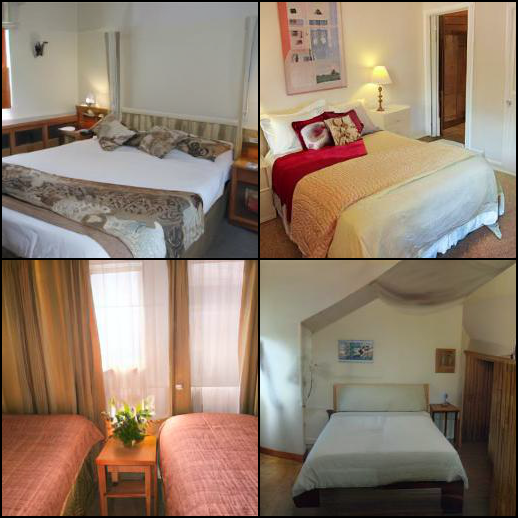}
    \end{subfigure} &
    \begin{subfigure}{0.2\textwidth}
        \includegraphics[width=\textwidth]{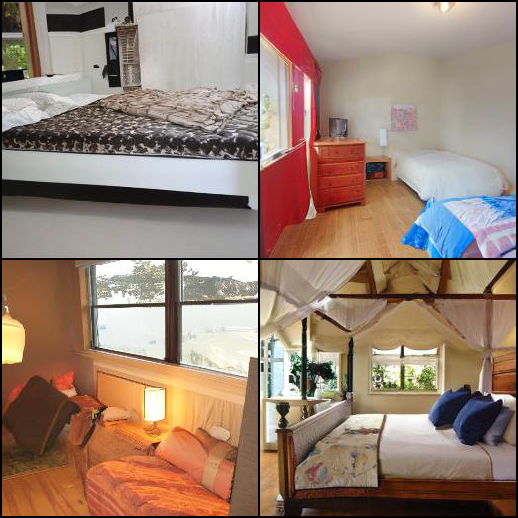}
    \end{subfigure} &
    \begin{subfigure}{0.2\textwidth}
        \includegraphics[width=\textwidth]{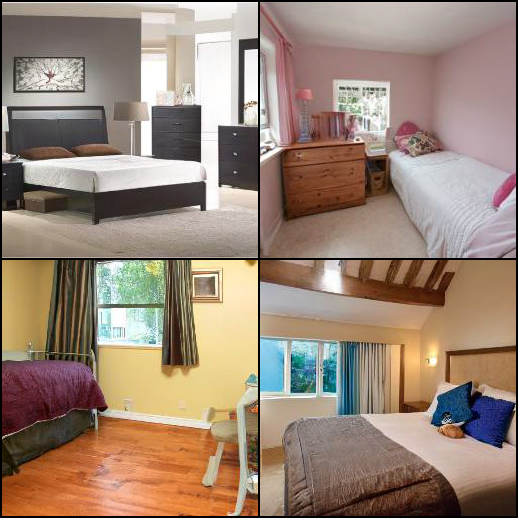}
    \end{subfigure} &
    \begin{subfigure}{0.2\textwidth}
        \includegraphics[width=\textwidth]{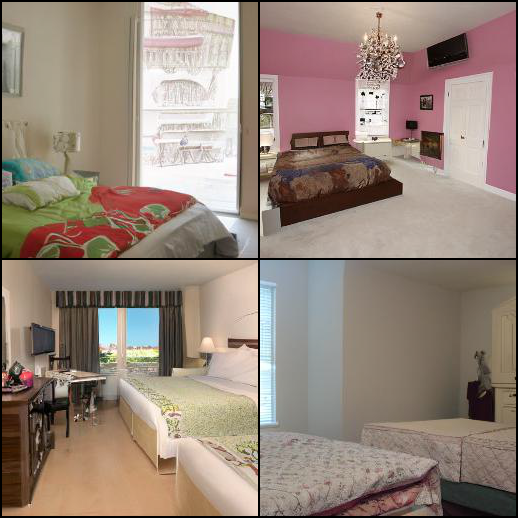}
    \end{subfigure} \\    
    \raisebox{5.0\height}{SA-Solver($\eta=1.0$)} &
    \begin{subfigure}{0.2\textwidth}
        \includegraphics[width=\textwidth]{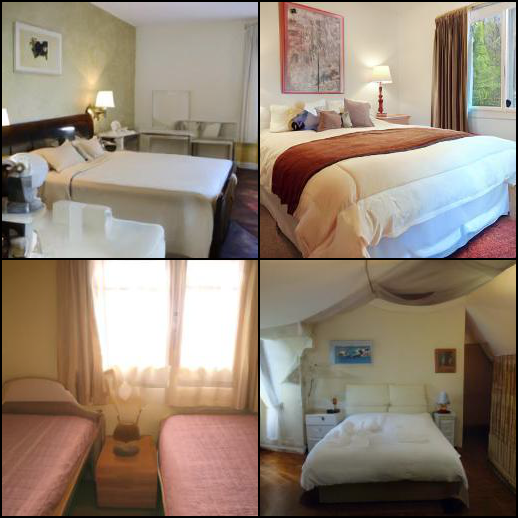}
    \end{subfigure} &
    \begin{subfigure}{0.2\textwidth}
        \includegraphics[width=\textwidth]{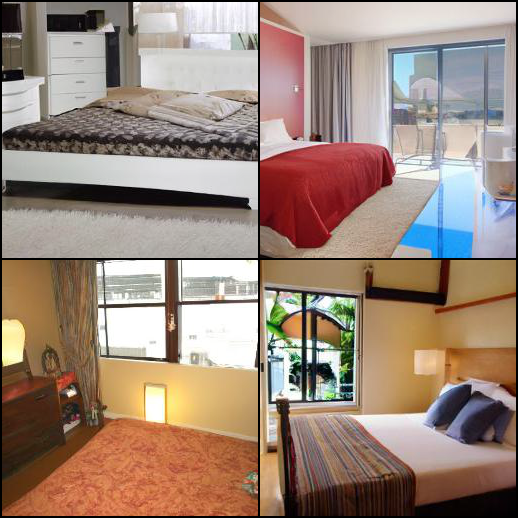}
    \end{subfigure} &
    \begin{subfigure}{0.2\textwidth}
        \includegraphics[width=\textwidth]{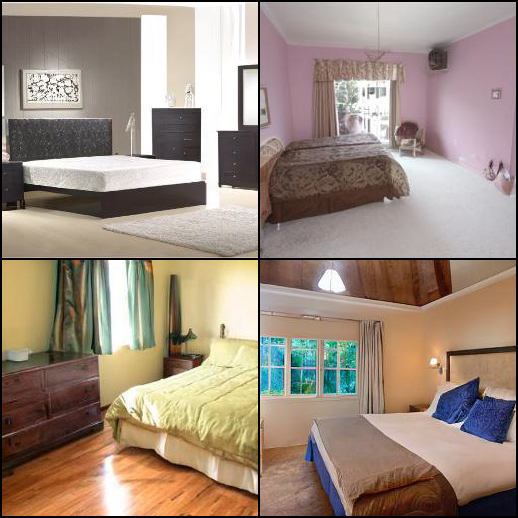}
    \end{subfigure} &
    \begin{subfigure}{0.2\textwidth}
        \includegraphics[width=\textwidth]{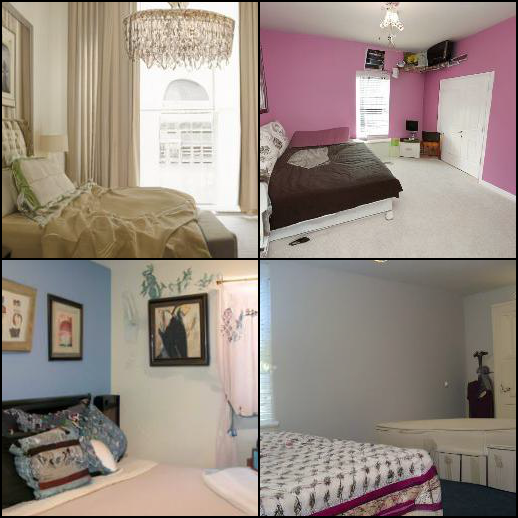}
    \end{subfigure} \\
    \raisebox{5.0\height}{SA-Solver($\eta=1.2$)} &
    \begin{subfigure}{0.2\textwidth}
        \includegraphics[width=\textwidth]{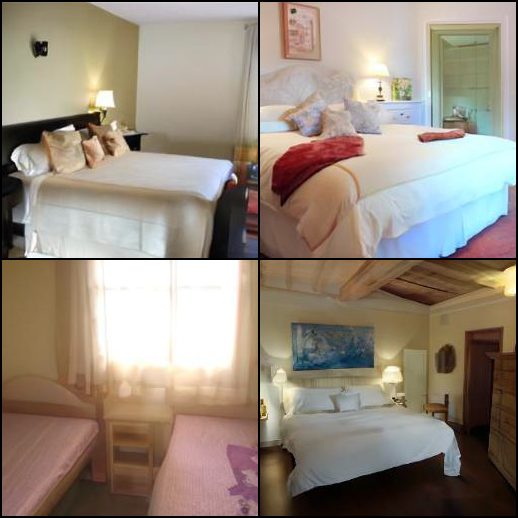}
    \end{subfigure} &
    \begin{subfigure}{0.2\textwidth}
        \includegraphics[width=\textwidth]{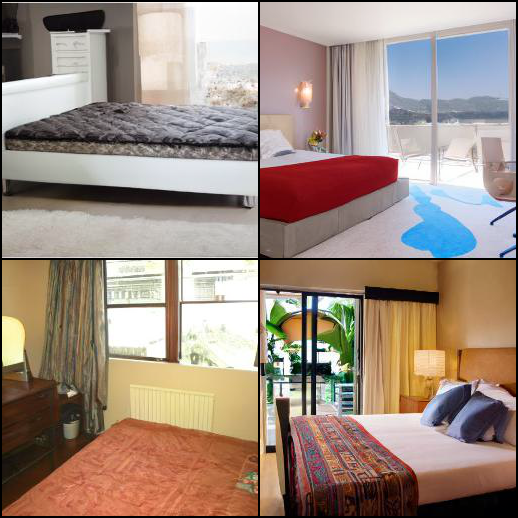}
    \end{subfigure} &
    \begin{subfigure}{0.2\textwidth}
        \includegraphics[width=\textwidth]{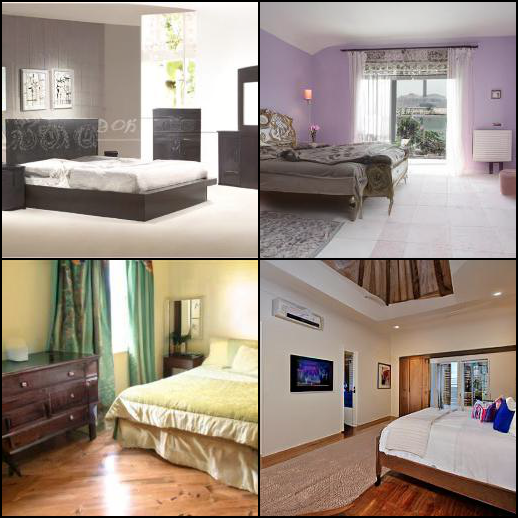}
    \end{subfigure} &
    \begin{subfigure}{0.2\textwidth}
        \includegraphics[width=\textwidth]{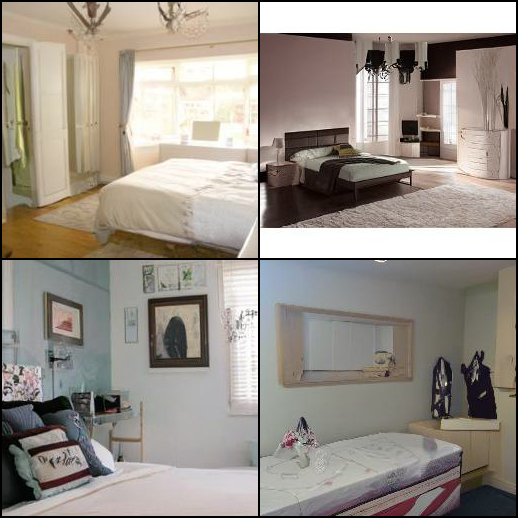}
    \end{subfigure} \\
    \end{tabular}
    \caption{Samples by SA-Solver with 20, 40, 60, 100 NFEs varying stochasticity($\tau$) with the same random seed from LSUN-Bedroom 256x256 model~\citep{dhariwal2021diffusion}(unconditional sampling).}
\label{figure: lsun}
\end{figure}

\begin{figure}[!htb]
    \centering
    \setlength{\tabcolsep}{1pt}
    \begin{tabular}{c c }
     & NFE = 60  \\
    \raisebox{10.0\height}{SA-Solver($\tau = 0$)} &
    \begin{subfigure}{0.7\textwidth}
        \includegraphics[width=\textwidth]{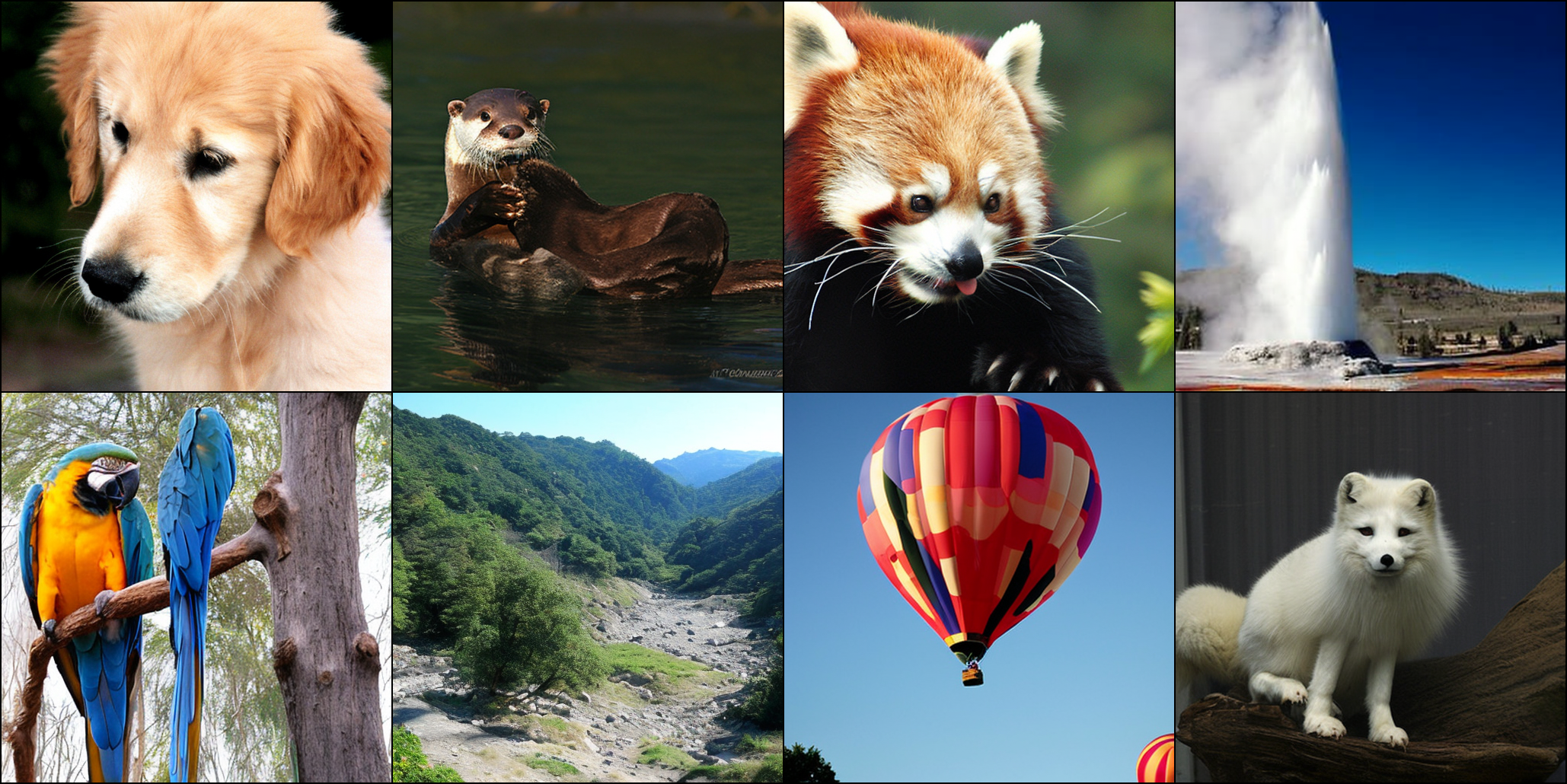} 
    \end{subfigure} \\
        \raisebox{10.0\height}{SA-Solver($\tau = 0.4$)} &
    \begin{subfigure}{0.7\textwidth}
        \includegraphics[width=\textwidth]{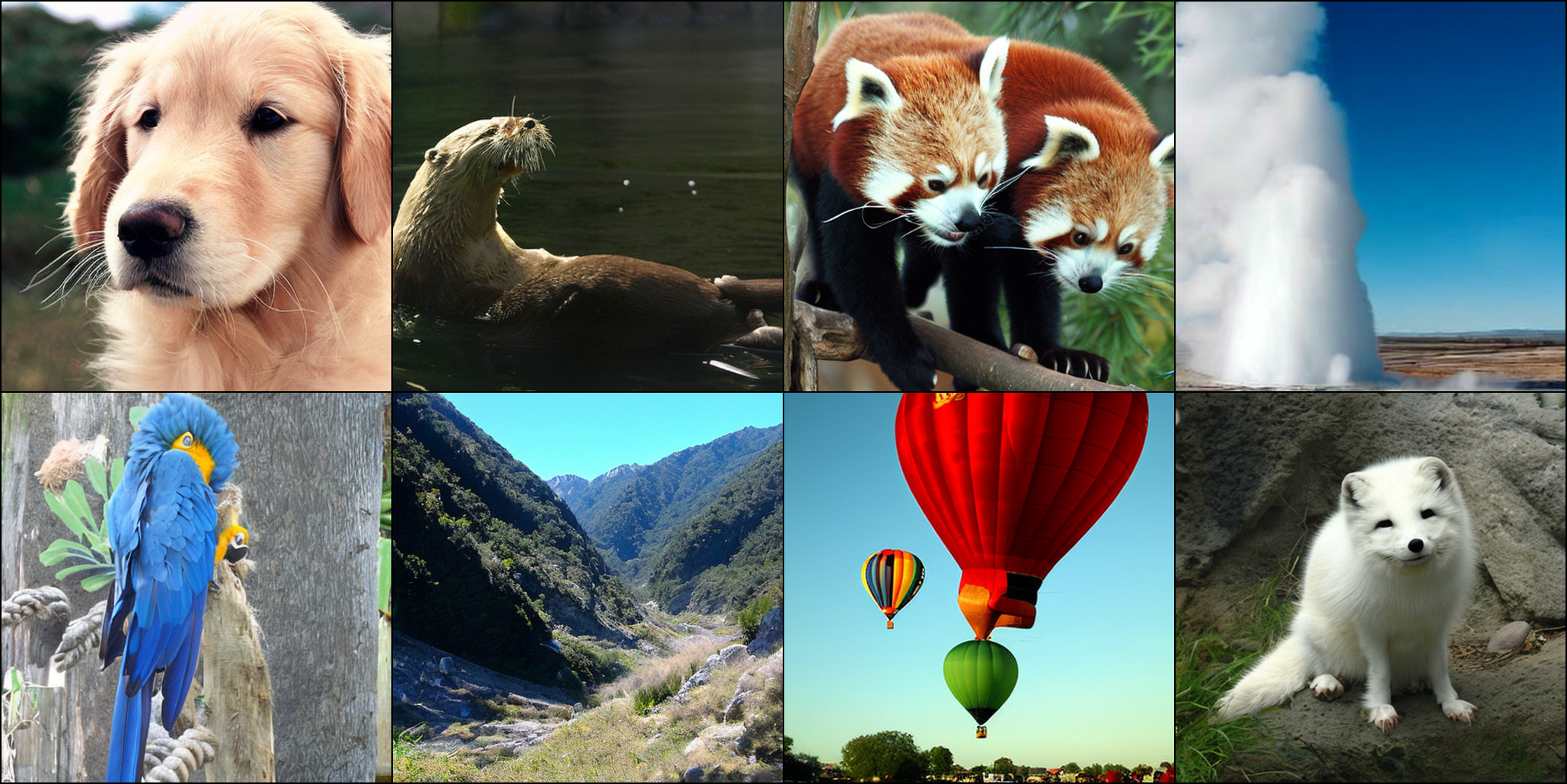} 
    \end{subfigure} \\
        \raisebox{10.0\height}{SA-Solver($\tau = 0.8$)} &
    \begin{subfigure}{0.7\textwidth}
        \includegraphics[width=\textwidth]{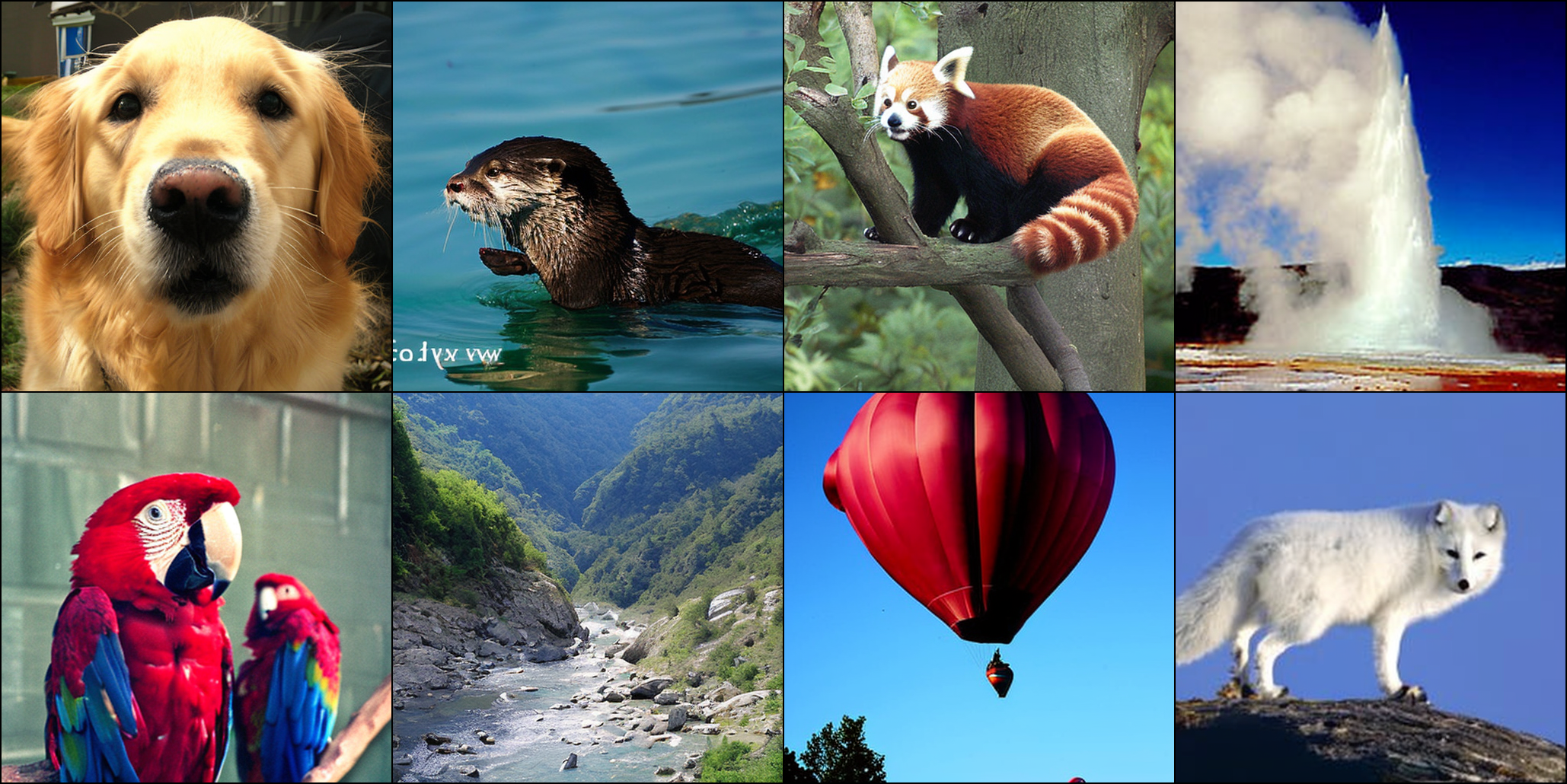} 
    \end{subfigure} \\
        \raisebox{10.0\height}{SA-Solver($\tau = 1.0$)} &
    \begin{subfigure}{0.7\textwidth}
        \includegraphics[width=\textwidth]{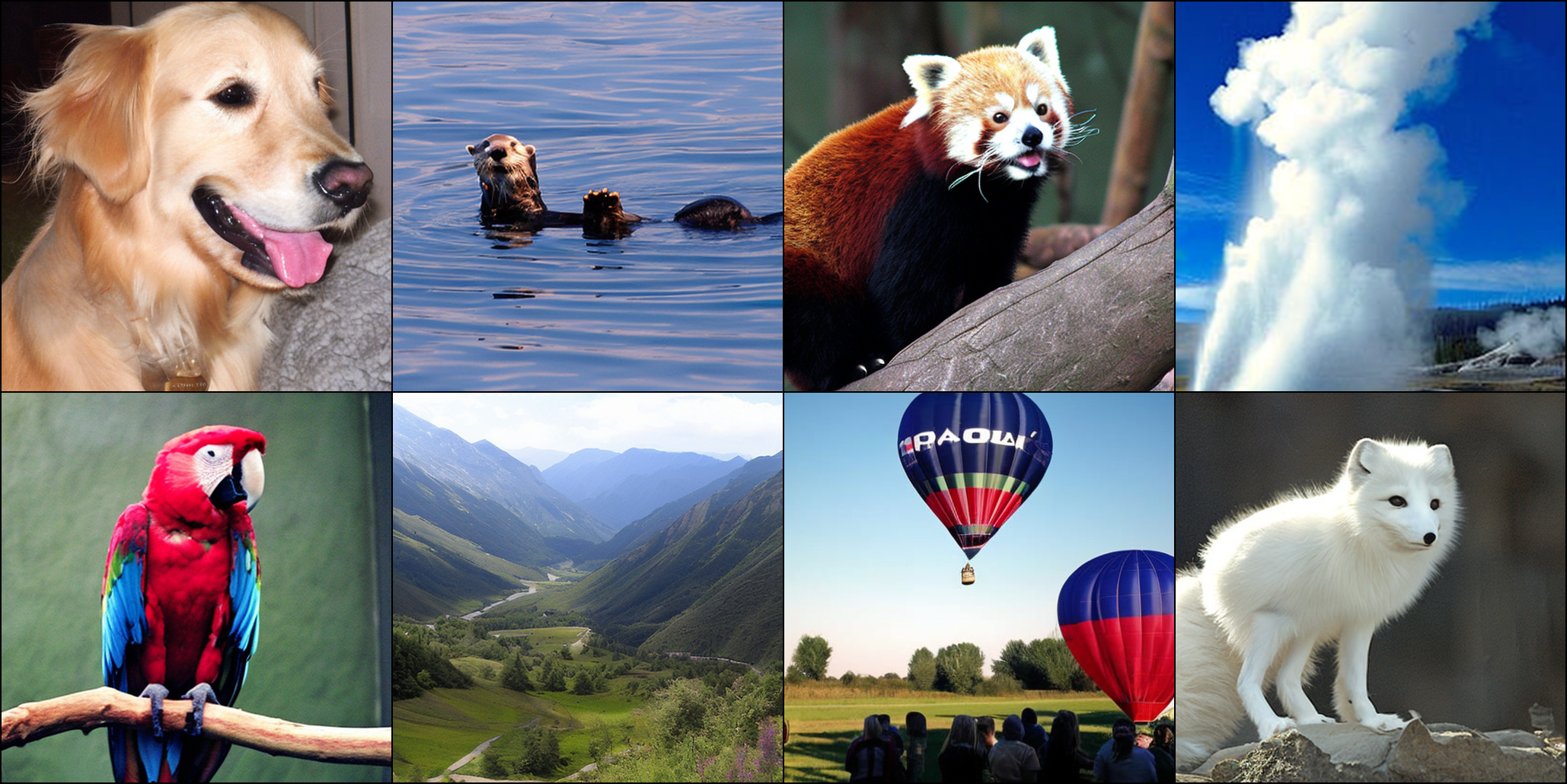} 
    \end{subfigure} \\

    \end{tabular}
    \caption{Samples by SA-Solver with 60 NFEs varying stochasticity($\tau$) with the same random seed from ImageNet 512x512 DiT model~\citep{peebles2022scalable} with classifer-free guidance scale $s = 4.0$(default setting to show image).}
\label{figure: dit512}
\end{figure}

\begin{figure}[!htb]
    \centering
    \setlength{\tabcolsep}{1pt}
    \begin{tabular}{c c c c}
     & NFE = 20 & NFE = 50 & NFE = 100  \\
    \raisebox{7.0\height}{DDIM($\eta=0$)} &
    \begin{subfigure}{0.25\textwidth}
        \includegraphics[width=\textwidth]{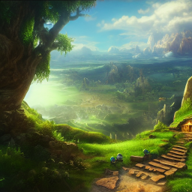}
    \end{subfigure} &
    \begin{subfigure}{0.25\textwidth}
        \includegraphics[width=\textwidth]{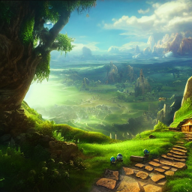}
    \end{subfigure} &
    \begin{subfigure}{0.25\textwidth}
        \includegraphics[width=\textwidth]{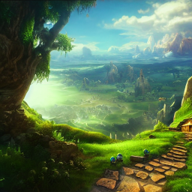}
    \end{subfigure} \\
    \raisebox{7.0\height}{UniPC} &
    \begin{subfigure}{0.25\textwidth}
        \includegraphics[width=\textwidth]{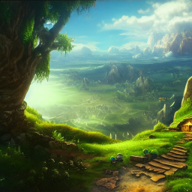}
    \end{subfigure} &
    \begin{subfigure}{0.25\textwidth}
        \includegraphics[width=\textwidth]{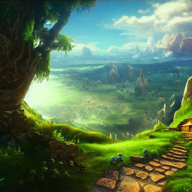}
    \end{subfigure} &
    \begin{subfigure}{0.25\textwidth}
        \includegraphics[width=\textwidth]{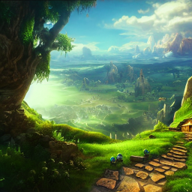}
    \end{subfigure} \\
        \raisebox{7.0\height}{SA-Solver(Ours)} &
    \begin{subfigure}{0.25\textwidth}
        \includegraphics[width=\textwidth]{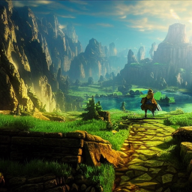}
    \end{subfigure} &
    \begin{subfigure}{0.25\textwidth}
        \includegraphics[width=\textwidth]{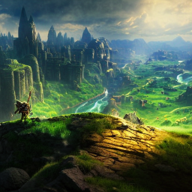}
    \end{subfigure} &
    \begin{subfigure}{0.25\textwidth}
        \includegraphics[width=\textwidth]{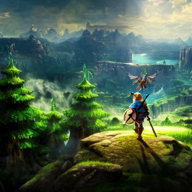}
    \end{subfigure} \\
    \end{tabular}
    \caption{Samples using Stable-Diffusion v1.5~\citep{Rombach_2022_CVPR} with a classifier-free guidance scale 7.5 with different solvers and NFEs. Prompt:The Legend of Zelda landscape atmospheric, hyper realistic, 8k, epic composition, cinematic, octane render, artstation landscape vista photography by Carr Clifton Galen Rowell, 16K resolution, Landscape veduta photo by Dustin Lefevre tdraw, 8k resolution, detailed landscape painting by Ivan Shishkin, DeviantArt, Flickr, rendered in Enscape, Miyazaki, Nausicaa Ghibli, Breath of The Wild, 4k detailed post processing, artstation, rendering by octane, unreal engine.}
\label{figure: t2i1}
\end{figure}

\begin{figure}[!htb]
    \centering
    \setlength{\tabcolsep}{1pt}
    \begin{tabular}{c c c c}
     & NFE = 20 & NFE = 50 & NFE = 100  \\
    \raisebox{7.0\height}{DDIM($\eta=0$)} &
    \begin{subfigure}{0.25\textwidth}
        \includegraphics[width=\textwidth]{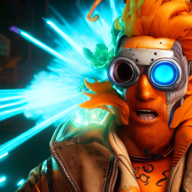}
    \end{subfigure} &
    \begin{subfigure}{0.25\textwidth}
        \includegraphics[width=\textwidth]{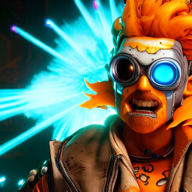}
    \end{subfigure} &
    \begin{subfigure}{0.25\textwidth}
        \includegraphics[width=\textwidth]{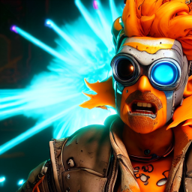}
    \end{subfigure} \\
    \raisebox{7.0\height}{UniPC} &
    \begin{subfigure}{0.25\textwidth}
        \includegraphics[width=\textwidth]{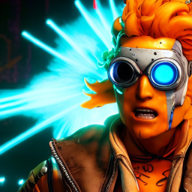}
    \end{subfigure} &
    \begin{subfigure}{0.25\textwidth}
        \includegraphics[width=\textwidth]{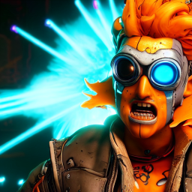}
    \end{subfigure} &
    \begin{subfigure}{0.25\textwidth}
        \includegraphics[width=\textwidth]{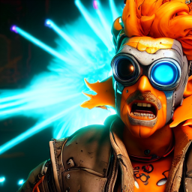}
    \end{subfigure} \\
        \raisebox{7.0\height}{SA-Solver(Ours)} &
    \begin{subfigure}{0.25\textwidth}
        \includegraphics[width=\textwidth]{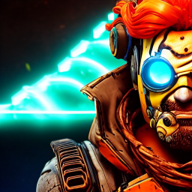}
    \end{subfigure} &
    \begin{subfigure}{0.25\textwidth}
        \includegraphics[width=\textwidth]{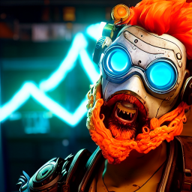}
    \end{subfigure} &
    \begin{subfigure}{0.25\textwidth}
        \includegraphics[width=\textwidth]{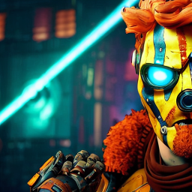}
    \end{subfigure} \\
    \end{tabular}
    \caption{Samples using Stable-Diffusion v1.5~\citep{Rombach_2022_CVPR} with a classifier-free guidance scale 7.5 with different solvers and NFEs. Prompt:glowwave portrait of curly orange haired mad scientist man from borderlands 3, au naturel, hyper detailed, digital art, trending in artstation, cinematic lighting, studio quality, smooth render, unreal engine 5 rendered, octane rendered, art style by pixar dreamworks warner bros disney riot games and overwatch.}
\label{figure: t2i2}
\end{figure}


\end{document}